\documentclass{article}

 \usepackage[main, final]{neurips_2026}
\usepackage[utf8]{inputenc} % allow utf-8 input
\usepackage[T1]{fontenc}    % use 8-bit T1 fonts
\usepackage{url}            % simple URL typesetting
\usepackage{hyperref}       % hyperlinks
\usepackage{booktabs}       % professional-quality tables
\usepackage{amsfonts}       % blackboard math symbols
\usepackage{nicefrac}       % compact symbols for 1/2, etc.
\usepackage{microtype}      % microtypography
\usepackage{xcolor}         % colors
\usepackage{amsmath}
\usepackage{amssymb}
\usepackage{amsthm}
\usepackage{graphicx}
\usepackage{subcaption}
\usepackage{algorithm}
\usepackage{algpseudocode}
\usepackage{mathtools}
\usepackage[table]{xcolor}
\usepackage{multirow}
\usepackage[normalem]{ulem} % for \uline
\usepackage{array}
\usepackage{pifont}
\usepackage{titletoc}
\usepackage{tabularx}
\usepackage[most]{tcolorbox}
\usepackage{wrapfig}
\usepackage{enumitem}
\usepackage{makecell}
\usepackage{fontawesome5} % for icons
\usepackage{adjustbox}

\hypersetup{
    colorlinks=true,
    citecolor=cyan,
    linkcolor=magenta,
}

\definecolor{epflred}{HTML}{E30613}
\definecolor{mutedblue}{HTML}{EAF2F8}
\definecolor{blueborder}{HTML}{1A4D8F}

\newtcolorbox{bluelinebox}{
    width=\linewidth,
    enhanced,
    colback=mutedblue!70,
    colframe=mutedblue!70,
    boxrule=0.6pt,
    arc=2mm,
    outer arc=2mm,
    left=2mm,
    right=2mm,
    top=1mm,
    bottom=0mm,
    before skip=0.5em,
    after skip=0.8\baselineskip,
    fuzzy shadow={0.8mm}{-0.8mm}{0mm}{0.15mm}{black!20},
}

\newtcolorbox{redlinebox}{
    enhanced,
    breakable,
    colback=epflred!4,
    colframe=epflred!4,
    boxrule=0.6pt,
    arc=2mm,
    outer arc=2mm,
    left=2mm,
    right=2mm,
    top=0.25mm,
    bottom= 0mm,
    before skip=0.5em,
    after skip=0.8\baselineskip,
    fuzzy shadow={0.8mm}{-0.8mm}{0mm}{0.15mm}{black!20},
}

\newtcolorbox{grayoutlinebox}{
    enhanced,
    breakable,
    colback=gray!4,
    colframe=gray!35,
    boxrule=0.6pt,
    arc=2mm,
    outer arc=2mm,
    left=2mm,
    right=2mm,
    top=0.25mm,
    bottom=0.25mm,
}

\definecolor{buttonbg}{HTML}{F5F5F5}
\definecolor{buttontext}{HTML}{000000} % black
\newtcbox{\linkbuttonbox}{
  on line,
  arc=4pt,
  boxrule=0pt,
  colback=buttonbg,
  colframe=buttonbg,
  boxsep=0pt,
  left=6pt,
  right=6pt,
  top=4pt,
  bottom=4pt
}

\newcommand{\paperbutton}[3]{%
  \href{#1}{%
    \linkbuttonbox{%
      \textcolor{buttontext}{\sffamily\small #2\hspace{0.35em}#3}%
    }%
  }%
}

\theoremstyle{definition}
\newtheorem{definition}{Definition}[section]

\newtheorem{theorem}{Theorem}
\newtheorem*{theorem*}{Theorem}
\newtheorem{lemma}{Lemma}
\newtheorem{assumption}{Assumption}
\newtheorem{proposition}{Proposition}

\newcommand{\tmark}{\ding{51}}
\newcommand{\xmark}{\ding{55}}

\title{RePercENT: Scaling Disentangled Representation Learning Beyond Two Modalities}

\author{%
  Vasiliki Rizou \\
  EPFL \\
  \texttt{vasiliki.rizou@epfl.ch} \\
  \And
  Pascal Frossard \\
  EPFL \\
  \texttt{pascal.frossard@epfl.ch} \\
  \And
  Dorina Thanou \\
  EPFL \\
  \texttt{dorina.thanou@epfl.ch} \\
}

\begin{document}

\setlength{\abovedisplayskip}{4pt plus 2pt minus 2pt}
\setlength{\belowdisplayskip}{4pt plus 2pt minus 2pt}
\setlength{\abovedisplayshortskip}{2pt}
\setlength{\belowdisplayshortskip}{2pt}

\maketitle

\begin{center}
\vspace{-2em}
\paperbutton{https://github.com/Vascorn/RePercENT}{\faIcon{github}}{ Code}
\vspace{0.5em}
\end{center}

% aiming to 

% to surpass the black-box limitations of joint representations 
\begin{abstract}
To leverage the full potential of multimodal data, we need representations that go beyond the state-of-the-art alignment and fusion approaches and exploit all cross-modal interactions without sacrificing modality-specific information. Learning disentangled representations is a principled way to identify the underlying shared and unique factors that are hidden in observational data. However, while multimodal disentanglement is a compelling paradigm, existing methods are largely confined to the two-modality regime due to its inherent scalability bottleneck. To address this, we propose RePercENT, a self-supervised framework designed to surpass these limitations and \textit{unlocks scalable pairwise disentanglement} beyond two modalities. Through a multimodal `plug-and-play' architecture, our approach operates directly on pre-extracted embeddings, eliminating the need for extensive joint pre-training while making no assumptions regarding the underlying modalities or foundation model backbones. Moreover, we introduce a joint optimization objective for simultaneously deriving the shared and unique components, and provide formal theoretical guarantees that characterize the optimality of our solution. Spanning modalities from images and videos to clinical reports and omics, and tasks from sentiment analysis to tumor classification, RePercENT successfully recovers disentangled components while maintaining competitive performance and significantly reducing computational complexity.
\end{abstract}

\section{Introduction}
    A central aim of multimodal machine learning is to endow models with the ability to synthesize and reason over information coming from multiple sources, e.g. images, audio, text~\citep{ml_survey, Krones2024ReviewOM}. Landmark multimodal foundation models, such as CLIP~\citep{pmlr-v139-radford21a}, ALIGN~\citep{Jia2021ScalingUV}, Flamingo~\citep{10.5555/3600270.3601993}, have shown that integrating different modalities gives rise to richer, more semantic and broadly applicable representations, than those derived from a single source alone. Notably, most previous work focuses on cross-modal alignment, exploiting the critical assumption of \textit{multi-view redundancy}~\citep{liang_factorized_2023}, which suggests that \textit{all} task-relevant information is shared across modalities. However, as demonstrated by~\citet{tjandrasuwita_understanding_2025}, this assumption critically depends on the degree of similarity between the modalities and the balance between redundant and unique information they provide for the task under consideration (see Figure \ref{Figure01:Venn}, left).
    
    \begin{figure}[ht!]
        \centering
        \vspace{-1.5em}
        \begin{subfigure}[t]{0.54\textwidth}
            \centering
            \includegraphics[width= \linewidth]{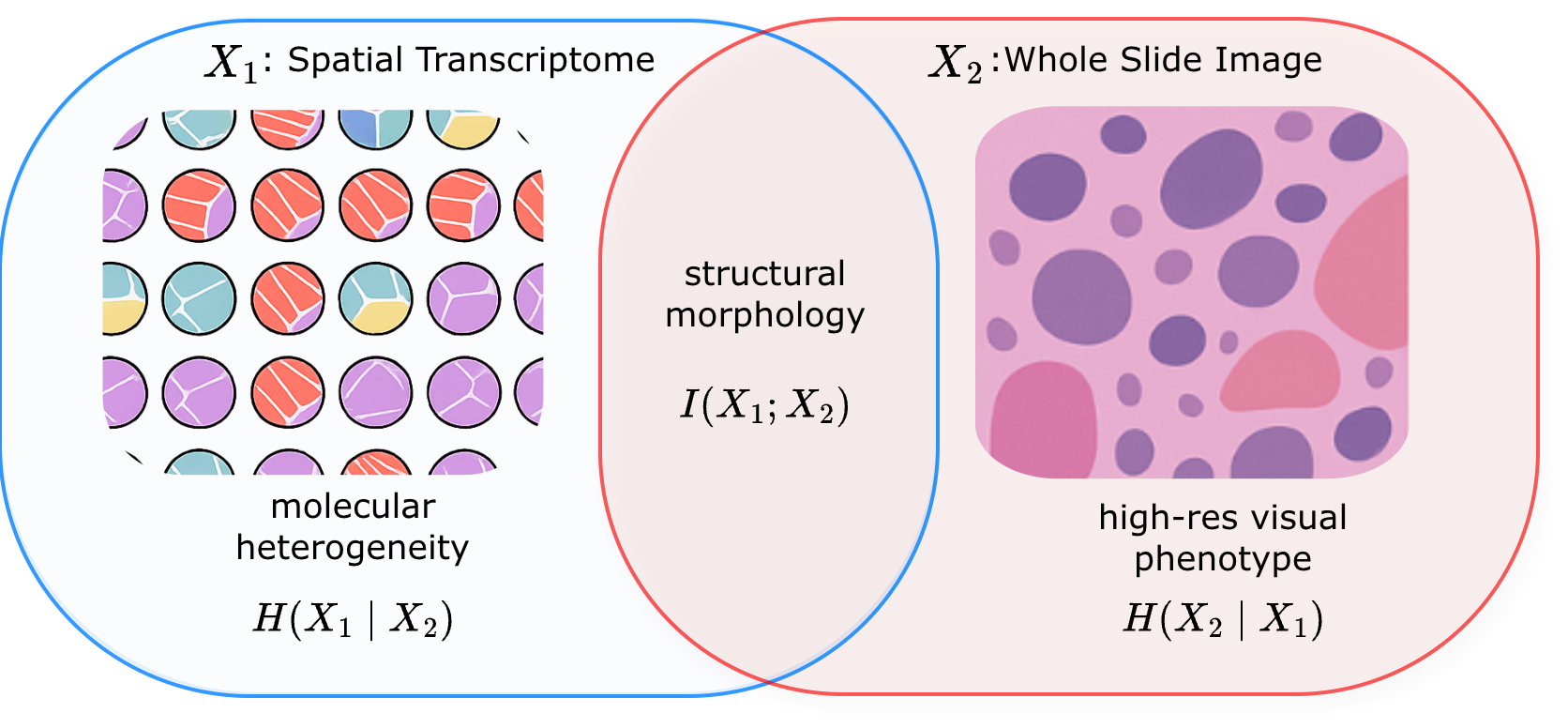}
        \end{subfigure}
        % \hspace{1em}
        \begin{subfigure}[t]{0.45\textwidth}
            \centering
            \includegraphics[width= \linewidth]{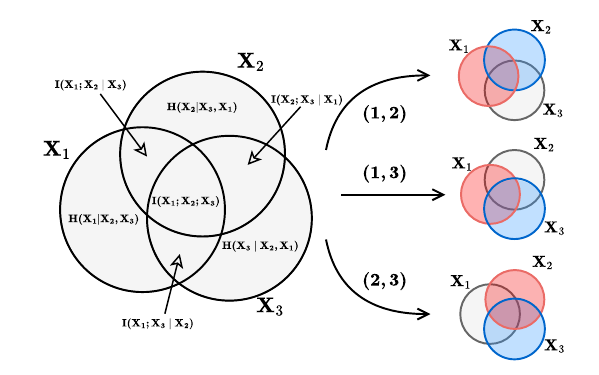}
        \end{subfigure}
        \caption{\textbf{Left:} Example in oncology when multi-view redundancy is limited. While both modalities capture shared structural morphology, WSI resolves fine-grained cellular features, whereas ST reveals underlying molecular variations, that are invisible in histology. \textbf{Right:} Information Venn diagram for three modalities, along with their pairwise shared and unique component visualizations.}
        \label{Figure01:Venn}
        \vspace{-1em}
    \end{figure}
    Multimodal fusion provides a better solution for capturing more complex cross-modal interactions, by combining heterogeneous modalities into a joint representation. While effective, this representation often relies on the most "dominant" modality~\citep{kontras_balancing_nodate}, under-representing the seemingly weaker signals and provides very little intuition regarding how information is distributed across them. This motivates \textit{disentangled representation learning}, where the goal is to decompose the multimodal information into distinct, meaningful components, separating what is shared across modalities from what remains modality-specific. Identifying the underlying shared and unique factors, reveals which modalities are most task-critical, which contribute unique information, and which exhibit significant overlap. This, in turn, can inform practical decisions such as pruning modalities that are computationally or cost demanding, yet \textit{redundant}. Additionally, such a decomposition can facilitate adaptability to unseen tasks, integration of new heterogeneous modalities and efficient fine-tuning on new datasets or cohorts.
    
    Despite recent progress~\citep{Wu2025LearningOM, liang_factorized_2023}, reliably disentangling multimodal information remains an open challenge. Importantly, the majority of prior work remains confined to two modalities, largely due to a fundamental scalability bottleneck, arising from two main factors. First, the number of possible shared and modality-specific factors grows exponentially with the number of modalities, making them increasingly difficult to model. Second, beyond two modalities, tractable objectives or bounds that explicitly characterize all these interactions remain limited. Moreover, in real-world settings, unique and shared information is often highly entangled, leading to a trade-off between minimizing redundancy across components and preserving as much relevant information as possible \citep{wang_information_2025}. Capturing only the Minimal Necessary Information (MNI) \citep{fischer_conditional_2020} is therefore key to learning informative, non-redundant components without inducing representational collapse. These challenges naturally raise a fundamental question: 
    
    \textit{"How can we scalably extract the underlying unique and shared factors beyond two modalities?"}

    We answer this question with \textbf{RePercENT}, a novel information-theoretic framework for \textbf{Re}duced-complexity \textbf{Perc}eiver-based dis\textbf{ENT}anglement. Our approach effectively addresses the scalability barrier, through two key ingredients. First, to avoid the exponential growth of components, we model pairwise interactions between modalities: this preserves the information of each modality while capturing all relevant cross-modal interactions--see Figure \ref{Figure01:Venn} (right). Second, we propose an architecture that redefines the efficient, modality-agnostic design of the Perceiver encoder  \citep{jaegle2021perceiver} through dedicated structural mechanisms that effectively route the available information into shared and modality-specific representations. Furthermore, by leveraging pre-extracted foundation model embeddings and a self-supervised training regime, our method remains modality- and task-agnostic, while still being robust to missing or incomplete modalities at inference time. Finally, our approach is accompanied with formal theoretical guarantees on the optimality of our solution. Our contributions are as follows:
    \begin{itemize}
        \item \textbf{Scalable multimodal disentanglement:} We introduce RePercENT, a novel information-theoretic framework for multimodal disentanglement that efficiently scales beyond the two-modality setting.
        \item \textbf{Plug-and-play architecture:} RePercENT operates directly on pre-extracted embeddings, making it agnostic to modality types and backbone architectures. By design, it accommodates missing data during inference and seamlessly transfers to unseen tasks.
        \item \textbf{Formal optimality guarantees:} We provide formal optimality guarantees both when minimal necessary information (MNI) is attainable and when it is not.
        \item \textbf{Extensive empirical validation:}  
        From synthetic data to figurative language and oncology, we demonstrate the robustness and effectiveness of our framework across nine real-world modalities and five domains.
    \end{itemize}
    Overall, we provide a principled framework that bridges information-theoretic foundations and formal guarantees with a practical and efficient implementation, paving the way towards more reliable, interpretable, and mechanistically grounded multimodal representations.
\section{Problem formulation}
    In this section, we formalize the problem of information-theoretic disentangled representation learning in multimodal settings and introduce the necessary notation and definitions.
    
    \noindent
    \begin{minipage}[t]{0.70\textwidth}
     We denote as $\mathcal{X} = \{X_1, X_2, \dots, X_M\}$ the set of $M$ modalities, and let $\mathcal{M} = \{1, 2, \dots, M\}$ be the corresponding index set. Let $Z_i$ be a latent, information-preserving representation of $X_i$. We associate each subset $A \subseteq \mathcal{M}$ with its corresponding latent representation $z_A$, which we term \textit{atomic representation} (see Definition \ref{def:atomic_representation}), reflecting the distinct information subspaces between modalities (see Figure \ref{fig:atomic_representations}).
    
    \begin{bluelinebox}
    \begin{definition}
    \textbf{Atomic representation.}
    For any non-empty subset \(A \subseteq \mathcal{M}\), we define as the atomic representation \(z_A\), the latent representation that captures information shared by all modalities in $A$ and exclusive to that subset, such that
    \begin{equation}
            I(z_A;z_{A'})=0 \qquad \forall\,A'\subseteq\mathcal{M} \ \text{with}\ A\cap A'=\varnothing,
        \end{equation}
    where \(I(z_A;z_{A'})\) is the mutual information between \(z_A\) and \(z_{A'}\).\looseness=-1
    \label{def:atomic_representation}
    \end{definition}
    \end{bluelinebox}
    \end{minipage}
    \hfill
    \begin{minipage}[t]{0.28\textwidth}
    \vspace{0pt}
        \centering
        \includegraphics[width=\textwidth]{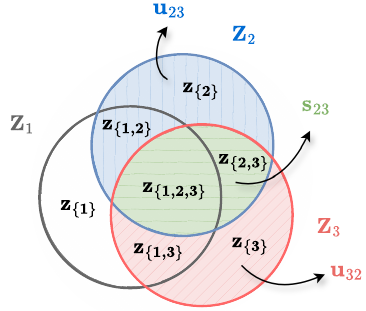}
        % \captionsetup{font=small}
        \captionof{figure}{Atomic representation decomposition for three modalities, through a Venn diagram illustration. \looseness=-1}
        \label{fig:atomic_representations}
    \end{minipage}
    
    Notice that each modality-specific latent representation $Z_i$ can then be interpreted as a composition of atomic representations:
    \begin{bluelinebox}
        \begin{definition} \textbf{Composite representation.}
        For \(i\in\mathcal{M}\), let \(\mathcal{A}_i=\{A\subseteq \mathcal{M}:i\in A\}\). The composite representation \(Z_i\) is the combination of all atomic representations whose subsets contains \(i\): \(Z_i=\mathop{\bigoplus}\limits_{A\in\mathcal{A}_i} z_A\). Here, \(\bigoplus\) denotes a generic composition operation, e.g., concatenation or any aggregation function.
            \label{def: 2.2}
        \end{definition}
    \end{bluelinebox}
    Subsequently, using the atomic representations as building blocks, we define shared and modality-specific representations for each modality pair $(i, j)$, according to Definition \ref{def: 2.3}. 
    \begin{bluelinebox}
    \begin{definition}
    \textbf{Pairwise unique and shared representations.}
    For \(i,j\in\mathcal{M}\), let \(\mathcal{A}_{i\mid j}=\mathcal{A}_i\cap\mathcal{A}_j^c\), where \(\mathcal{A}^c_j\) denotes the complementary set of \(\mathcal{A}_j\), and let \(\mathcal{A}_{ij}=\mathcal{A}_i\cap\mathcal{A}_j\). The pairwise unique and shared representations are defined as,
            \[
            \mathbf{u}_{ij} = \bigoplus_{A \in \mathcal{A}_{i \mid j}} z_A,
            \qquad
            \mathbf{s}_{ij} = \bigoplus_{A \in \mathcal{A}_{ij}} z_A.
            \]Here, \(\mathbf{u}_{ij}\) and \(\mathbf{s}_{ij}\) capture information unique to \(i\) relative to \(j\), and shared by \(i\) and \(j\), respectively.
    \label{def: 2.3}
    \end{definition}
    \end{bluelinebox}
    Finally, we formalize the learning objective as follows: 
    \begin{redlinebox}
        \textbf{Problem Statement}
        Given the composite representation $Z_i$ of each modality $i\in\mathcal{M}$, for every $j\in\mathcal{M}\setminus\{i\}$, our goal is to derive the pairwise disentangled latent representations $\mathbf{u}_{ij}$ and $\mathbf{s}_{ij}$:
        \[
        Z_i \;\longmapsto\; \{(\mathbf{u}_{ij},\, \mathbf{s}_{ij})\}_{j \in \mathcal{M}\setminus\{i\}}.
        \]
        \label{problem_definition}
    \end{redlinebox}
   Notice that the above formulation efficiently circumvents the exponential growth in the number of atomic representations\footnote{Each composite representation $Z_i$ consists of exactly $2^{M -1}$ atomic representations.} reducing the component complexity from \(\mathcal{O}(2^M)\) to \(\mathcal{O}(M^2)\), 
    while explicitly capturing all pairwise cross-modal interactions. 
    Importantly, no modality information is lost, as for each $j\neq i$, $(\mathbf{u}_{ij}, \mathbf{s}_{ij})$ encapsulate the complete information reflected in $Z_i$. Consequently, even the higher order components $z_A$, while not explicitly disentangled are fully retained through the pairwise components, residing in $\mathbf{s}_{ij}$ whenever $i, j \in A$ and in $\mathbf{u}_{ij}$ whenever $i \in A$ and $j \notin A$.
    
\section{RePercENT: Unlocking scalable disentangled representations}
\label{framework_overview}
   Building on this formulation, we introduce RePercENT. We first present the architectural design, then define the optimization objective leading to tractable training losses, and finally derive theoretical guarantees on the optimality of our solution.
   \vspace{-0.5em}
\subsection{Architecture overview}
    \begin{figure}[ht!]
        \centering
        \includegraphics[width=\linewidth]{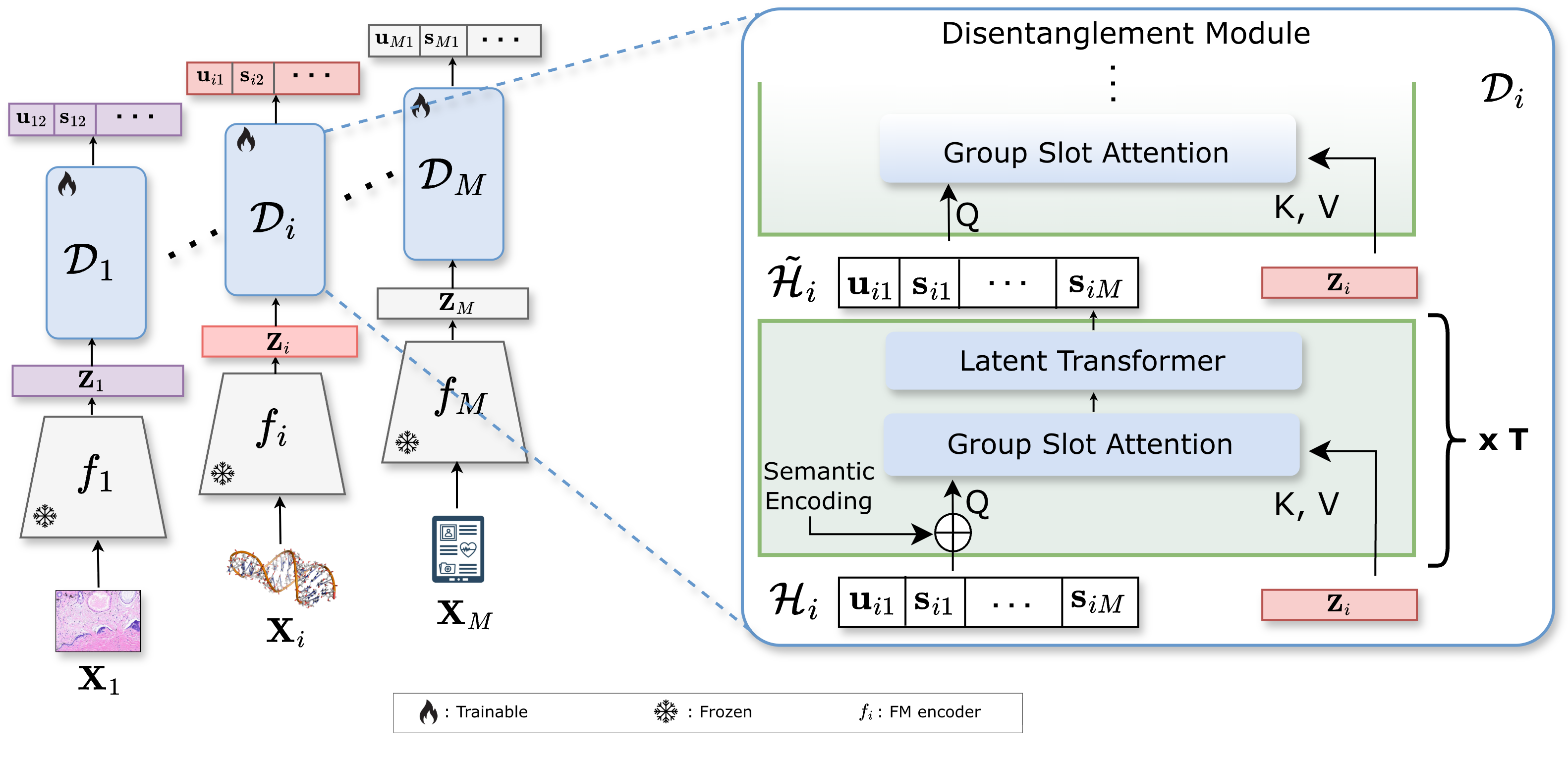}
        \caption{\textit{Model overview}. Each modality $X_i$ is first encoded, using modality-specific FMs. Afterwards, each encoded representation is processed through its dedicated disentanglement module $\mathcal{D}_i$.}
        \label{Figure03:model_overview}
        \vspace{-1em}
    \end{figure}
    
    We propose a scalable framework for multimodal disentanglement that extracts the desired information-theoretic representations while requiring \textit{a single encoder} per modality, yielding linear scaling in $M$. Specifically, from each modality $i$, we leverage pre-trained modality specific foundation models (FMs) to obtain the initial embeddings $Z_i$. Under the assumption that these encoders are sufficiently expressive and thus the resulting representations preserve most of the relevant modality information, we treat the embeddings $Z_i$, as composite representations. Subsequently, each representation is processed via a separate disentangling module $\mathcal{D}_i$, that decomposes $Z_i$ into all desired information components involving modality $i$, i.e., $Z_i\;\longmapsto\; \{(\mathbf{u}_{ij},\, \mathbf{s}_{ij})\}_{j \in \mathcal{M}\setminus\{i\}}$. 
    % \vspace{-0.25em}
    \paragraph{Disentanglement module} Within each $\mathcal{D}_i$, we model the disentangled representations using a learnable latent array $\mathcal{H}_i \in \mathbb{R}^{N \times D}$, where \(N=2(M-1)\) denotes the number of latent slots per modality and $D$ is the dimension of each slot. To link pairwise unique and shared components to the latent array, we assign each row $h_{ik} \in \mathbb{R}^{D}$ of $\mathcal{H}_i$ to exactly one component (unique or shared) of modality $i$ (see Appendix \ref{def: mapping_def}). $\mathcal{H}_i$ is then iteratively refined through a Perceiver-inspired latent attention mechanism. While \citet{jaegle2021perceiver} introduced the Perceiver as a general-purpose architecture for scalable input processing, we leverage its latent attention mechanism for a different objective: learning granular disentangled representations. In our setup, $\mathcal{H}_i$ serves as a query over the encoded representation $Z_i$, followed by a latent self-attention block, as illustrated in Figure \ref{Figure03:model_overview}. This latent bottleneck is well suited to our setting, as it enables efficient representation encoding without making any modality-specific assumptions. To encourage disentanglement, we introduce two complementary routing mechanisms: \textbf{Semantic Encoding}, which assigns a functional role to each latent slot, and \textbf{Group Slot Attention}, a novel attention mechanism, that promotes competition among slots.
    \paragraph{Semantic encoding (SE)} To promote component specialization, each slot $h_{ik}$ is augmented with two learnable encodings: a pair encoding, \(e^{\mathrm{p}}_{ij} \in \mathbb{R}^{D}\), which identifies the modality pair $(i,j)$, and a type encoding, \(e^{\mathrm{t}}_{ij} \in \mathbb{R}^{D}\), indicating the component type, $\mathrm{t} \in \{\mathrm{unique}, \mathrm{shared}\}$. The initialized latent slot is thus defined as: $\tilde{h}_{ik} = h_{ik} + e^{\mathrm{p}}_{ij} + e^{\mathrm{t}}_{ij}$.
    \paragraph{Group slot attention (GSA)} Rather than compressing modality \(X_i\) into a single pooled embedding, we retain the full sequence of embeddings \(Z_i \in \mathbb{R}^{S_i \times E_i}\), allowing the latent slots to attend to fine-grained input structure. We then introduce \textit{group slot attention}: a structural routing mechanism tailored to pairwise disentanglement. While motivated by the competitive assignment principle of Slot Attention~\citep{10.5555/3495724.3496691}, our mechanism applies competition within modality-pair groups. Specifically, for each pair \((i,j)\), the slots corresponding to \((\mathbf{u}_{ij}, \mathbf{s}_{ij})\) compete over the same input embeddings, enforcing specialization within the pair. This way, each group effectively separates modality \(i\)'s information into \(i\)-specific and \(i\)-\(j\) shared latent representations (see Appendix~\ref{framework_details}).

    Together, these design choices impose an explicit routing bias, enabling each slot to specialize in its assigned representation while keeping the inference modality-local. This offers a central scalability advantage, as all pairwise unique and shared representations are obtained with only one encoder per modality. In contrast, state-of-the-art disentanglement methods \citep{Wu2025LearningOM, liang_factorized_2023} rely on separate encoders for different representations, leading to \(\mathcal{O}(M^2)\) structural complexity, which quickly becomes prohibitive as $M$ grows. Our method reduces this to $\mathcal{O}(M)$ while retaining the full set of pairwise decompositions. Moreover, by construction, for every \(j \in \mathcal{M}\), the representations \((\mathbf{u}_{ij}, \mathbf{s}_{ij})\) are inferred from \(X_i\) alone, i.e., 
    $
    \mathbf{u}_{ij} \sim p(\cdot \mid X_i),
    $ and $
    \mathbf{s}_{ij} \sim p(\cdot \mid X_i).
    $
    Equivalently,
    \begin{equation}
        (\mathbf{u}_{ij}, \mathbf{s}_{ij}) \perp X_j \mid X_i,
    \qquad
    (\mathbf{u}_{ji}, \mathbf{s}_{ji}) \perp X_i \mid X_j.
    \label{eq: markov_structure}
    \end{equation}
    Crucially, this decoupling ensures robustness to missing modalities during inference. Since \((\mathbf{u}_{ij}, \mathbf{s}_{ij})\) are inferred directly from \(X_i\), they remain available whenever \(X_i\) is observed, regardless of whether \(X_j\) is present, bypassing the need for complex imputation or heuristic mapping.

    \subsection{Information criteria for optimal disentangled representations}
    In order to learn the desired representations, we cast each pairwise decomposition as a single optimization problem. Unlike recent approaches that follow a sequential strategy, in which the shared representation is learned first and the modality-specific representation is subsequently inferred as a residual component~\citep{wang_information_2025}, we propose to simultaneously optimize all representations $\theta_{ij} = (u_{ij}, u_{ji}, s_{ij}, s_{ji})$, according to the following optimization objective:
    \begin{bluelinebox}
        \label{optimization_objective}
    \textbf{Joint optimization objective.}
    For each pair \((i,j)\), we optimize \(\theta_{ij}=(u_{ij},u_{ji},s_{ij},s_{ji})\),
        \begin{equation}
        \hat{\theta}_{ij}
        \in
        \arg\max_{\theta_{ij}}\,\mathcal{J}(\theta_{ij})
        =
        \arg\max_{\theta_{ij}}
        \left[
        \alpha(\mathcal{L}_{s_i}+\mathcal{L}_{s_j})
        +\mathcal{L}_{u_i}
        +\mathcal{L}_{u_j}
        \right],  \text{    where},
        \label{eq:eq_01}
        \end{equation}
        \begin{equation}
        \begin{aligned}
        \mathcal{L}_{s_i}
        &= I(s_{ij};X_j)-\beta I(s_{ij};X_i\mid X_j),
        &
        \mathcal{L}_{u_i}
        &= I(u_{ij},s_{ji};X_i)-\lambda I(u_{ij};s_{ji})
        \end{aligned}
        \label{eq:eq_02}
        \end{equation}
        \begin{equation}
        \begin{aligned}
        \mathcal{L}_{s_j}
        &= I(s_{ji};X_i)-\beta I(s_{ji};X_j\mid X_i),
        &
        \mathcal{L}_{u_j}
        &= I(u_{ji},s_{ij};X_j)-\lambda I(u_{ji};s_{ij}).
        \end{aligned}
        \label{eq:eq_03}
        \end{equation}
    \end{bluelinebox}
    
    The shared objectives $\mathcal{L}_{s_i}, \mathcal{L}_{s_j}$, encourage each shared representation to capture all \textit{necessary} cross-modal information, while at the same time being \textit{minimal}, i.e. penalizing information that can only be explained by its source modality. Specifically, the mutual information $I(s_{ij}; X_j)$, promotes cross-modal relevance, whereas the conditional mutual information, $I(s_{ij}; X_i \mid X_j)$, suppresses cross-modal leakage. Complementarily, the unique objectives $\mathcal{L}_{u_i}, \mathcal{L}_{u_j}$ preserve source-modality information via $I(u_{ij},s_{ji};X_i)$, while enforcing cross-modal disentanglement by penalizing overlap between unique and shared components across modalities through $I(u_{ij}; s_{ji})$. We introduce, the hyperparameter $\alpha > 0$ acting as a weighting constant between the shared and unique objectives, while (\(\beta, \lambda\))~\citep{wang_information_2025} control the trade-off between representation coverage and redundancy.

\subsubsection{Tractable training objectives.}
    For training the framework, we instantiate the information-theoretic components in Eq.~\eqref{eq:eq_01} using tractable, principled relaxations of the quantities in Eqs. \eqref{eq:eq_01} - \eqref{eq:eq_03}. Specifically, we employ the InfoNCE objective \citep{Oord2018RepresentationLW}, which variationally lower-bounds mutual information to estimate the terms \(I(s_{ij};X_j)\) and \(I(u_{ij}, s_{ji};X_i)\). For the latter, positive pairs are formed from two independently augmented views of \(X_i\). We model $I(s_{ij};X_i \mid X_j)$ using a KL-divergence penalty, $\mathcal{L}_{\mathrm{KL}}$, between the conditional distributions $p(s_{ij}\mid X_i)$ and $p(s_{ji}\mid X_j)$, as it provides a variational upper bound on the corresponding conditional mutual information~\citep{Federici2020LearningRR}. Lastly, we reduce the dependence between \(u_{ij}\) and \(s_{ji}\) via a cross-covariance penalty, $\mathcal{L}_{\mathrm{xcov}}$, on standardized representations. Concretely, Eq.~\eqref{eq:eq_02} is approximated as
    \begin{equation}
        \mathcal{L}_{s_i}
        \approx
        \mathcal{L}^{\mathrm{tr}}_{s_i}
        =
        -\mathcal{L}^{\mathrm{INCE}}_{s_i}
        -
        \beta\,\mathcal{L}_{\mathrm{KL}},
        \qquad
        \mathcal{L}_{u_i}
        \approx
        \mathcal{L}^{\mathrm{tr}}_{u_i}
        =
        -\mathcal{L}^{\mathrm{INCE}}_{u_i}
        -
        \lambda\,\mathcal{L}_{\mathrm{xcov}}.
        \label{eq:eq_tract_1}
    \end{equation}
    The terms corresponding to Eq.~\eqref{eq:eq_03} are defined analogously by exchanging \(i\) and \(j\). Averaging over all modality pairs, we obtain the full trainable objective
    \begin{equation}
        \mathcal{L}^{\mathrm{tr}}_{\mathcal{E}}
        =
        \frac{1}{|\mathcal{E}|}
        \sum_{(i,j)\in\mathcal{E}}
        \left[
        \alpha \left(\mathcal{L}^{\mathrm{tr}}_{s_i} + \mathcal{L}^{\mathrm{tr}}_{s_j}\right)
        + \mathcal{L}^{\mathrm{tr}}_{u_i}
        + \mathcal{L}^{\mathrm{tr}}_{u_j}
        \right],
        \qquad
        \mathcal{E} = \{(i,j)\in\mathcal{M}^2 : i < j\}.
        \label{eq:eq_tract_2}
    \end{equation}
    Detailed expressions and further discussion on the losses can be found in Appendix \ref{training_tractable_objectives}.
    \subsubsection{Optimality guarantees}
    We now establish theoretical guarantees for our joint objective relative to the optimal information-theoretic decomposition. Specifically, we show that under attainable \textit{Minimum Necessary Information} (MNI)~\citep{fischer_conditional_2020}, our proposed objective admits the same global optimum as the ideal optimizer. Moreover, we prove near-optimality, in the more challenging regime, when MNI is unattainable.
    \paragraph{Optimal information-theoretic objectives} Concretely, considering the ideal sequential optimization problem \citep{wang_information_2025}, the optimal shared and unique representations, can be derived following:\looseness=-1
    \begin{equation}
        \textbf{shared:} \qquad s^*_{ij} \in \arg\max_{s_{ij}} \mathcal{L}^{o}_{s_i} \; =\arg\max_{s_{ij}} \; I(s_{ij}; X_j) - \beta\, I(s_{ij}; X_i \mid X_j),
        \label{eq:eq_04}
    \end{equation}
    \vspace{-0.9em}
    \begin{equation}
             \textbf{unique:} \qquad u^*_{ij} \in \arg\max_{u_{ij}} \mathcal{L}^{o}_{u_i} \; =\arg\max_{u_{ij}} \; I(u_{ij}, X_j; X_i) - \lambda I(u_{ij}; s^*_{ij}).\footnotemark
         \label{eq:eq_05}
    \end{equation}
    \footnotetext{The representations $s_{ji}, u_{ji}$ are accordingly derived by exact symmetry.} 
    Notice that the unique representations are inferred from the optimal solutions of the first step.
    \paragraph{From sequential to joint optimization} 
    For the case when MNI holds, Theorem \eqref{theo: the_01}, demonstrates that any step-by-step optimal solution, is also a global maximizer of our joint objective.

    \begin{redlinebox}
        \begin{theorem}
            \textit{Exact optimality under attainable MNI: }
            Let \(\theta^*_{ij} = (u^*_{ij}, u^*_{ji}, s^*_{ij}, s^*_{ji})\) be an optimal solution of the step-by-step optimization problem defined in Eqs.~\eqref{eq:eq_04} and \eqref{eq:eq_05}. 
            If the MNI criterion is attainable and \(\alpha \ge \lambda\), then \(\theta^*_{ij}\) is a global optimizer of the joint objective defined in Eqs.~\eqref{eq:eq_01} - \eqref{eq:eq_03}. Equivalently,
            \[
            \theta^*_{ij} \in \arg\max_{\theta_{ij}}\mathcal{J}(\theta_{ij}).
            \]
            \label{theo: the_01}
        \end{theorem}
        \vspace{-1.5em}
        \begin{proof}
            \renewcommand{\qedsymbol}{}
             See Appendix \ref{proofs}.
        \end{proof}
    \end{redlinebox}
    Moreover, our guarantees extend beyond the idealized setting where the MNI criterion holds. 
    This is essential, as real multimodal data rarely admit such a clean separation: shared and modality-specific factors can be inherently entangled. Therefore, when MNI is unattainable, we show that the maximum value of our joint objective is no worse than the optimal value of the ideal problem, up to a tractable constant, as stated in Theorem \eqref{theo: the_02}. Further discussion of MNI is provided in Appendix~\ref{information_theory_background}.
    % still satisfies near-optimality guarantees. Specifically, Theorem \ref{theo: the_02} states that for any maximizer of the joint objective $\hat{\theta}_{ij}$, the achieved objective value is lower bounded by the ideal step-by-step optimum up to an additive error of \(2(1+\lambda)\delta_c\).
    \begin{redlinebox}
        \begin{theorem}
            \textit{Near-optimality under unattainable MNI: }
            Let $\theta^*_{ij} = (u^*_{ij}, u^*_{ji}, s^*_{ij}, s^*_{ji})$ be an optimal solution of the ideal step-by-step optimizer as stated in Eqs. \eqref{eq:eq_04} - \eqref{eq:eq_05}, and $\hat{\theta}_{ij}$ the corresponding optimal solution of the joint optimization problem, as stated in Eqs. \eqref{eq:eq_01} - \eqref{eq:eq_03}. If MNI is unattainable, any maximizer $\hat{\theta}_{ij}$ of the joint objective satisfies
            \[ 
            \mathcal{J}(\hat{\theta}_{ij}) \geq \mathcal{J}^o(\theta^*_{ij}) - 2 (1 + \lambda) \delta_c,
            \]
            where, 
            \[
            \mathcal{J}^o(\theta^*_{ij}) = \alpha \mathcal{L}^o_{s_i}(\theta^*_{ij}) + \alpha \mathcal{L}^o_{s_j}(\theta^*_{ij}) + \mathcal{L}^o_{u_i}(\theta^*_{ij}) + \mathcal{L}^o_{u_j}(\theta^*_{ij}),
            \]
            and $\delta_c$ denotes the information-gap term introduced in \citet{wang_information_2025}, which quantifies the trade-off between coverage and redundancy in the shared representations.
            \label{theo: the_02}
        \end{theorem}
        \begin{proof}
            \renewcommand{\qedsymbol}{}
             See Appendix \ref{proofs}.
        \end{proof}
    \end{redlinebox}
    
    Together, Theorems~\eqref{theo: the_01} and~\eqref{theo: the_02} establish that our coupled objective is not merely a practical surrogate, but a formally grounded solution for the ideal step-by-step solution. This is not just a convenience, but rather a prerequisite for scalability beyond two modalities, where a sequential treatment of all shared and unique components becomes computationally prohibitive. By reconciling joint optimization with formal guarantees under both attainable and unattainable MNI, we offer a robust solution that scales without sacrificing theoretical integrity.
    
\section{Experiments} 
    In this section, we evaluate RePercENT across a challenging synthetic setup, and three real-world benchmarks, spanning five datasets and nine modalities in total. Our analysis focuses on three core axes: i) the trade-off between representation recovery and scalability as the number of modalities grows, ii) the quality and downstream utility of the resulting decomposition relative to state-of-the-art alignment-based and disentanglement approaches, and iii) robustness to missing modalities.\looseness=-1
    
    \subsection{Experimental setup} 
     % For each sample, 
    \paragraph{Datasets}
    For the synthetic experiments, we generate each \(Z_i \in \mathbb{R}^{E}\), by sampling independently from a normal distribution all the relevant atomic representations \(z_A \sim \mathcal{N}(\mathbf{0},\sigma^2 I_{d_\ell}), ~i \in A\). Unlike prior work~\citep{liang_factorized_2023, wang_information_2025}, we further map each \(Z_i\) into a sequence representation rather than compressing it into a single pooled vector, mimicking how information is distributed in realistic settings, e.g. across image patches or text tokens.
    This construction provides direct access to the ground-truth unique and shared representations \(\mathbf{u}_{ij}\) and \(\mathbf{s}_{ij}\), allowing us to conduct a linear probing experiment. Specifically, we associate each ground-truth $\mathbf{u}_{ij}$ and $\mathbf{s}_{ij}$, with binary labels $y_{u_{ij}}$ and $y_{s_{ij}}$ respectively, via a non-linear deterministic projection. The learned representation $\hat{\mathbf{u}}_{ij}$, should accurately predict $y_{u_{ij}}$, while remaining uninformative about $y_{s_{ij}}$. Analogously, $\hat{\mathbf{s}}_{ij}$ should exhibit the opposite pattern. $\Delta_{\mathrm{model}}$ averages the absolute deviations of these accuracies from their ideal values (see Appendix~\ref{synthetic_data_creation} for further details). Regarding real-world datasets, we provide an overview along with the considered downstream tasks in Table~\ref{table: datasets}.

    \begin{table*}[h]
        \centering
        \footnotesize
        \setlength{\tabcolsep}{3pt}
        \renewcommand{\arraystretch}{1}
        \caption{Real-world benchmarks. Modalities: T = \texttt{Text}, V = \texttt{Video}, A = \texttt{Audio}, I = \texttt{Image}, C = \texttt{Caption}, D = \texttt{Definition}, WSI = \texttt{Whole Slide Image}, Mol = \texttt{Molecular}, Path = \texttt{Pathology} reports, Clin = \texttt{Clinical} records. $^{\dagger}$Patients, train / test split. $^{\ddagger}$Following \citet{yosef2023irfl}; see Appendix~\ref{irfl_overview}.}
        \label{table: datasets}
        \begin{tabularx}{\textwidth}{@{}
        >{\raggedright\arraybackslash}p{3cm}
        >{\hsize=1.2\hsize\raggedright\arraybackslash}X
        >{\raggedright\arraybackslash}p{1.4cm}
        r
        >{\hsize=0.8\hsize\raggedright\arraybackslash}X@{}}
        \toprule
        \textbf{Dataset} & \textbf{Downstream Task} & \textbf{Modalities} & \textbf{\# Samples} & \textbf{Preprocessing} \\
        \midrule
        \textbf{MOSI}     & Binary sentiment classification (YouTube videos)   & T, V, A              & 2{,}199  & MultiBench \newline \citep{liang_multibench_2021} \\
        \textbf{MOSEI}    & Binary emotion classification (monologue videos)   & T, V, A & 23{,}000 & MultiBench \\
        \textbf{UR-FUNNY} & Binary humor classification (TED talks)& T, V, A              & $>$16{,}000 & MultiBench \\
        \addlinespace[3pt]
        \textbf{IRFL}     & Four-way image choice for a figurative phrase (see Figure \ref{fig: irfl_detection_task})& I, C, D  & 3{,}380 & Ours$^{\ddagger}$ \\
        \addlinespace[3pt]
        \textbf{TCGA} \newline \citep{weinstein2013cancer}     & Cancer-type classification \newline (10 cancer types) & WSI, Mol, \newline Clin, Path & 4{,}585 / 1{,}147$^{\dagger}$ & HONeYBEE \newline \citep{Tripathi2025.04.22.25326222} \\
        \bottomrule
        \end{tabularx}
    \end{table*}
    
    \paragraph{Baselines} 
    For the synthetic experiments, we compare RePercENT, under a shared training regime, against the architectures used by leading disentanglement approaches, each with separate representation-specific encoders, including MLP, gMLP~\citep{liu2021pay}, and GRU~\citep{chung_empirical_2014} variants. For IRFL, we additionally include CLIP baselines, trained with the standard contrastive objective by~\citet{pmlr-v139-radford21a}, offering a direct comparison between alignment-based approaches and disentangled decomposition. As most existing disentanglement methods are limited to two modalities, we compare against DisentangledSSL~\citep{wang_information_2025}, FOCAL~\citep{liu_focal_2023}, JointOpt~\citep{pan_disentangled_2020}, and FactorCL~\citep{liang_factorized_2023} on the MultiBench Text–Video pairs. In the oncology setting, we use the original HONeYBEE embeddings, assessing whether our decomposition improves the utility of existing biomedical representations.
    Finally, for missing-modality robustness on WSIs and molecular data, we consider three fusion strategies: late fusion with averaging, early fusion trained with either mean imputation or modality dropout. These are applied both to the raw embeddings and ClipSur, our parameter-matched baseline of per-modality transformer encoders aligned with the standard CLIP contrastive loss.\looseness=-1
    
    \begin{table}[ht!]
    \centering
    \footnotesize
    \setlength{\tabcolsep}{6pt}
    \renewcommand{\arraystretch}{1}
    \caption{
    Linear probe accuracy (\%) on MultiBench, reported as (mean $\pm$ std.) across three random seeds. For each model we present the best result among the latent representations: unique, shared and the concatenation of both. For each dataset, the best result among the reported rows is shown in bold.\looseness=-1
    }
    \label{fig:linear_probe_multibench}
    \begin{adjustbox}{max width=\linewidth, max totalheight=\textheight}
    \begin{tabular}{lccccc}
    \toprule
    \textbf{Model} & \textbf{MOSI} ($\uparrow$) & \textbf{MOSEI} ($\uparrow$) & \textbf{UR-FUNNY} ($\uparrow$) & \textbf{Params} ($\downarrow$) & \textbf{FLOPs} ($\downarrow$)\\
    \midrule
    FactorCL & 
    52.87 $\pm$ 0.96 & 75.71 $\pm$ 2.83 & 58.76 $\pm$ 1.87 & 21.98M & 3.789B\\
    FOCAL & 
    60.93 $\pm$ 0.36 & 77.83 $\pm$ 0.63 & 62.67 $\pm$ 0.94 & 17.68M & 1.674B\\
    JointOpt & 
    61.42 $\pm$ 2.56 & 77.40 $\pm$ 0.35 & 62.04 $\pm$ 0.47 & 17.68M & 1.675B\\
    DisentangledSSL & 
    \textbf{65.35 $\pm$ 1.48} & 74.33 $\pm$ 0.59 &61.56 $\pm$ 0.88 & 17.68M & 0.835B\\
    \rowcolor{mutedblue}
    \textbf{RePercENT} & 
    62.59 $\pm$ 0.83 & \textbf{77.99 $\pm$ 0.32} & \textbf{62.95 $\pm$ 0.87} & \textbf{7.34M} & \textbf{0.052B}\\
    \bottomrule
    \end{tabular}
    \end{adjustbox}
    \end{table}
    
  \subsection{Results}
    \paragraph{Scalability and component recovery}
    Figure~\ref{Figure03: results_synthetic} reports \(\Delta_{model}\), w.r.t. parameter count as the number of modalities grows. Multi-encoder baselines scale by adding one encoder per derived representation, with fixed per-component capacity. RePercENT instead adds latent slots and only one disentanglement module per new modality, making it the only method with linear scaling in $M$ while maintaining competitive disentanglement performance. The \(M=2\) confusion matrix further shows strong intra-component prediction and limited cross-component leakage, indicating effective
    \begin{minipage}{0.57\textwidth}
       separation of unique and shared components. A detailed breakdown in Appendix~\ref{synthetic_dataset_experiments} shows that only RePercENT and gMLP recover all components, while GRU and MLP suffer representation collapse as \(M\) increases. These results confirm that successful component recovery depends on both the objective and the architecture. Similar conclusions are drawn for MultiBench, when observing Table~\ref{fig:linear_probe_multibench}, where RePercENT matches or outperforms the baselines using $\mathbf{58.5-66.6\%}$ less parameters and $\mathbf{93.8-98.6\%}$ less compute (FLOPs). Notably, in the three-modality regime (Text-Video-Audio), where the other methods are not applicable, RePercENT's parameters are even lower than all baselines at M=2, as shown in Table~\ref{tab:linear_probe_multibench_3m}, Appendix~\ref{additional_results_multibench}. Accordingly on IRFL, Table~\ref{tab:model_complexity_irfl} in Appendix \ref{scaling_complexity_irfl} shows that gMLP requires roughly \(2\times\) more parameters and \(7\times\) more FLOPs than RePercENT without improving figurative language detection performance.

    \paragraph{Downstream performance}
    Beyond scaling, we investigate under what conditions, and to what extent, such a granular decomposition is beneficial for downstream tasks. On IRFL (Table~\ref{tab:irfl_results}), disentanglement generally improves over fine-tuned CLIP variants, and RePercENT achieves the best overall and out-of-distribution (OoD) accuracy on the figurative language task. Comparing our method with alignment approaches, we observe that isolating the shared component yields superior performance,  especially for idioms and metaphors, but is not ideal for similes. This is consistent with the linguistic nature of the task, as similes are typically more explicit and structurally constrained, favoring direct cross-modal alignment. By contrast, our principled decomposition appears to benefit \looseness=-1
    \end{minipage}
    \hfill
    \begin{minipage}{0.41\textwidth}
        \centering
        \includegraphics[width=0.98\linewidth]{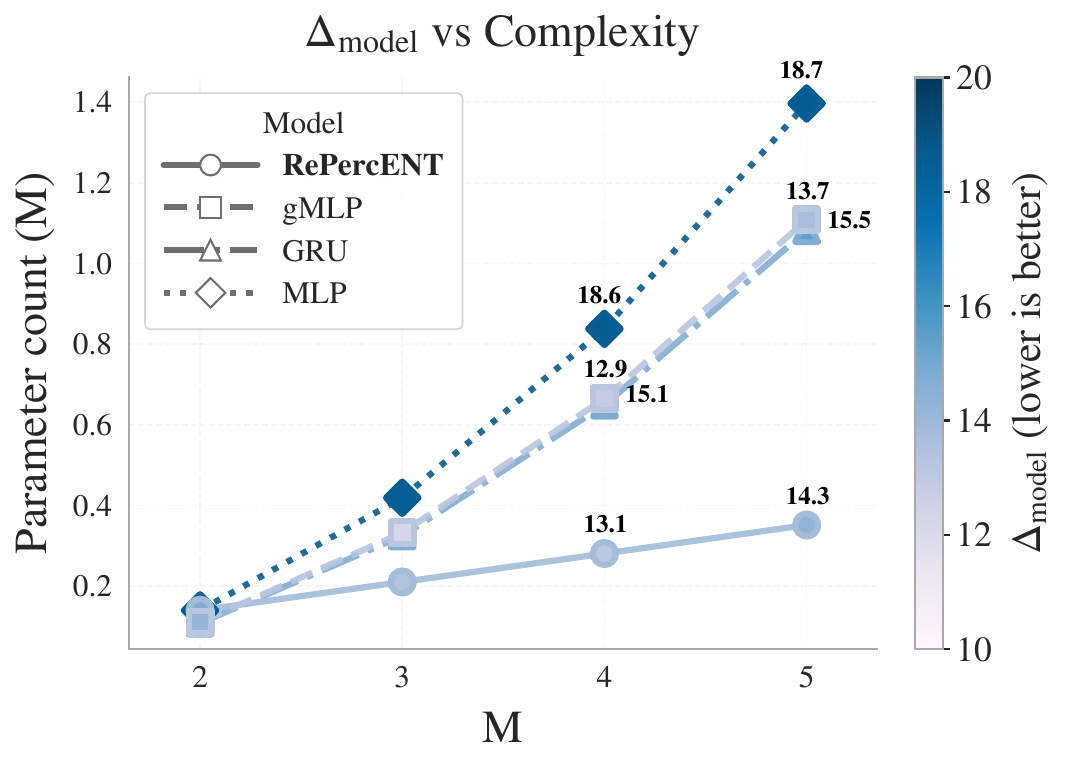}
    
        % \vspace{0.3em}
        \includegraphics[width=1\linewidth]{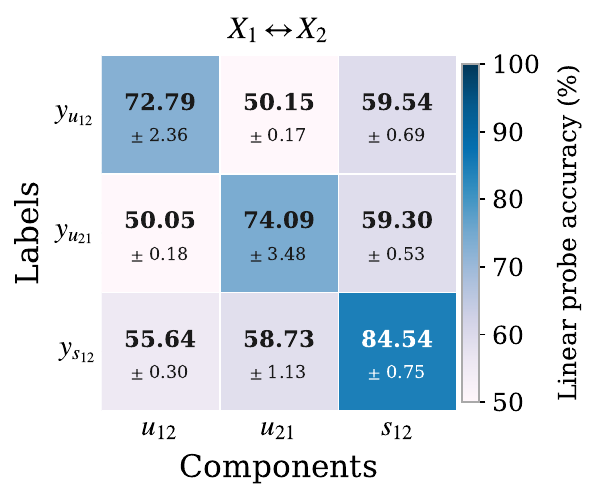}
    
        \captionof{figure}{
        \textbf{Top:} Synthetic performance across modality count and parameter budgets. RePercENT achieves competitive scores at substantially lower model complexity.
        \textbf{Bottom:} Linear-probe confusion matrix for \(M=2\).
        }
        \label{Figure03: results_synthetic}
    \end{minipage}
     idioms and mainly metaphors since they are more abstract and context-dependent. The ablation further supports RePercENT's key structural mechanisms: removing both semantic encodings and group slot attention gives the weakest performance, while enabling either components improves accuracy. Using both results in the best performance, improving overall accuracy from $40.30$ to $44.07$, with large gains on idioms ($28.2 \rightarrow 33.60$) and metaphors ($28.87 \rightarrow 35.56$). 
    \begin{table}[h!]
    \centering
    \caption{Top-1 accuracy (\%)on the IRFL detection task, reported as (mean \(\pm\) std) over 5 random seeds.\looseness=-1}
    \label{tab:irfl_results}
    \footnotesize
    \setlength{\tabcolsep}{6pt}
    \renewcommand{\arraystretch}{1}
    \begin{adjustbox}{max width=\linewidth, max totalheight=\textheight}
    \begin{tabular}{lcccccc}
        \toprule
        \textbf{Model} & \textbf{SE} & \textbf{GSA} & \textbf{Idioms} & \textbf{Simile} & \textbf{Metaphor--\scriptsize OoD} & \textbf{Overall} \\
        \midrule

        \multicolumn{7}{l}{\textit{Baselines}} \\
        CLIP-ViT-B/32 \scriptsize{(zero-shot)}
            & -- & --
            & 16.00
            & 45.49
            & 23.42
            & 29.14 \\
        CLIP-ViT-B/32 \scriptsize{(projection-only FT)}
            & -- & --
            & $20.4 \pm 1.0$
            & $68.38 \pm 1.1$
            & $31.59 \pm 1.0$
            & $41.41 \pm 0.7$ \\
        CLIP-ViT-B/32 \scriptsize{(end-to-end FT)}
            & -- & --
            & $22.5 \pm 1.8$
            & $\mathbf{70.18 \pm 0.6}$
            & $29.61 \pm 1.0$
            & $41.73 \pm 1.0$ \\
        gMLP
            & -- & --
            & $33.70 \pm 2.4$
            & $54.37 \pm 2.1$
            & $28.23 \pm 2.4$
            & $38.52 \pm 0.7$ \\
        GRU
            & -- & --
            & $\mathbf{36.10 \pm 3.1}$
            & $48.81 \pm 3.0$
            & $33.69 \pm 1.9$
            & $39.46 \pm 1.1$ \\
        \midrule
        \multicolumn{7}{l}{\textit{Architectural ablations}} \\
        RePercENT
            & \xmark & \xmark
            & $28.20 \pm 5.0$
            & $63.97 \pm 2.0$
            & $27.87 \pm 3.96$
            & $40.30 \pm 1.7$ \\
        RePercENT
            & \xmark & \tmark
            & $30.20 \pm 4.3$
            & $63.11 \pm 2.9$
            & $30.45 \pm 3.38$
            & $41.56 \pm 1.34$ \\
        RePercENT
            & \tmark & \xmark
            & $32.60 \pm 2.2$
            & $61.88 \pm 0.9$
            & $31.95 \pm 5.7$
            & $42.35 \pm 2.24$ \\
        \rowcolor{mutedblue}
        \textbf{RePercENT}
            & \tmark & \tmark
            & $33.60 \pm 5.3$
            & $61.88 \pm 3.1$
            & $\mathbf{35.56 \pm 3.4}$
            & $\mathbf{44.07 \pm 0.8}$ \\

        \bottomrule
        \end{tabular}
        \end{adjustbox}
        \vspace{-1.5em}
    \end{table}

    Likewise, RePercENT's per-component results on MultiBench (Table~\ref{tab:linear_probe_multibench_full}, Appendix~\ref{additional_results_multibench}) mirror how label information is distributed across modalities. Unique text is strongest on MOSI/MOSEI, where sentiment is conveyed mainly through language, whereas the shared component is strongest on UR-FUNNY, as laughter depends on delivery as well as content. Together, these results reflect the inherent structure of the data and show that task-relevant information is not only shared across modalities.\looseness=-1
    
    On TCGA (Table~\ref{tab:cancer_type_classification}) this fine-grained decomposition improves accuracy in cancer type classification over HONeYBEE embeddings for most cancer types, with the largest gains observed on molecular data. Most strikingly, for KIRC and LGG, where raw molecular embeddings are nearly uninformative, it recovers substantially higher accuracy, indicating that cross-modal conditioning exposes predictive structure hidden in the original embedding space. Notably, the full decomposition $D_{ij}$ consistently outperforms both $U_{ij}$ and $S_{ij}$ alone, further challenging the assumption of multi-view redundancy.
    \vspace{-1em}
    \begin{table}[h]
        \centering
        \caption{Cancer type prediction accuracy (\%, mean $\pm$ std) on TCGA cohorts. HONeYBEE embeddings vs.\ RePercENT components of modality $i$ relative to $j$: unique $U_{ij}$, shared $S_{ij}$, and complete $D_{ij}=U_{ij}\oplus S_{ij}$. $\Delta$\textbf{RePercENT} denotes the absolute gains over HONeYBEE.}
        \label{tab:cancer_type_classification}
        \setlength{\tabcolsep}{3pt}
        \renewcommand{\arraystretch}{1.05}
        \resizebox{\textwidth}{!}{%
        \begin{tabular}{lcccccccccc}
            \toprule
            \textbf{Representation} & \textbf{BRCA} & \textbf{COAD} & \textbf{GBM} & \textbf{HNSC} & \textbf{KIRC} & \textbf{LGG} & \textbf{LUAD} & \textbf{LUSC} & \textbf{OV} & \textbf{PRAD} \\
            \midrule
            \multicolumn{11}{l}{\textit{WSI}} \\
            Honeybee
            & \underline{95.4} & 32.2 & \underline{41.5} & 18.3 & 48.6 & 22.6 & 26.0 & 14.3 & 36.2 & 29.0 \\
            $U_{\mathrm{wsi},\mathrm{mol}}$
            & 79.2 $\pm$ 1.8 & \underline{39.3 $\pm$ 3.7} & 39.8 $\pm$ 2.5 & 28.5 $\pm$ 3.2 & 61.3 $\pm$ 4.2 & 28.8 $\pm$ 2.6 & 24.2 $\pm$ 2.5 & 22.7 $\pm$ 1.3 & 47.8 $\pm$ 3.1 & 37.8 $\pm$ 1.9 \\
            $S_{\mathrm{wsi},\mathrm{mol}}$
            & 81.6 $\pm$ 1.5 & 37.2 $\pm$ 4.8 & 39.8 $\pm$ 2.0 & 26.7 $\pm$ 3.0 & 61.1 $\pm$ 2.3 & 24.3 $\pm$ 2.3 & 26.0 $\pm$ 2.6 & 20.8 $\pm$ 3.1 & 46.9 $\pm$ 2.5 & 34.8 $\pm$ 2.8 \\
            $D_{\mathrm{wsi},\mathrm{mol}}$
            & 79.3 $\pm$ 1.8 & 39.1 $\pm$ 2.7 & 40.2 $\pm$ 1.5 & \underline{28.7 $\pm$ 2.2} & \underline{63.4 $\pm$ 2.1} & \underline{31.0 $\pm$ 2.8} & \underline{27.8 $\pm$ 2.2} & \underline{23.3 $\pm$ 1.5} & \underline{51.2 $\pm$ 1.4} & \underline{38.8 $\pm$ 2.8} \\
            \cellcolor{mutedblue}
            \textbf{$\Delta$ RePercENT}
            & \cellcolor{epflred!8}\textbf{-13.8}
            & \cellcolor{mutedblue}\textbf{+7.1}
            & \cellcolor{epflred!8}\textbf{-1.3}
            & \cellcolor{mutedblue}\textbf{+10.4}
            & \cellcolor{mutedblue}\textbf{+14.8}
            & \cellcolor{mutedblue}\textbf{+8.4}
            & \cellcolor{mutedblue}\textbf{+1.8}
            & \cellcolor{mutedblue}\textbf{+9.0}
            & \cellcolor{mutedblue}\textbf{+15.0}
            & \cellcolor{mutedblue}\textbf{+9.8} \\
            \midrule
    
            \multicolumn{11}{l}{\textit{Molecular}} \\
            Honeybee
            & 69.8 & 21.8 & \underline{100.0} & 31.7 & 2.8 & 2.9 & \underline{58.0} & 51.0 & 65.5 & 74.0 \\
            $U_{\mathrm{mol},\mathrm{clin}}$
            & 67.9 $\pm$ 0.9 & 32.9 $\pm$ 3.7 & 98.1 $\pm$ 0.4 & 66.7 $\pm$ 1.9 & 60.6 $\pm$ 2.9 & 50.0 $\pm$ 2.1 & 51.0 $\pm$ 3.4 & 65.3 $\pm$ 3.8 & \underline{80.3 $\pm$ 1.4} & 91.0 $\pm$ 0.7 \\
            $S_{\mathrm{mol},\mathrm{clin}}$
            & 67.5 $\pm$ 2.9 & 33.6 $\pm$ 2.7 & 98.3 $\pm$ 0.0 & 66.0 $\pm$ 1.5 & 60.4 $\pm$ 3.1 & 49.2 $\pm$ 1.6 & 50.8 $\pm$ 4.3 & 66.7 $\pm$ 1.7 & 78.5 $\pm$ 2.9 & 91.8 $\pm$ 0.5 \\
            $D_{\mathrm{mol},\mathrm{clin}}$
            & \underline{70.0 $\pm$ 1.4} & \underline{37.5 $\pm$ 3.1} & 98.3 $\pm$ 0.0 & \underline{67.9 $\pm$ 1.1} & \underline{62.2 $\pm$ 1.6} & \underline{52.8 $\pm$ 1.8} & 55.4 $\pm$ 5.2 & \underline{72.0 $\pm$ 2.8} & 79.5 $\pm$ 1.4 & \underline{92.0 $\pm$ 0.0} \\
            \cellcolor{mutedblue}
            \textbf{$\Delta$ RePercENT}
            & \cellcolor{mutedblue}\textbf{+0.2}
            & \cellcolor{mutedblue}\textbf{+15.7}
            & \cellcolor{epflred!8}\textbf{-1.7}
            & \cellcolor{mutedblue}\textbf{+36.2}
            & \cellcolor{mutedblue}\textbf{+59.4}
            & \cellcolor{mutedblue}\textbf{+49.9}
            & \cellcolor{epflred!8}\textbf{-2.6}
            & \cellcolor{mutedblue}\textbf{+21.0}
            & \cellcolor{mutedblue}\textbf{+14.8}
            & \cellcolor{mutedblue}\textbf{+18.0} \\
            \bottomrule
        \end{tabular}}
        \vspace{-1.5em}
    \end{table}
    \paragraph{Measuring Disentanglement}
    In real-world settings, downstream performance indicates the utility and semantic information of the learned disentangled representations. Beyond this empirical evaluation, we argue that a faithful decomposition should satisfy three testable properties. (i) \textbf{Shared-subspace invariance:} the shared components of both modalities should span a consistent subspace, assessed through similarity metrics such as CKA and mutual-kNN alignment~\citep{huh_platonic_2024}. (ii) \textbf{Cross-modal exclusion:} a unique component, $u_{ij}$ should contain minimal information about the partner modality $j$, evaluated through its ability to reconstruct that modality. (iii) \textbf{Complementarity:} $u_{ij}$ and $s_{ij}$ should be more informative jointly than individually, measured via the Reconstruction Gain (RG)~\citep{wang_information_2025}. On MultiBench (Table~\ref{tab:disen_metrics}), no model leaks 
    information (all scores $\le 0$). RePercENT significantly outperforms all baselines on both alignment metrics, with similarly good performance in terms of RG.
    \vspace{-1em}
    \begin{table}[h]
        \centering
        \caption{Disentanglement metrics on MOSEI.}
        \footnotesize
        \setlength{\tabcolsep}{6pt}
        \renewcommand{\arraystretch}{1.05}
        \begin{adjustbox}{max width=\linewidth, max totalheight=\textheight}
        \label{tab:disen_metrics}
        \begin{tabular}{lcccccc}
        \toprule
        & \multicolumn{2}{c}{\textbf{Reconstruction Gain} ($\uparrow$)} & \multicolumn{2}{c}{\textbf{Shared Space Alignment} ($\uparrow$)} & \multicolumn{2}{c}{\textbf{Cross-Modal Leakage} ($\leq 0$)}\\
        \cmidrule(lr){2-3}\cmidrule(lr){4-5}\cmidrule(lr){6-7}
        Model & RG-\texttt{Text} & RG-\texttt{Video} & CKA & mKNN & $U_{T, V} \rightarrow X_V $ & $U_{V, T} \rightarrow X_T$ \\
        \midrule
        FOCAL & $0.027_{\pm.001}$ & $0.025_{\pm.001}$ & $0.117_{\pm.012}$ & $0.033_{\pm.002}$ & $-0.011_{\pm.001}$ & $-0.006_{\pm.002}$ \\
        JointOpt & $0.060_{\pm.017}$ & $0.034_{\pm.003}$ & $0.108_{\pm.012}$ & $0.040_{\pm.002}$ & $-0.021_{\pm.003}$ & $-0.020_{\pm.006}$  \\
        DisentangledSSL & $0.015_{\pm.002}$ & $0.054_{\pm.010}$ & $0.007_{\pm.001}$ & $0.019_{\pm.0004}$ & $-0.001_{\pm.005}$ & $-0.005_{\pm.002}$\\
        \rowcolor{mutedblue}
        \textbf{RePercENT} & $\mathbf{0.076}_{\pm.004}$ & $\mathbf{0.076_{\pm.005}}$ & $\mathbf{0.838_{\pm.010}}$ & $\mathbf{0.356_{\pm.022}}$ & 
        $\mathbf{-0.003_{\pm.003}}$ & 
        $\mathbf{-0.014_{\pm.001}}$ \\
        \bottomrule
        \end{tabular}
        \end{adjustbox}
    \end{table}
    
    \begin{minipage}[h]{0.5\textwidth}
    \paragraph{Robustness to missing modalities} 
   Figure~\ref{fig:honeybee_missingness} examines the robustness of different models in terms of cancer type prediction, as the fraction of available \texttt{WSI} samples decreases. All fusion baselines exhibit a clear performance drop, including those explicitly trained with modality dropout.
   This suggests that simply accounting for missing modalities during training is insufficient when the models still rely on jointly combining modalities at inference. 
   In contrast, RePercENT maintains substantially higher Macro-F1 across all missingness levels. This behavior follows directly from its structure (recall Eq.~\ref{eq: markov_structure}), which allows, both the unique information of the \texttt{Molecular} w.r.t. the \texttt{WSI} as well as the shared between the two modalities, to remain fully accessible. 
   % Therefore, missing modalities do not collapse the full inference pathway, but only remove the information rooted in the absent modality. 
   Notably, whenever \texttt{WSI} information is available, RePercENT improves over both single-modality references, further supporting the argument that there is significant complementary information between the two modalities.
    \end{minipage}
    \hfill
    \begin{minipage}{0.48\textwidth}
        \centering
        \includegraphics[width=\linewidth]{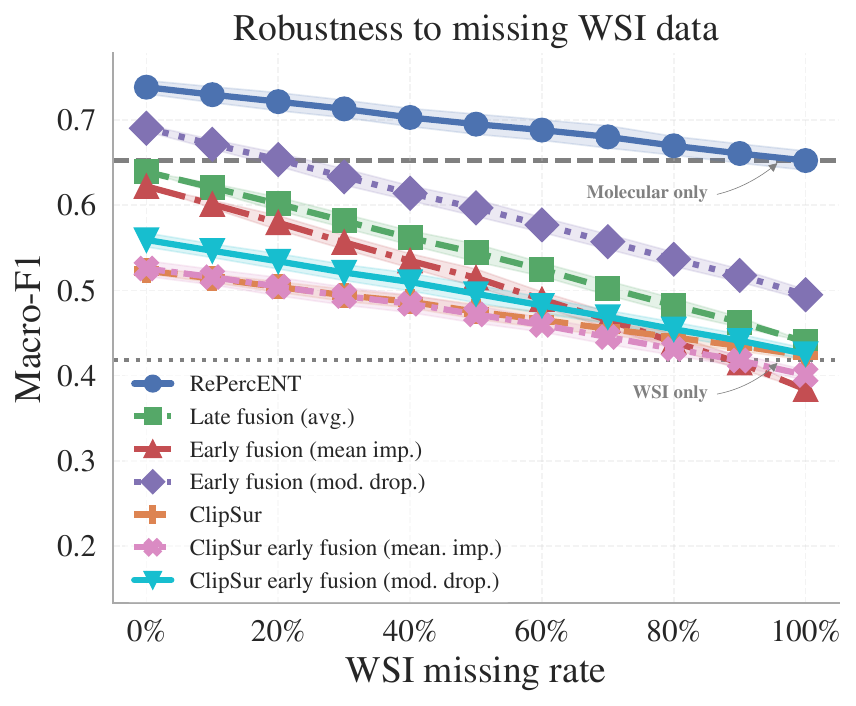}
        \captionof{figure}{We reduce \texttt{WSI} availability while preserving \texttt{Molecular} data. Fusion baselines degrade sharply, whereas RePercENT remains robust.\looseness=-1}
        \label{fig:honeybee_missingness}
    \end{minipage}
    \section{Limitations and future directions}
    \label{limiations_and_future_work}
    RePercENT establishes a scalable and theoretically grounded framework for disentanglement across multiple modalities, inviting further exploration in several promising research directions. Our formulation focuses on pairwise unique and shared components, which provide a provably sufficient granularity for modeling the interactions considered in this work. While this choice comes with favorable scaling and formal guarantees, an important direction for future work is modeling richer, higher order interactions among three or more modalities through more expressive decompositions.

    In addition, our practical implementation depends on the quality of the pre-extracted modality embeddings. While foundation models provide strong representations and enable a plug-and-play design, they also introduce a trade-off between expressivity and compression ~\citep{tishby_information}. This motivates future work on end-to-end training, in settings with sufficient data, which could yield richer and more informative representations. Another promising direction is to dynamically balance the objectives across modality pairs, potentially improving optimization efficiency and the quality of the learned disentangled components. Finally, the structured nature of the learned decomposition opens an exciting path toward more interpretable multimodal models, motivating future work to examine how these components encode, exchange, and use information to shape the model's predictions for different downstream tasks.
    \vspace{-0.25em}  
    \section{Conclusion}
    As multimodal intelligence aims to derive representations that capture the world's complex multimodal reality, it becomes essential to decouple the nuanced interactions across diverse modalities. In this work, we introduce RePercENT, and demonstrate that, in settings involving more than two modalities, disentangled representations i) benefit downstream tasks, ii) can be learned computationally efficiently, while iii) remaining theoretically grounded. Through an extensive validation of our framework across synthetic, figurative language, and complex biomedical benchmarks, we demonstrate competitive performance across diverse modalities, foundation model backbones and tasks, while scaling seamlessly as the number of modalities grows. Ultimately, we believe that Re\textit{Perc}ENT, via its flexible structure and granular representations, will help pave the way towards a more comprehensive, multimodal \textit{Perc}eption.
    \vspace{-0.25em}    
    \section{Acknowledgements}
    This work was supported by the Swiss National Science Foundation (SNSF) under Lead Agency grant 200021E\_229452. We thank the anonymous reviewers for their insightful feedback and constructive discussion.

\newpage

\bibliographystyle{unsrtnat}
\bibliography{references}

\clearpage

%%%%%%%%%%%%%%%%%%%%%%%%%%%%%%%%%%%%%%%%%%%%%%%%%%%%%%%%%%%%

\appendix

\startcontents[appendix]
\section*{Appendix Contents}
\printcontents[appendix]{}{1}{\setcounter{tocdepth}{2}}

\setcounter{theorem}{0}

\clearpage

\section{Useful Notation}
\label{appendix_0}

\begin{table}[h]
    \centering
    \small 
    \caption{Notation summary.}
    \label{tab:notation}
    \begin{tabularx}{\textwidth}{lX} % X column automatically wraps and fills width
    \toprule
    \textbf{Symbol} & \textbf{Meaning} \\
    \midrule
    \rowcolor{gray!10} \multicolumn{2}{l}{\textit{Sets and Indices}} \\
    $M$ & Total number of modalities \\
    $\mathcal{M}$ & Modality index set $\{1,\dots,M\}$ \\
    $\mathcal{E}$ & Set of all unique modality pairs $\{(i,j) \in \mathcal{M}^2 : i < j\}$ \\
    $A \subseteq \mathcal{M}$ & A non-empty subset of modalities \\
    $\mathcal{A}^{u}_{ij}, \mathcal{A}^{s}_{ij}$ & Index sets for atomic components unique to $i$ (w.r.t $j$) or shared by $\{i, j\}$ \\
    
    \midrule
    \rowcolor{gray!10} \multicolumn{2}{l}{\textit{Representations and Components}} \\
    $X_i$ & Input modality $i$ \\
    $Z_i$ & Composite latent representation for modality $i$ \\
    $z_A$ & Atomic representation associated with subset $A$ \\
    $\mathbf{u}_{ij}, \mathbf{s}_{ij}$ & Unique and shared latent components for modality $i$ relative to modality $j$ \\
    $\theta_{ij}$ & Optimization variable as a collection of all pairwise components $(u_{ij},u_{ji},s_{ij},s_{ji})$ \\
    
    \midrule
    \rowcolor{gray!10} \multicolumn{2}{l}{\textit{Architectural Elements}} \\
    $\mathcal{D}_i$ & Disentangling module for modality $i$ \\
    $\mathcal{H}_i$ & Latent-slot array ($\mathbb{R}^{N\times D_l}$) \\
    $h_{ik}$ & The $k$-th individual row of $\mathcal{H}_i$ \\
    $\phi_i$ & Mapping function from latent slots to pairwise components for modality $i$\\
    
    \midrule
    \rowcolor{gray!10} \multicolumn{2}{l}{\textit{Hyperparameters}} \\
    $\alpha$ & Weighting parameter between unique and shared losses \\
    $\beta$ & Penalty parameter for shared component coverage \\
    $\lambda$ & Penalty parameter for redundancy between unique and shared components \\
    \bottomrule
    \end{tabularx}
\end{table}

\section{Supplementary Related Work}
\label{supplementatry_rel_work}
    Multimodal representation learning has revolutionized the way we process and integrate information coming from different modalities~\citep{Liang2022FoundationsT, Krones2024ReviewOM}. 
    Recent advances in multimodal foundation models~\citep{team2023gemini, Li2022BLIPBL, Singh2021FLAVAAF} reflect a growing trend toward flexible, general-purpose AI systems. However, these models typically require large amounts of paired data and extensive pre-training, making them difficult to apply in domains, such as medicine, where paired multimodal data are limited.

    % A notable example is clinical practice, where despite recent progress in biomedical foundation models~\citep{Moor2023MedFlamingoAM, Lee2019BioBERTAP}, large-scale paired pre-training remains difficult due to costly, privacy-sensitive, and unevenly available data across modalities. 
    
    To address this limitation, several approaches have been proposed, such as STRUCTURE~\citep{Grger2025WithLD} and FuseMix~\citep{Vouitsis2023DataEfficientMF}, that align representations into a shared space using frozen foundation models in low-data regimes. Complementary, training-free approaches have also been explored. For instance, Centered Kernel Alignment (CKA)~\citep{Kornblith2019SimilarityON} provides a way to compare representation spaces without additional training, while ASIF~\citep{Norelli2022ASIFCD} constructs a common multimodal space from paired anchor data. On the other hand, UniAlign by \citet{11095214}, improves scalability by using a single encoder to align diverse modalities. 

    While alignment approaches improve efficiency and exploit the expressive power of foundation models, they mainly capture shared cross-modal information. Fusion methods overcome this \citep{hemker_healnet_2024, Chen2023AGS, Xu2023MultimodalOT}, by leveraging joint representations that additionally encode complementary information across modalities. ImageBind~\citep{Girdhar2023ImageBindOE} is another notable example that learns a joint embedding space for six different modalities, showing that not all modality-pairs need to be observed. Additionally, \citet{liang_high-modality_2023} propose HighMMT that learns efficient-joint representations, by scaling multimodal transformers to many modalities and measuring modality and interaction heterogeneity to guide parameter sharing.

    Despite their success, fusion methods usually result to less interpretable representations, and face a difficulty in handling missing modalities. A promising, recent direction to address this is through disentanglement \citep{10.1109/TPAMI.2024.3420937}, where the different modality-interactions are explicitly encoded in separate representations. For example, DIMAF~\citep{eijpe_disentangled_nodate} models shared and specific components between spatial transcriptomics and histology, improving downstream performance and interpretability. \citet{gao_disentangled_nodate}, DrFuse~\citep{yao_drfuse_2024}, and DRIM~\citep{robinet_drim_2024} further explore disentangled fusion under incomplete-modality settings, while FactorCL~\citep{liang_factorized_2023} formalizes multi-view redundancy and separates shared from unique task-relevant information. CooKIE~\citep{namgung_less_2025} extends FactorCL beyond two modalities for geospatial region representation, replacing explicit high-order factorization with a pairwise inter-view learning approach that reduces complexity in factorized multimodal region representation learning. However, it leaves the resulting cross-pair redundancy to be handled by the downstream predictor and offers no optimality guarantees for its decomposition. More recently, \citet{wang_information_2025} propose information-theoretic criteria for controlled two-modality disentanglement.

\section{Impact Statement}
\label{broader_impacts}
The primary objective of this paper is to advance multimodal representation learning under a flexible framework that extracts disentangled pairwise unique and shared components, with applications spanning figurative language, and biomedical research. Although advances in multimodal representation learning may have both positive and negative societal and ethical implications, particularly in areas like biomedical research, we currently do not foresee any immediate societal concerns associated with the proposed methodology.

\section{Information Theory Background}
\label{information_theory_background}

In this section we include useful definitions, and key background elements on which we build upon our information theory objectives. 

\subsection{Basic quantities and identities}
We begin with some basic notions from classical information theory as well as some key inequalities that have been essential for the proofs that follow.

\paragraph{Mutual information (MI)}

Let \((X, Y)\) be a set of continuous random variables over the space \(\mathcal{X} \times\mathcal{Y}\). The mutual information $I(X;Y)$ is defined as:
\begin{equation*}
    I(X; Y) = \iint_{\mathcal{X} \times \mathcal{Y}} p_{X,Y}(x, y) \log \left( \frac{p_{X,Y}(x, y)}{p_X(x) p_Y(y)} \right) \, dx \, dy = D_{KL}(p_{X, Y} \mathrel{\Vert} p_Xp_Y),
\end{equation*}
where $p_{X, Y}$ is the joint probability density function of $X, Y$, $p_X$ and $p_Y$ are the marginal probability density functions of $X$ and $Y$ respectively, and $D_{KL}$ is the \textit{Kullback-Leibler divergence}. Intuitively, MI measures how much knowing one of these variables reduces uncertainty about the other.

\paragraph{Conditional mutual information (CMI)}

Let \(X, Y\) and \(Z\) be three continuous random variables over the spaces $\mathcal{X}$, $\mathcal{Y}$ and $\mathcal{Z}$. The conditional mutual information (CMI) $I(X;Y \mid Z)$ is denoted as:

\begin{align*}
     I(X; Y) &= \iiint_{\mathcal{X} \times \mathcal{Y} \times \mathcal{Z}} p_{X,Y, Z}(x, y, z) \log \left( \frac{p_{X, Y, Z}(x, y, z) p_Z(z)}{p_{X, Z}(x, z)p_{Y, Z}(y, z)} \right) \, dx \, dy \, dz \\
     &= \int_\mathcal{Z} D_{KL}(p_{(X, Y) \mid Z} \mid \mid p_{X\mid Z} \otimes p_{Y \mid Z}) \, dz,
\end{align*}

where all the conditional, marginal and joint probability distributions are denoted as $p$ with the corresponding subscript. Similarly to the MI, CMI expresses the mutual information between $X$ and $Y$ when conditioning on a third variable $Z$.

\paragraph{Useful properties}

We present below several useful identities \& inequalities, used extensively in the following proofs.
\begin{itemize}
    \item $I(X ; Y) \ge 0$, non-negativity of MI
    \item $I(X;Y \mid Z) \ge 0$, non-negativity of CMI
    \item $I(X;Y) = I(Y;X)$, symmetry
    \item $I(X;Y, Z) = I(X;Y \mid Z) + I(X;Z)$, chain rule of MI
    \item $I(X;Y)-I(X;Z)= I(X;Y\mid Z)-I(X;Z\mid Y)$.
\end{itemize}
The last one is derived, by applying the chain rule of MI in its two equivalent forms.

\subsection{Multi-view redundancy}
We extend the notion of multi-view redundancy, proposed by \cite{liang_factorized_2023}, to the general multimodal setting as follows:
    \begin{bluelinebox}
        \begin{definition}
            \textit{Multi-view redundancy}: 
            Let $\mathcal{X} =\{X_1, X_2, \dots, X_M\}$ be a set of $M$ modalities, and $Y$ a downstream task. Let also, $X_{-i} =\{X_1, \dots,X_{i-1}, X_{i + 1}, \dots X_M\}$ represent the set of all modalities except $X_i$. We say that multi-view redundancy holds, if there exists sufficiently small $\epsilon > 0$, such that: 
            $$ I(X_{-i}; Y \mid X_i) < \epsilon, \quad \forall i \in \mathcal{M}. $$
        \end{definition}
    \end{bluelinebox}

This formulation implies that each $X_i$ captures sufficient information to perform the task $Y$ and any complementary information from the remaining modalities offers insignificant predictive gain.

\subsection{Minimum Necessary Information}

In the simpler unimodal case, assume a dataset \((X, Y)\) with observations $X$ and target labels $Y$. Given the Markov structure $Z \leftarrow X \leftrightarrow Y$, we want to derive representations $Z$ that capture useful information from $X$ that is relevant to task $Y$~\citep{tishby_information}. As introduced by \cite{fischer_conditional_2020}, the \textit{Minimum Necessary Information} (MNI) criterion 
for a representation $Z$ is satisfied, when i) the representation $Z$ attains all the \textit{necessary} information from $X$ to perform $Y$, and ii) given all the possible representations $Z$ that solve the task $Y$, the desired $Z$ needs to be \textit{minimal}, i.e. achieve $\inf_{Z \in \mathcal{Z}}I(Z;X, Y)$. In other words, at the MNI point it holds:
\begin{equation}
    I(X;Y) = I(X;Z) = I(Y;Z).
    \label{eq: unimodal_mni_point}
\end{equation}

Following \cite{wang_information_2025}, we extend this idea to the multimodal self-supervised setting by considering observations \((X_i, X_j)\) from two distinct modalities. We would like to derive the representations $s_{ij}$ and $s_{ji}$, respecting the Markov structures: \(s_{ij} \leftarrow X_i \leftrightarrow X_j\) and $s_{ji} \leftarrow X_j \leftrightarrow X_i$, that optimally balance between expressivity and compression. In other words, $s_{ij}$ is desired to capture all the necessary information from $X_i$ that is shared with $X_j$, while minimizing redundancy with $X_i$. The same holds for $s_{ji}$. Formally, when the MNI point is attainable, $s_{ij}$ and $s_{ji}$ satisfy,
\begin{equation*}
    I(s_{ij};X_i) = I(s_{ij};X_j) = I(X_i;X_j), \qquad I(s_{ji};X_i) = I(s_{ji};X_j) = I(X_i;X_j).
    \label{eq: multimodal_mni_point}
\end{equation*}

\subsection{Information gap under unattainable MNI}

In real world scenarios, the MNI point described in section \ref{eq: multimodal_mni_point} is not always attainable, and therefore \cite{wang_information_2025} propose the following definition for extracting the  optimal shared representations $s_{ij}$ and $s_{ji}$:
\begin{equation}
    \begin{aligned}
        s^*_{ij}
        &\in \arg\min_{s_{ij}} I(s_{ij}; X_i \mid X_j), && \text{s.t. } I(X_i; X_j) - I(s_{ij}; X_j) \leq \delta_c, \\
        s^*_{ji}
        &\in \arg\min_{s_{ji}} I(s_{ji}; X_j \mid X_i), && \text{s.t. } I(X_i; X_j) - I(s_{ji}; X_i) \leq \delta_c .
    \end{aligned}
\label{eq: shared_unattainable_mni}
\end{equation}

Minimizing $I(s_{ij}; X_i \mid X_j)$ penalizes information retained by $s_{ij}$ about $X_i$ that is not explained by $X_j$. The constraint, on the other hand, ensures that \(s_{ij}\) preserves nearly all shared information between \(X_i\) and \(X_j\), up to an information gap \(\delta_c\). The same interpretation applies symmetrically to \(s_{ji}\). Notice that if MNI was always attainable, we could substitute $I(X_i; X_j) - I(s_{ij}; X_j) \leq \delta_c$ with $I(X_i; X_j) = I(s_{ij}; X_j)$, as suggested by \cite{10.1109/TPAMI.2017.2784440}. However, when MNI is unattainable, this leads to a suboptimal solution and therefore the parameter $\delta_c$ is introduced, which controls the gap between capturing the complete shared information and having a compressed representation. Using the Lagrangian formulation of constrained problem in Eq. \eqref{eq: shared_unattainable_mni}, we end up with the first part of the optimal step-by-step optimization objective:
\begin{equation}
    \begin{aligned}
    s^*_{ij} &\in \arg\max_{s_{ij}} \mathcal{L}^{o}_{s_i} \; =\arg\max_{s_{ij}} \; I(s_{ij}; X_j) - \beta\, I(s_{ij}; X_i \mid X_j), \\
    s^*_{ji} &\in \arg\max_{s_{ji}} \mathcal{L}^{o}_{s_j} \; =\arg\max_{s_{ji}} \; I(s_{ji}; X_i) - \beta\, I(s_{ji}; X_j \mid X_i).
    \end{aligned}
    \label{eq:eq_04_repeat}
\end{equation}
Here, \(\beta>0\) controls the trade-off between preserving cross-modal shared information and removing modality-specific information. As shown by \cite{wang_information_2025}, when MNI is attainable, then for any positive $\beta$, maximizing the objective $\mathcal{L}^o_{i}$ achieves MNI, and there exists a bijective mapping from $\beta$ in $\mathcal{L}^o_{s_i}$ to the value of the information constraint $\delta_c$ as defined in Eq. \eqref{eq: shared_unattainable_mni}.

 % ves, that optimally balance relevance and redundancy in the shared representation, while promoting coverage and disentanglement in the modality-specific representation. More precisely, for a modality pair \(i, j\), then each observed pair of the two modalities \(X_i, X_j\), is generated from two underlying modality-specific representations $u_{ij}$ and $u_{ji}$ respectively, that contain information unique to each modality, as well as a shared representation $s$ that is common to both. During inference, we assume the Markov structures, \(\hat{s}_{ij} \leftarrow X_i \leftrightarrow X_j\) and \(\hat{s}_{ji} \leftarrow X_j \leftrightarrow X_i\), i.e. we extract both the shared and unique representations from both modalities independently. 

\section{Proofs}

\label{proofs}

We first derive the following Lemma \ref{lemma: 01} under attainable MNI, that will help us with the proof of Theorem \ref{theo: the_01}.

\begin{bluelinebox}
    \begin{lemma}
        Let ($s^*_{ij}, s^*_{ji}$) be an optimal solution of the step-by-step optimization problem as defined in Eq.~\eqref{eq:eq_04}. Under attainable MNI and $u_{ij} \perp (X_j, s^*_{ij}) \mid X_i$,
        \begin{equation*}
            I(u_{ij};s^*_{ij}) = I(u_{ij};X_j) = I(u_{ij};s^*_{ji}),
        \end{equation*}
        for every feasible $u_{ij}$.
        \label{lemma: 01}
    \end{lemma}
\end{bluelinebox}

\begin{proof}
    Using the standard mutual-information identity, we get

    $$I(u_{ij};X_j)-I(u_{ij};s^\ast_{ij})=I(u_{ij};X_j\mid s_{ij}^\ast)-I(u_{ij};s^\ast_{ij}\mid X_j).$$
    
    For attainable MNI and under Proposition 1 by \citet{wang_information_2025}) for $s^*_{ij}$ it holds that,
    
    $$I(s^\ast_{ij};X_i\mid X_j)=0, \qquad I(s^\ast_{ij};X_i)=I(X_i;X_j).$$
    
    Now, from the data processing inequality one also gets $$0 \le I(u_{ij};s^\ast_{ij}\mid X_j)\le I(X_{i};s^\ast_{ij}\mid X_j)=0 \implies I(u_{ij};s^\ast_{ij}\mid X_j) = 0.$$
    
    For attainable MNI at $s^\ast_{ij}$ we also have $I(X_i;X_j\mid s^\ast_{ij})=0$, which follows from standard identities and since $u_{ij}$ depends only on $X_i$, $I(u_{ij};X_j\mid s^\ast_{ij})=0$, which proves $I(u_{ij};s^\ast_{ij})=I(u_{ij};X_j)$. 
    
    The equality with $s^\ast_{ji}$ follows symmetrically.
\end{proof}

    \subsection{Proof of Theorem \ref{theo: the_01}}
    \begin{redlinebox}
        \begin{theorem}
            \textit{Exact optimality under attainable MNI: }
            Let \(\theta^*_{ij} = (u^*_{ij}, u^*_{ji}, s^*_{ij}, s^*_{ji})\) be an optimal solution of the step-by-step optimization problem defined in Eqs.~\eqref{eq:eq_04} and \eqref{eq:eq_05}. 
            If the MNI criterion is attainable and \(\alpha \ge \lambda\), then \(\theta^*_{ij}\) is a global optimizer of the joint objective defined in Eqs.~\eqref{eq:eq_01} - \eqref{eq:eq_03}. Equivalently,
            \[
            \theta^*_{ij} \in \arg\max_{\theta_{ij}}\mathcal{J}(\theta_{ij}).
            \]
        \end{theorem}
    \end{redlinebox}

    \begin{proof}
    Let \((s^*_{ij}, s^*_{ji})\) and \((u^*_{ij}, u^*_{ji})\) denote the optimal shared and unique representations obtained from the step-by-step optimization in Eqs.~\eqref{eq:eq_04} and \eqref{eq:eq_05} respectively. We proceed by deriving an upper bound on the joint objective and then showing that this bound is attained at \(\theta^*\) when MNI is attainable.
    
    We begin with the term \(\mathcal{L}_{u_i}\) in the joint objective:
    \begin{align}
        \mathcal{L}_{u_i}
        &= I(u_{ij}, s_{ji}; X_i) - \lambda I(u_{ij}; s_{ji}) \notag\\
        &= I(u_{ij}, s_{ji}; X_i) - \lambda I(u_{ij}; s_{ji}) \pm \lambda I(u_{ij}; X_j) \notag\\
        &\leq I(u_{ij}, X_j; X_i) - \lambda I(u_{ij}; X_j)
        + \lambda \bigl[I(u_{ij}; X_j) - I(u_{ij}; s_{ji})\bigr].
        \label{eq: ap_eq_01_clean}
    \end{align}
    The inequality follows from
    \[
    I(u_{ij}, s_{ji}; X_i) \le I(u_{ij}, X_j; X_i),
    \]
    which holds due to the Markov relation \(s_{ji} \leftarrow X_j \leftrightarrow X_i\).
    
    Next, we use the identity
    \begin{equation}
        I(u_{ij}; X_j) - I(u_{ij}; s_{ji})
        =
        I(u_{ij}; X_j \mid s_{ji}) - I(u_{ij}; s_{ji} \mid X_j).
        \label{eq: ap_eq_02_clean}
    \end{equation}
    Since the joint distribution factorizes according to the Markov structure
    \[
    s_{ji} \leftarrow X_j \leftrightarrow X_i \rightarrow u_{ij},
    \]
    we have \(u_{ij} \perp s_{ji} \mid X_j\), and therefore
    \[
    I(u_{ij}; s_{ji} \mid X_j) = 0.
    \]
    Hence Eq.~\eqref{eq: ap_eq_02_clean} becomes
    \begin{align}
        I(u_{ij}; X_j) - I(u_{ij}; s_{ji})
        &= I(u_{ij}; X_j \mid s_{ji}) \notag\\
        &\le I(X_i; X_j \mid s_{ji}),
        \qquad \qquad \quad\text{by } u_{ij} \leftarrow X_i \leftrightarrow X_j \notag\\
        &= I(X_i; X_j, s_{ji}) - I(X_i; s_{ji}) \notag\\
        &= I(X_i; X_j) - I(X_i; s_{ji}),
        \quad \quad \text{by } s_{ji} \leftarrow X_j \leftrightarrow X_i.
        \label{eq: ap_eq_03_clean}
    \end{align}
    Substituting Eq.~\eqref{eq: ap_eq_03_clean} into Eq.~\eqref{eq: ap_eq_01_clean} yields
    \begin{align}
        \mathcal{L}_{u_i}
        &\le
        \Bigl[I(u_{ij}, X_j; X_i) - \lambda I(u_{ij}; X_j)\Bigr]
        + \lambda I(X_i; X_j) - \lambda I(X_i; s_{ji}) \notag\\
        &= \mathbf{B}_{u_i} - \lambda I(X_i; s_{ji}) + c,
        \label{eq: ap_eq_04_clean}
    \end{align}
    where
    \[
    \mathbf{B}_{u_i} := I(u_{ij}, X_j; X_i) - \lambda I(u_{ij}; X_j),
    \qquad
    c := \lambda I(X_i; X_j) \ge 0.
    \]
    
    By symmetry, the same argument gives
    \begin{equation}
        \mathcal{L}_{u_j}
        \le
        \mathbf{B}_{u_j} - \lambda I(X_j; s_{ij}) + c,
        \label{eq: ap_eq_05_clean}
    \end{equation}
    where
    \[
    \mathbf{B}_{u_j} := I(u_{ji}, X_i; X_j) - \lambda I(u_{ji}; X_i).
    \]
    
    We now combine these bounds in the full joint objective:
    \begin{align}
        \mathcal{J}
        &= \alpha(\mathcal{L}_{s_i} + \mathcal{L}_{s_j}) + \mathcal{L}_{u_i} + \mathcal{L}_{u_j} \notag\\
        &\le
        \alpha(\mathcal{L}_{s_i} + \mathcal{L}_{s_j})
        + \mathbf{B}_{u_i} + \mathbf{B}_{u_j} + 2c
        - \lambda I(X_i; s_{ji}) - \lambda I(X_j; s_{ij}).
        \label{eq: ap_eq_06_clean}
    \end{align}
    
    Define the modified shared objectives
    \begin{equation}
        \tilde{\mathcal{L}}_{s_i}
        := \alpha \mathcal{L}_{s_i} - \lambda I(X_j; s_{ij}),
        \label{eq: ap_eq_07_clean}
    \end{equation}
    and
    \begin{equation}
        \tilde{\mathcal{L}}_{s_j}
        := \alpha \mathcal{L}_{s_j} - \lambda I(X_i; s_{ji}).
        \label{eq: ap_eq_08_clean}
    \end{equation}
    Using the definition of \(\mathcal{L}_{s_i}\), we obtain
    \begin{align}
        \tilde{\mathcal{L}}_{s_i}
        &= \alpha\bigl[I(s_{ij}; X_j) - \beta I(s_{ij}; X_i \mid X_j)\bigr] - \lambda I(X_j; s_{ij}) \notag\\
        &= (\alpha-\lambda) I(s_{ij}; X_j) - \alpha\beta I(s_{ij}; X_i \mid X_j) \notag\\
        &\le (\alpha-\lambda) I(X_i; X_j),
        \qquad \text{for } \alpha \ge \lambda.
        \label{eq: ap_eq_09_clean}
    \end{align}
    The last step follows from
    \[
    I(s_{ij}; X_j) \le I(X_i; X_j)
    \]
    by the assumed Markov structure, together with the non-negativity of conditional mutual information.
    
    By symmetry,
    \begin{equation}
        \tilde{\mathcal{L}}_{s_j}
        \le (\alpha-\lambda) I(X_i; X_j).
        \label{eq: ap_eq_10_clean}
    \end{equation}
    
    Substituting Eqs.~\eqref{eq: ap_eq_09_clean} and \eqref{eq: ap_eq_10_clean} into Eq.~\eqref{eq: ap_eq_06_clean} gives
    \begin{equation}
        \mathcal{J}
        \le
        2(\alpha-\lambda) I(X_i; X_j)
        + \mathbf{B}_{u_i} + \mathbf{B}_{u_j} + 2c,
        \qquad \forall\, u_{ij},u_{ji}.
        \label{eq: ap_eq_11_clean}
    \end{equation}
    
    We now show that this bound is attained at the step-by-step shared optimum when MNI is attainable. Since \(s^*_{ij}\) and \(s^*_{ji}\) satisfy the MNI criterion, we have
    \[
    I(s^*_{ij}; X_j) = I(s^*_{ij}; X_i) = I(X_i; X_j),
    \qquad
    I(s^*_{ji}; X_i) = I(s^*_{ji}; X_j) = I(X_i; X_j),
    \]
    and therefore
    \[
    I(s^*_{ij}; X_i \mid X_j)=0,
    \qquad
    I(s^*_{ji}; X_j \mid X_i)=0
    \]
    as well as,
    \[
    I(u_{ij}; X_j \mid s^*_{ji})= I(X_i; X_j \mid s^*_{ji}),
    \qquad
    I(u_{ji}; X_i \mid s^*_{ij})= I(X_j; X_i \mid s^*_{ij}).
    \]
    Moreover, in the attainable-MNI regime, Proposition~3 by \citet{wang_information_2025} implies that
    \[
    I(u_{ij}, s^*_{ji}; X_i) = I(u_{ij}, X_j; X_i),
    \qquad
    I(u_{ji}, s^*_{ij}; X_j) = I(u_{ji}, X_i; X_j),
    \]

    while Lemma~\eqref{lemma: 01} suggests that $I(u_{ij};s^*_{ij}) = I(u_{ij};X_j)=I(u_{ij};s^*_{ji})$ for every feasible $u_{ij}$.
    
    It follows that all inequalities above become equalities when \((s_{ij},s_{ji})=(s^*_{ij},s^*_{ji})\). Hence
    \begin{align}
        \mathcal{J}
        &=
        2(\alpha-\lambda) I(X_i; X_j)
        + \mathbf{B}_{u_i} + \mathbf{B}_{u_j} + 2c \notag\\
        &=
        2\alpha I(X_i; X_j) + \mathbf{B}_{u_i} + \mathbf{B}_{u_j},
        \qquad c=\lambda I(X_i;X_j).
        \label{eq: ap_eq_12_clean}
    \end{align}
    
    Since the term \(2\alpha I(X_i;X_j)\) is constant with respect to \(u_{ij}\) and \(u_{ji}\), maximizing \(\mathcal{J}\) is equivalent to maximizing
    \[
    \mathbf{B}_{u_i} + \mathbf{B}_{u_j}.
    \]

    Furthermore, \(\mathbf{B}_{u_i}\) depends only on \(u_{ij}\), while \(\mathbf{B}_{u_j}\) depends only on \(u_{ji}\). Therefore the maximization decouples as
    \[
    \arg\max_{u_{ij},u_{ji}} \bigl(\mathbf{B}_{u_i} + \mathbf{B}_{u_j}\bigr)
    =
    \arg\max_{u_{ij}} \mathbf{B}_{u_i}
    \times
    \arg\max_{u_{ji}} \mathbf{B}_{u_j}.
    \]
    By construction, this is exactly the optimization problem for the unique components in Eq.~\eqref{eq:eq_05} under attainable MNI. Therefore the step-by-step optimizer \((u^*_{ij},u^*_{ji})\) also maximizes the joint surrogate objective. Consequently,
    \[
    \theta^*_{ij} \in \arg\max_{\theta_{ij}}\mathcal{J}(\theta_{ij}),
    \]
    which proves the claim.
    \end{proof}

    \newpage
    
    \subsection{Proof of Theorem \ref{theo: the_02}}
    \begin{redlinebox}
        \begin{theorem}
            \textit{Near-optimality under unattainable MNI: }
            Let $\theta^*_{ij} = (u^*_{ij}, u^*_{ji}, s^*_{ij}, s^*_{ji})$ be an optimal solution of the ideal step-by-step optimizer as stated in Eqs. \eqref{eq:eq_04} - \eqref{eq:eq_05}, and $\hat{\theta}_{ij}$ the corresponding optimal solution of the joint optimization problem, as stated in Eqs. \eqref{eq:eq_01} - \eqref{eq:eq_03}. If MNI is unattainable, any maximizer $\hat{\theta}_{ij}$ of the joint objective satisfies
            \[ 
            \mathcal{J}(\hat{\theta}_{ij}) \geq \mathcal{J}^o(\theta^*_{ij}) - 2 (1 + \lambda) \delta_c,
            \]
            where, 
            \[
            \mathcal{J}^o(\theta^*_{ij}) = \alpha \mathcal{L}^o_{s_i}(\theta^*_{ij}) + \alpha \mathcal{L}^o_{s_j}(\theta^*_{ij}) + \mathcal{L}^o_{u_i}(\theta^*_{ij}) + \mathcal{L}^o_{u_j}(\theta^*_{ij}),
            \]
            and $\delta_c$ denotes the information-gap term introduced in \citet{wang_information_2025}, which quantifies the trade-off between coverage and redundancy in the shared representations.
        \end{theorem}
    \end{redlinebox}

    \begin{proof}
    Assume now that MNI is unattainable. Let
    \[
    \theta^*_{ij} = (u^*_{ij}, u^*_{ji}, s^*_{ij}, s^*_{ji})
    \]
    denote the solution of the ideal step-by-step optimization problem defined in
    Eqs.~\eqref{eq:eq_04} and \eqref{eq:eq_05}, i.e.,
    \[
    \begin{aligned}
    s^*_{ij}
    &\in \arg\max_{s_{ij}} \mathcal{L}^{o}_{s_i}
     = \arg\max_{s_{ij}} \Bigl[I(s_{ij}; X_j) - \beta\, I(s_{ij}; X_i \mid X_j)\Bigr], \\
    s^*_{ji}
    &\in \arg\max_{s_{ji}} \mathcal{L}^{o}_{s_j}
     = \arg\max_{s_{ji}} \Bigl[I(s_{ji}; X_i) - \beta\, I(s_{ji}; X_j \mid X_i)\Bigr], \\
    u^*_{ij}
    &\in \arg\max_{u_{ij}} \mathcal{L}^{o}_{u_i}
     = \arg\max_{u_{ij}} \Bigl[I(u_{ij}, X_j; X_i) - \lambda I(u_{ij}; s^*_{ij})\Bigr], \\
    u^*_{ji}
    &\in \arg\max_{u_{ji}} \mathcal{L}^{o}_{u_j}
     = \arg\max_{u_{ji}} \Bigl[I(u_{ji}, X_i; X_j) - \lambda I(u_{ji}; s^*_{ji})\Bigr].
    \end{aligned}
    \]
    
    We define the corresponding ideal reweighted value at \(\theta^*_{ij}\) as
    \begin{align*}
        \mathcal{J}^o(\theta^*_{ij})
        &= \alpha I(s^*_{ij};X_j) - \alpha \beta I(s^*_{ij};X_i\mid X_j) \\
        &\quad + \alpha I(s^*_{ji};X_i) - \alpha \beta I(s^*_{ji};X_j\mid X_i) \\
        &\quad + I(u^*_{ij}, X_j;X_i) - \lambda I(u^*_{ij};s^*_{ij}) \\
        &\quad + I(u^*_{ji}, X_i;X_j) - \lambda I(u^*_{ji};s^*_{ji}).
    \end{align*}
    Equivalently,
    \[
    \mathcal{J}^o(\theta^*_{ij})
    =
    \alpha \mathcal{L}^o_{s_i}(\theta^*_{ij})
    + \alpha \mathcal{L}^o_{s_j}(\theta^*_{ij})
    + \mathcal{L}^o_{u_i}(\theta^*_{ij})
    + \mathcal{L}^o_{u_j}(\theta^*_{ij}).
    \]
    
    Next, for an arbitrary feasible point
    \[
    \theta_{ij} = (u_{ij}, u_{ji}, s_{ij}, s_{ji}),
    \]
    the joint objective is
    \begin{align*}
        \mathcal{J}(\theta_{ij})
        &= \alpha I(s_{ij};X_j) - \alpha \beta I(s_{ij};X_i\mid X_j) \\
        &\quad + \alpha I(s_{ji};X_i) - \alpha \beta I(s_{ji};X_j\mid X_i) \\
        &\quad + I(u_{ij}, s_{ji};X_i) - \lambda I(u_{ij};s_{ji}) \\
        &\quad + I(u_{ji}, s_{ij};X_j) - \lambda I(u_{ji};s_{ij}).
    \end{align*}
    Let
    \[
    \hat{\theta}_{ij} \in \arg\max_{\theta_{ij}} \mathcal{J}(\theta_{ij}).
    \]
    By definition of \(\hat{\theta}_{ij}\), we have
    \begin{equation}
        \mathcal{J}(\hat{\theta}_{ij}) \geq \mathcal{J}(\theta^*_{ij}),
        \qquad
        \hat{\theta}_{ij} \in \arg\max_{\theta_{ij}} \mathcal{J}(\theta_{ij}).
        \label{eq: ap_eq_14_clean}
    \end{equation}
    
    Evaluating the joint objective at \(\theta^*_{ij}\), we obtain
    \begin{align*}
        \mathcal{J}(\theta^*_{ij})
        &= \alpha I(s^*_{ij};X_j) - \alpha \beta I(s^*_{ij};X_i\mid X_j) \\
        &\quad + \alpha I(s^*_{ji};X_i) - \alpha \beta I(s^*_{ji};X_j\mid X_i) \\
        &\quad + I(u^*_{ij}, s^*_{ji};X_i) - \lambda I(u^*_{ij};s^*_{ji}) \\
        &\quad + I(u^*_{ji}, s^*_{ij};X_j) - \lambda I(u^*_{ji};s^*_{ij}).
    \end{align*}
    
    We now consider the difference between the ideal reweighted value and the joint objective at \(\theta^*_{ij}\):
    \begin{align*}
        \mathcal{J}^o(\theta^*_{ij}) - \mathcal{J}(\theta^*_{ij})
        &= \Bigl[I(u^*_{ij}, X_j;X_i) - I(u^*_{ij}, s^*_{ji};X_i)\Bigr] \\
        &\quad + \Bigl[I(u^*_{ji}, X_i;X_j) - I(u^*_{ji}, s^*_{ij};X_j)\Bigr] \\
        &\quad + \lambda \Bigl[I(u^*_{ij};s^*_{ji}) - I(u^*_{ij};s^*_{ij})\Bigr] \\
        &\quad + \lambda \Bigl[I(u^*_{ji};s^*_{ij}) - I(u^*_{ji};s^*_{ji})\Bigr].
    \end{align*}
    For brevity, we define
    \[
    \Delta_{u_i}
    := I(u^*_{ij}, X_j;X_i) - I(u^*_{ij}, s^*_{ji};X_i),
    \qquad
    \Lambda_{u_i}
    := I(u^*_{ij};s^*_{ji}) - I(u^*_{ij};s^*_{ij}),
    \]
    and similarly
    \[
    \Delta_{u_j}
    := I(u^*_{ji}, X_i;X_j) - I(u^*_{ji}, s^*_{ij};X_j),
    \qquad
    \Lambda_{u_j}
    := I(u^*_{ji};s^*_{ij}) - I(u^*_{ji};s^*_{ji}).
    \]
    Then
    \[
    \mathcal{J}^o(\theta^*_{ij}) - \mathcal{J}(\theta^*_{ij})
    =
    \Delta_{u_i} + \Delta_{u_j} + \lambda(\Lambda_{u_i} + \Lambda_{u_j}).
    \]
    
    We next bound the terms \(\Lambda_{u_i}\) and \(\Lambda_{u_j}\). From the Markov relation
    \[
    s_{ji} \leftarrow X_j \leftrightarrow X_i \rightarrow u_{ij},
    \]
    we have
    \begin{equation}
        I(u^*_{ij}; s^*_{ji}) \leq I(u^*_{ij};X_j),
        \label{eq: ap_eq_15a_clean}
    \end{equation}
    and therefore
    \begin{equation}
        I(u^*_{ij}; s^*_{ji}) - I(u^*_{ij};s^*_{ij})
        \leq I(u^*_{ij};X_j) - I(u^*_{ij};s^*_{ij}).
        \label{eq: ap_eq_15b_clean}
    \end{equation}
    Hence
    \begin{align}
        \Lambda_{u_i}
        &\leq I(u^*_{ij};X_j) - I(u^*_{ij};s^*_{ij}) \notag\\
        &= I(u^*_{ij};X_j\mid s^*_{ij}) - I(u^*_{ij};s^*_{ij}\mid X_j) \notag\\
        &\leq I(u^*_{ij};X_j\mid s^*_{ij}) \notag\\
        &\leq I(X_i;X_j\mid s^*_{ij}),
        \qquad \text{by } u_{ij} \leftarrow X_i \leftrightarrow X_j \notag\\
        &= I(X_j;X_i,s^*_{ij}) - I(X_j;s^*_{ij}) \notag\\
        &= I(X_j;X_i) - I(X_j;s^*_{ij}) \notag\\
        &\leq \delta_c.
        \label{eq: ap_eq_16_clean}
    \end{align}
    The last inequality follows from the definition of the ideal shared optimization problem: since \(s^*_{ij}\) is an optimal solution of Eq.~\eqref{eq:eq_04}, it satisfies
    \[
    I(X_j;X_i) - I(X_j;s^*_{ij}) \leq \delta_c.
    \]
    By the same argument, we also obtain
    \begin{equation}
        \Lambda_{u_j} \leq \delta_c.
        \label{eq: ap_eq_16b_clean}
    \end{equation}
    
    Furthermore, by Proposition~4 of \citet{wang_information_2025}, the terms \(\Delta_{u_i}\) and \(\Delta_{u_j}\) are each upper bounded by \(\delta_c\). Therefore,
    \begin{equation}
        \mathcal{J}^o(\theta^*_{ij}) - \mathcal{J}(\theta^*_{ij})
        \leq 2\delta_c + 2\lambda \delta_c.
        \label{eq: ap_eq_17_clean}
    \end{equation}
    
    Combining Eqs.~\eqref{eq: ap_eq_14_clean} and \eqref{eq: ap_eq_17_clean}, we conclude that
    \[
    \mathcal{J}(\hat{\theta}_{ij})
    \geq
    \mathcal{J}^o(\theta^*_{ij}) - 2(1+\lambda)\delta_c.
    \]
    proving near-optimality of the joint optimization objective with respect to the ideal step-by-step optimization problem.
    \end{proof}

    \paragraph{Stronger form of near-optimality} When MNI is unattainable, we further state the following assumption:

    \begin{bluelinebox}
        
        \begin{assumption} At the optimal solution $\hat{\theta}_{ij}$ of the joint optimizer, we define $d$ to quantify the residual redundancy between the unique and shared components when MNI is unattainable, and for which it holds:
            $$ 
            I(\hat{u}_{ij};\hat{s}_{ij}) \le d, \qquad I(\hat{u}_{ji};\hat{s}_{ji}) \le d.
            $$
            
            Here, $d$ is the disentanglement analogue of the coverage gap $\delta_c$ and is zero under attainable MNI (follows from Theorem \ref{theo: the_01}).
            \label{assumption_A1}
        \end{assumption}
    \end{bluelinebox}

    Under Assumption \ref{assumption_A1}, we can additionally derive Propositions~\ref{prop: 01} and~\ref{prop: 02}.

    \begin{bluelinebox}
        \begin{proposition}
        Let $\theta^*_{ij} = (u^*_{ij}, u^*_{ji}, s^*_{ij}, s^*_{ji})$ be an optimal solution of the step-by-step optimizer as stated in Eqs. \eqref{eq:eq_04} - \eqref{eq:eq_05}, and $\hat{\theta}_{ij}$ the corresponding optimal solution of the joint optimization problem, as stated in Eqs. \eqref{eq:eq_01} - \eqref{eq:eq_03}. When MNI is unattainable, under Assumption \ref{assumption_A1} it holds,
        $$
        \mid \mathcal{J}(\hat{\theta}_{ij}) - \mathcal{J}(\theta_{ij}^*)\mid \le 2(1 + \lambda) \delta_c + 2\lambda d.
        $$
            \label{prop: 01}
        \end{proposition}
    \end{bluelinebox}

    \begin{proof}
    By definition of $\hat{\theta}_{ij}$, we have
    $$
        \mathcal{J}(\hat{\theta}_{ij}) - \mathcal{J}(\theta_{ij}^*) \ge 0, \qquad \hat{\theta}_{ij} \in \arg\max_{\theta_{ij}} \mathcal{\mathcal{J}}(\theta_{ij}).
    $$
    Moreover, the difference can we rewritten as follows:
    $$ \mathcal{J}(\hat{\theta}_{ij}) - \mathcal{J}(\theta_{ij}^*) = \underbrace{[\mathcal{J}(\hat{\theta}_{ij}) - \mathcal{J}^o(\hat{\theta}_{ij})]}_{A} + \underbrace{[\mathcal{J}^o(\hat{\theta}_{ij}) - \mathcal{J}^o(\theta_{ij}^*)]}_{B} + \underbrace{[\mathcal{J}^o(\theta_{ij}^*) - \mathcal{J}(\theta_{ij}^*)]}_{C},$$
    where, 
    \[
    \mathcal{J}^o(\theta_{ij}) = \alpha \mathcal{L}^o_{s_i}(\theta_{ij}) + \alpha \mathcal{L}^o_{s_j}(\theta_{ij}) + \mathcal{L}^o_{u_i}(\theta_{ij}) + \mathcal{L}^o_{u_j}(\theta_{ij}).
    \]
    
    For A we have: $$\mathcal{J}(\hat{\theta}_{ij}) - \mathcal{J}^o(\hat{\theta}_{ij}) = - \hat{\Delta}_{u_i} - \hat{\Delta}_{u_j} - \lambda \hat{\Lambda}_{u_i} - \lambda \hat{\Lambda}_{u_j},$$
    where,
    \[
    \hat{\Delta}_{u_i}
    := I(\hat{u}_{ij}, X_j;X_i) - I(\hat{u}_{ij}, \hat{s}_{ji};X_i),
    \qquad
    \hat{\Lambda}_{u_i}
    := I(\hat{u}_{ij};\hat{s}_{ji}) - I(\hat{u}_{ij};\hat{s}_{ij}),
    \]
    and similarly,
    \[
    \hat{\Delta}_{u_j}
    := I(\hat{u}_{ji}, X_i;X_j) - I(\hat{u}_{ji}, \hat{s}_{ij};X_j),
    \qquad
    \hat{\Lambda}_{u_j}
    := I(\hat{u}_{ji};\hat{s}_{ij}) - I(\hat{u}_{ji};\hat{s}_{ji}).
    \]

    Notice that $\hat{\Delta}_{u_i}, \hat{\Delta}_{u_j}, \hat{\Lambda}_{u_i}, \hat{\Lambda}_{u_j}$ are the same expressions defined in Theorem \ref{theo: the_02} evaluated at $\hat{u}_{ij}, \hat{s}{ij}, \hat{u}_{ji}, \hat{s}_{ji}$.
    It holds,
    $$
    I(\hat{u}_{ij},\hat{s}_{ji};X_i) \le I(\hat{u}_{ij},X_j;X_i) \implies - \hat{\Delta}_{u_i} \le 0,
    $$
    from data processing, since $X_j \rightarrow s_{ji}$. By exact symmetry, we obtain for $-\hat{\Delta}_{u_j} \le 0$. 
    
    Moreover, we have,
    $$
    I(\hat{u}_{ij};\hat{s}_{ij}) - I(\hat{u}_{ij};\hat{s}_{ji}) \le I(\hat{u}_{ij};\hat{s}_{ij}) \implies -\Lambda_{u_i} \le d,
    $$
    where the first inequality follows from the non-negativity of mutual information and the latter holds under Assumption \ref{assumption_A1}. Similarly, we have under exact symmetry that $-\Lambda_{u_j} \le d$.
    
    Overall we obtain,
    
    $$
    \mathcal{J}(\hat{\theta}_{ij}) - \mathcal{J}^o(\hat{\theta}_{ij}) \le 2 \lambda d.
    $$
    
    For B we have:
    
    $$
    \mathcal{J}^o(\hat{\theta}_{ij}) - \mathcal{J}^o(\theta_{ij}^*) \le 0
    $$
    since $\theta_{ij}^* \in \arg \max \mathcal{J}^o$.
    
    For C it holds:
    
    $$
    \mathcal{J}^o(\theta_{ij}^*) - \mathcal{J}(\theta_{ij}^*) \le 2(1 + \lambda) \delta_c
    $$
    which follows from Theorem \ref{theo: the_02}, Eq. \eqref{eq: ap_eq_17_clean}. Thus we have,
    
    $$
        \mid \mathcal{J}(\hat{\theta}_{ij}) - \mathcal{J}(\theta_{ij}^*)\mid \le 2(1 + \lambda) \delta_c + 2\lambda d.
    $$
    which concludes the proof.
    \end{proof}

    Moreover, when MNI is unattainable we can also derive the following:

    \begin{bluelinebox}
         \begin{proposition}
         Let $\theta^*_{ij} = (u^*_{ij}, u^*_{ji}, s^*_{ij}, s^*_{ji})$ be an optimal solution of the step-by-step optimizer as stated in Eqs. \eqref{eq:eq_04} - \eqref{eq:eq_05}, and $\hat{\theta}_{ij}$ the corresponding optimal solution of the joint optimization problem, as stated in Eqs. \eqref{eq:eq_01} - \eqref{eq:eq_03}. When MNI is unattainable, under Assumption \ref{assumption_A1} it holds,
        $$
        - 2(1 + \lambda)\delta_c \le \mathcal{J}(\hat{\theta}_{ij}) - \mathcal{J}^o(\theta_{ij}^*) \le 2\lambda d,
        $$
        where, 
        \[
        \mathcal{J}^o(\theta^*_{ij}) = \alpha \mathcal{L}^o_{s_i}(\theta^*_{ij}) + \alpha \mathcal{L}^o_{s_j}(\theta^*_{ij}) + \mathcal{L}^o_{u_i}(\theta^*_{ij}) + \mathcal{L}^o_{u_j}(\theta^*_{ij}).
        \]
            \label{prop: 02}
        \end{proposition}
    \end{bluelinebox}

    \begin{proof}
        By Theorem \ref{theo: the_02} it holds,
        $$
            \mathcal{J}(\hat{\theta}_{ij}) - \mathcal{J}^o(\theta_{ij}^*) \ge - 2(1 + \lambda)\delta_c.
        $$
        Moreover, the difference can be rewritten as follows:
        $$ \mathcal{J}(\hat{\theta}_{ij}) - \mathcal{J}^o(\theta_{ij}^*) = \underbrace{[\mathcal{J}(\hat{\theta}_{ij}) - \mathcal{J}^o(\hat{\theta}_{ij})]}_{A} + \underbrace{[\mathcal{J}^o(\hat{\theta}_{ij}) - \mathcal{J}^o(\theta_{ij}^*)]}_{B},$$
        
        and similarly to the Proposition \ref{prop: 01} for the terms A and B, under Assumption \ref{assumption_A1} we obtain:
        
        $$
        \mathcal{J}(\hat{\theta}_{ij}) - \mathcal{J}^o(\theta_{ij}^*) \le 2\lambda d.
        $$
        
        Thus, in total we have,
        $$
          - 2(1 + \lambda)\delta_c \le \mathcal{J}(\hat{\theta}_{ij}) - \mathcal{J}^o(\theta_{ij}^*) \le 2\lambda d,
        $$
        which concludes the proof.
    \end{proof}

    Overall, through Theorem \ref{theo: the_02} we show that for the maximization problem, the value of our optimizer $\mathcal{J}(\hat{\theta}_{ij})$ is lower bounded by $\mathcal{J}^o(\theta_{ij}^*)$ minus a tractable constant $2(1 + \lambda)\delta_c$ and under Assumption \ref{assumption_A1} and following Proposition \ref{prop: 02} we show that the difference $\mathcal{J}(\hat{\theta}_{ij})- \mathcal{J}^o(\theta_{ij}^*)$ is also upper bounded by $2\lambda d$. Hence, under Assumption \ref{assumption_A1}, $\mathcal{J}(\hat{\theta}_{ij})$ can exceed the ideal optimum by at most $ 2\lambda d$. The gap is governed by the constants $\delta_c, d$ and it closes completely under attainable MNI.

\section{Framework Implementation Details}
\label{framework_details}
\subsection{Slot assignment in latent array \texorpdfstring{$\mathcal{H}$}{H}}
The flexibility of our framework permits any bijective mapping $\phi_i$ between the different rows of each $\mathcal{H}_i$ and the corresponding components. This mapping may vary across modalities, as each disentangling module operates independently. Nevertheless, for each modality, $\phi_i$ should cover exactly once all the expected pairwise interactions, excluding self-interactions, and remain consistent throughout both training and inference. For our implementation, we adopt the following definition:

\begin{definition} We define the bijective mapping $\phi_i$ as:
\[
\phi_i : \mathcal{H}_i \rightarrow \{u_{ij}, s_{ij}\}_{j \neq i}, \quad \forall i \in \mathcal{M}.
\]
Let $k$ denote the row index of $\mathcal{H}_i$ corresponding to the $k$-th slot of $\mathcal{H}_i$, $h_{ik}$. For every modality $j$, we associate each $h_{ik}$, with either a unique component $u_{ij}$ or a shared component $s_{ij}$ involving modality $i$. We also set $j = \lceil k/2 \rceil$, as follows.

\begin{itemize}
    \item If $k$ is odd:
    \[
    \phi_i(h_{ik}) =
    \begin{cases}
    u_{ij} & j < i \\
    u_{i,j+1} & j > i.
    \end{cases}
    \]
    \item If $k$ is even:
    \[
    \phi_i(h_{ik}) =
    \begin{cases}
    s_{ij} & j < i \\
    s_{i,j+1} & j \ge i.
    \end{cases}
    \]
\end{itemize}

This construction systematically indexes all shared and unique components involving modality $i$ while excluding self-interactions.
\label{def: mapping_def}
\end{definition}

\begin{figure}[ht]
    \centering
    \includegraphics[width=0.8\linewidth]{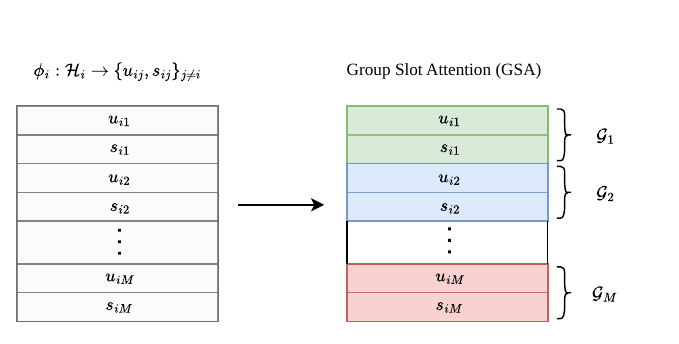}
    \caption{Illustration of the bijective mapping $\phi_i$ and Group Slot Attention. \textbf{Left:} Visualization of the mapping defined in Definition \ref{def: mapping_def}, where each latent slot, $h_{ik}$, is assigned to a specific unique or shared component. \textbf{Right:} the grouping mechanism is observed, where each pair \((u_{ij}, s_{ij})\) belongs in a separate group $\mathcal{G}_j$.}
    \label{fig:gsa_mapping}
\end{figure}

\subsection{Group slot attention}
In grouped slot attention, slots are partitioned into disjoint groups, and competition is enforced only within each group rather than across all slots. In our setting, we select the group size equal to 2, resulting in exactly $M - 1$ groups, such that every group contains exactly two slots, e.g. the pair (\(u_{ij}, s_{ij}\)) corresponding to the unique and shared components for the modality pair (\(i, j\)). For each input token, attention scores are normalized across the two slots in the group, so that the two slots compete to explain that particular token. As a result, each token distributes its information only between the paired slots, encouraging them to specialize into complementary roles while remaining independent of other groups. 

Formally, for each cross-attention mechanism at iteration $t$, we compute
\[
Q = \tilde{\mathcal{H}}^t_{i} W^T_Q \in \mathbb{R}^{N \times D_h}, \qquad K = Z_i W^T_K \in \mathbb{R}^{D_h \times D}, \qquad V = Z_i W^T_V \in \mathbb{R}^{D_h \times D},
\]
representing the query, key and value respectively, where $W_Q \in \mathbb{R}^{D_h \times D_h}, \text{and } W_K, W_v \in \mathbb{R}^{T \times D_h}$ are learnable parameter matrices. Given the mapping defined in \ref{def: mapping_def}, the queries are arranged into contiguous pairs of size two. We compute the standard query-key similarities, but normalize the attention only over slots within the same group. If $\mathcal{G}_m(n)$ denotes the $m$-th group containing slot $n$, then
\[
\mathbf{A}_{n, m} = \text{softmax}_{n^{'} \in \mathcal{G}_m(n)} \frac{Q_{n^{'}} K^T_{m}}{\sqrt{D_h}}.
\]
Afterwards, we re-normalize across the token dimension,
\[
\hat{\mathbf{A}}_{n, m} = \frac{\mathbf{A}_{n, m}}{\sum_{j=1}^{T}\mathbf{A}_{n, j} + \epsilon}, \qquad \epsilon > 0.
\]

Lastly, slot representations are updated via the standard attention readout
\[
\tilde{\mathcal{H}}^{t + 1}_i = \hat{\mathbf{A}} V.
\]
Both the bijective mapping as well as the Group Slot Attention are illustrated in Figure \ref{fig:gsa_mapping} for additional intuition.

\subsection{Training objectives}
\label{training_tractable_objectives}
During training, in practice we optimize the losses in Eqs. \eqref{eq:eq_tract_1}, \eqref{eq:eq_tract_2}. We model the $s_{ij} \sim p(\cdot \mid X_i)$ and $s_{ji} \sim p(\cdot \mid X_j)$ as von Mises-Fisher (vMF) distributions, i.e. $s_{ij} \sim \text{vMF}(\mu (X_i), \kappa)$ and $s_{ji} \sim \text{vMF}(\mu(X_j), \kappa)$. The $\text{vMF}$ distribution represents data on a hypershere, where $\mu$ is the mean direction and $\kappa$ controls the concentration around that direction. For $I(s_{ij};X_j)$ we use the InfoNCE objective~\citep{Oord2018RepresentationLW} as follows:
\[
    \mathcal{L}_{s_i}^{\mathrm{INCE}}
    =
    \mathbb{E}_{s_{ij},\,s_{ji}^{+},\,\{s_{ji,k}^{-}\}_{k=1}^{N}}
    \left[
    -\log
    \frac{
    \exp\left( s_{ij}^{\top}s_{ji}^{+}/\tau \right)
    }{
    \exp\left( s_{ij}^{\top}s_{ji}^{+}/\tau \right)
    +
    \sum_{k=1}^{N}
    \exp\left( s_{ij}^{\top}s_{ji,k}^{-}/\tau \right)
    }
    \right],
\]
where $\tau$ is the temperature parameter, $s^+_{ji}$ is the positive anchor, $s^-_{ji, k}$ the negatives and $N$ is the total number of negative samples. A similar formulation follows by symmetry for the case of $\mathcal{L}^{INCE}_{s_{ji}}$. For the $KL$-divergence terms, as we have modeled the representations through $\text{vMF}$, the quantity $\mathcal{L}_{KL}$ obtains the exact form:
\[
\mathcal{L}_{KL} = \mathbb{E}_{x_i, x_j}\left[ \mu(X_i)^T\mu(X_j)\right]
\]
and accordingly for the case of $s_{ji}$. Analogously, for \(\mathcal{L}^{\mathrm{INCE}}_{u_i}\), we employ an InfoNCE objective using augmented views of \(X_i\) and \(X_j\). Specifically, we concatenate the unique representation \(u_{ij}\) with the shared representation \(s_{ji}\),
\[
\tilde{z} = u_{ij} \oplus s_{ji},
\]
and contrast it against the representation \(\tilde{z}_{\mathrm{a}}^{+}\) obtained from the augmented views of the same sample:
\[
\mathcal{L}_{u_i}^{\mathrm{INCE}}
=
\mathbb{E}_{\tilde{z},\,\tilde{z}_{\mathrm{a}}^{+},\,\{\tilde{z}_{k}^{-}\}_{k=1}^{N}}
\left[
-\log
\frac{
\exp\left( \tilde{z}^{\top}\tilde{z}_{\mathrm{a}}^{+}/\tau \right)
}{
\exp\left( \tilde{z}^{\top}\tilde{z}_{\mathrm{a}}^{+}/\tau \right)
+
\sum_{k=1}^{N}
\exp\left( \tilde{z}^{\top}\tilde{z}_{k}^{-}/\tau \right)
}
\right].
\]
Here, \(N\) denotes the number of negative samples, \(\tilde{z}_{\mathrm{a}}^{+}\) is the positive counterpart obtained from the augmented views of the same sample, and \(\tilde{z}_{k}^{-}\) denotes the \(k\)-th negative sample. Note that as discussed by \cite{10.5555/3495724.3496297}, the augmentation design within this setting is an important practical consideration and should be carefully chosen, so that all the desired information components are preserved within the augmented views. Lastly, the leakage between the unique and shared components is regularized through a cross-covariance orthogonality penalty. Concretely, let \(\bar{U}_{ij} \in \mathbb{R}^{B \times D}\) and \(\bar{S}_{ji}\in \mathbb{R}^{B \times D}\) denote the row-wise $\ell_2$-normalized and feature-wise standardized batch representations of \(u_{ij}\) and \(s_{ji}\), respectively. We define
\[
\mathcal{L}_{\mathrm{xcov}}(u_{ij},s_{ji}) = \frac{1}{D} \left\| \frac{1}{B} \bar{U}_{ij}^{\top}\bar{S}_{ji}\right\|_{F},
\]
where \(B\) is the number of samples, \(D\) is the representation dimension, and \(\|\cdot\|_{F}\) denotes the Frobenius norm.

\section{Experimental Details}

\subsection{Synthetic data creation}
\label{synthetic_data_creation}

% Each $W_{i, s}$ is designed to select and transform a specific subset of the atomic representations $z_A$, ensuring that while individual tokens $\mathbf{Z}_{i, s}$ contain partial information, the full sequence $\mathbf{Z}_i$ preserves the complete composite signal. 

For each atomic representation as stated in Definition \ref{def:atomic_representation}, we begin by sampling \(N\) instances for all \(2^M - 1\) latent factors independently from a normal distribution, i.e. \(z_A \sim \mathcal{N}(0, \sigma^2 I_{d_\ell})\). We then concatenate the sampled factors on a per-modality basis, according to Definition \ref{def: 2.2}, to construct \(Z_i\) for all \(i \in \mathcal{M}\). To avoid introducing structural biases related to dimensionality or distribution, we use the same \(\sigma\) and \(d_\ell\) for all sampled atomic representations. In addition, since we conduct experiments across varying numbers of modalities, we set \(d_\ell = \frac{E}{2^{M-1}}\), so that \(E = \dim Z_i\) remains constant regardless of \(M\).

Having obtained \(Z_i \in \mathbb{R}^{E}\), we next augment this vector representation to mimic sequence-style embeddings, yielding \(Z_i \in \mathbb{R}^{S \times E}\), as in practice foundation models typically produce patch- or token-level representations. To this end, for each latent factor we additionally sample a base transform \(U_A \sim \mathcal{N}(0, \sigma_U^2 I_{d_\ell})\), together with a small sequence-specific variation \(\Delta_{A,s} \sim \mathcal{N}(0, \sigma_\Delta^2 I)\), where \(\sigma_\Delta \ll \sigma_U\). We then place these transformation matrices in block-diagonal form over all \(s \in S\), defining \(U_{i,s} = \mathrm{blockdiag}(U_{A} + \Delta_{A,s})\) for all \(i \in \mathcal{M}\).

Moreover, we sample binary masks \(\mathbf{M}_{i,s} = [m^s_1, \dots, m^s_{2^{M-1}}]\), with \(m^s_j \in \{0,1\}\), so that not all atomic representations are accessible at every sequence position. Intuitively, this mimics the fact that, for example, not all image patches contain all information components of an image. As a result, \(\mathbf{M}_i \in \{0,1\}^{S \times 2^{M-1}}\). We apply this mask pointwise to the constructed transformations, yielding \(\tilde{U}_{i,s} = \mathbf{M}_{i,s} \odot U_{i,s}\). Finally, we define modality-specific rotations \(R_i\) for all \(i \in \mathcal{M}\), and set \(W_{i,s} = R_i \tilde{U}_{i,s}\). Each modality’s final representation is then given by \(\mathbf{Z}_{i,s} = \tanh(\gamma W_{i,s} Z_i)\).

\begin{algorithm}[ht!]
\caption{Synthetic generation of multimodal factorized representations}
\label{alg:synthetic_data}
\begin{algorithmic}[1]

\Require number of modalities $M$, sequence length $S$, embedding dimension $E$, variance parameters $\sigma, \sigma_U, \sigma_\Delta$, scaling factor $\gamma$

\State Let $\mathcal{A}$ denote the set of all non-empty modality subsets, $|\mathcal{A}| = 2^M - 1$

\State Set latent dimension per atomic factor:
\[
d_\ell = \frac{E}{2^{M-1}}
\]

\For{each atomic factor $A \in \mathcal{A}$}
    \State Sample atomic latent representation:
    \[
    z_A \sim \mathcal{N}(0, \sigma^2 I_{d_\ell})
    \]
\EndFor

\For{each modality $i \in \mathcal{M}$}
    \State Construct modality representation by concatenating relevant atomic factors:
    \[
    Z_i = \text{concat}\{ z_A \mid A \ni i \}, \quad Z_i \in \mathbb{R}^{E}
    \]

    \For{each sequence position $s = 1,\dots,S$}

        \For{each atomic factor $A \ni i$}
            \State Sample base transform:
            \[
            U_A \sim \mathcal{N}(0, \sigma_U^2 I_{d_\ell})
            \]

            \State Sample small sequence-specific variation:
            \[
            \Delta_{A,s} \sim \mathcal{N}(0, \sigma_\Delta^2 I_{d_\ell}), \quad \sigma_\Delta \ll \sigma_U
            \]
        \EndFor

        \State Construct block-diagonal transformation:
        \[
        U_{i,s} =
        \text{blockdiag}(U_A + \Delta_{A,s})
        \]

        \State Sample binary accessibility mask:
        \[
        \mathbf{M}_{i,s} \in \{0,1\}^{2^{M-1}}
        \]

        \State Apply mask:
        \[
        \tilde U_{i,s} =
        \mathbf{M}_{i,s} \odot U_{i,s}
        \]

        \State Apply modality-specific rotation:
        \[
        W_{i,s} =
        R_i \tilde U_{i,s}
        \]

        \State Generate sequence embedding:
        \[
        Z_{i,s} =
        \tanh(\gamma W_{i,s} Z_i)
        \]

    \EndFor

\EndFor

\State \Return $\{Z_i\}_{i \in \mathcal{M}}$, where $Z_i \in \mathbb{R}^{S \times E}$

\end{algorithmic}
\end{algorithm}

\paragraph{Evaluation details} 
\label{delta_metrics}
We evaluate disentanglement performance by measuring whether each learned component $(\hat{\mathbf{u}}_{ij}, \hat{\mathbf{s}}_{ij})$ predicts the label associated with its corresponding ground-truth component $(\mathbf{u}_{ij}, \mathbf{s}_{ij})$, while remaining uninformative about labels associated with other cross components. To assess this behavior, we report confusion matrices for each model across all values of $M$.

For a more holistic metric, we validate each model's performance based on its deviation from the ideal performance using $\Delta_{model}$, which is defined as follows:
\[
\Delta_{\text{model}}
=
\frac{1}{4}\Bigl(
\underbrace{|A_{\hat{\mathbf{u}}_{ij}\to y_{u_{ij}}}-100|+|A_{\hat{\mathbf{s}}_{ij}\to y_{s_{ij}}}-100|}_{\text{intra-component prediction}}
+
\underbrace{|A_{\hat{\mathbf{u}}_{ij}\to y_{s_{ij}}}-50|+|A_{\hat{\mathbf{s}}_{ij}\to y_{u_{ij}}}-50|}_{\text{cross-component leakage}}
\Bigr),
\]
where $A_{\mathbf{u}_{ij}\to y_{u_{ij}}} \coloneqq \mathbb{E}[Acc(\mathbf{u}_{ij}\to y_{u_{ij}})]$. 

\paragraph{Training and Hyperparameter specifications}

For the experiments, we test for the number of modalities $M=\{2, 3, 4, 5\}$. For every case, we generate $N=150.000$ training samples per modality and we set $E = 64$ and $S = 8$. We also choose $\sigma^2_{U} = 0.5$ and $\sigma_{\Delta} = 0.005 \ll \sigma_{U}$. For the scaling factor of the non-linearity we, select $\gamma = 0.3$.

The architectural specifications of each model are demonstrated in Table \ref{tab:training_models_synthetic} (a). In order to perform a fair scaling experiment, as $M$ grows, the encoder backbone of each disentanglement module for RePercENT, changes only in terms of the number of the required latent slots in $\mathcal{H}$. Specifically, there are 2 dedicated slots for the case of $M = 2$, 4 for the case of $M = 3$ etc. For the baselines, the capacity per separate encoder remains the same.

For the training specifications, see Table \ref{tab:training_models_synthetic} (b). For every $M$ the generated dataset is split into training, validation, and test sets with a ratio of $0.7/0.1/0.2$, respectively. For the disentanglement loss, we keep $\alpha$ and $\beta$ fixed throughout training. The coefficient $\lambda$ is annealed using an exponential scheduler and updated once per training iteration. Lastly, each model is trained for $150$ epochs to ensure convergence. We select the best checkpoint based on the lowest validation loss observed during training and report its performance on the test set.

All models are trained on three different train/test splits, using two random seeds for each split. We aggregate the results across all runs and report the average detection accuracy as mean $\pm$ standard deviation.

\begin{table}[ht]
\centering
\small
\caption{
Architecture and training specifications used for the synthetic dataset experiments. For each model, table (a) reports the single encoder parameters, which remain fixed across the different values of $M$. The training specifications are the same across all models.
}
\label{tab:training_models_synthetic}
\begin{minipage}[t]{0.48\linewidth}
\centering
\caption*{\textbf{(a) Model architecture}}
\begin{tabular}{p{4.4cm}c}
\toprule
\textbf{Model $\sim$ Hyperparameter} & \textbf{Value} \\
\midrule

\rowcolor{gray!10} 
\multicolumn{2}{l}{\textit{RePercENT}}\\
Embedding input dim & $64$ \\
Embedding sequence length & $8$ \\
Depth in $\mathcal{D}_i$ & $3$ \\
Cross-attention heads & $2$ \\
Transformer latent heads & $2$ \\
Latent dim $D_l$ & $32$ \\

\midrule
\rowcolor{gray!10} 
\multicolumn{2}{l}{\textit{gMLP}}\\
Embedding input dim & $64$ \\
Embedding sequence length & $8$ \\
Feed-forward hidden dim & $64$ \\
Projection head dim & $32$ \\
Number of layers & $2$ \\
Activation & GELU \\

\midrule
\rowcolor{gray!10} 
\multicolumn{2}{l}{\textit{GRU}}\\
Embedding input dim & $64$ \\
Hidden dim & $64$ \\
Latent dim & $32$ \\
Number of layers & $1$ \\
Activation & ReLU \\

\midrule
\rowcolor{gray!10} 
\multicolumn{2}{l}{\textit{MLP}}\\
Input dim & $512$ \\
Hidden dim & $64$ \\
Output dim & $32$ \\
Activation & ReLU \\

\bottomrule
\end{tabular}
\end{minipage}
\hfill
\begin{minipage}[t]{0.5\linewidth}
\centering
\caption*{\textbf{(b) Training specification}}
\begin{tabular}{lc}
\toprule
\textbf{Component} & \textbf{Value} \\
\midrule

\rowcolor{gray!10} 
\multicolumn{2}{l}{\textit{Disentanglement loss}} \\
$\alpha$ & $2.0$ \\
$\beta$ & $0.1$ \\
$\lambda$ start value & $0.1$ \\
$\lambda$ end value & $1.0$ \\
$\lambda$ annealing iterations & $8000$ \\
$\lambda$ start iteration & $800$ \\

\midrule
\rowcolor{gray!10} 
\multicolumn{2}{l}{\textit{Optimizer}} \\
Optimizer & Adam \\
Learning rate & $8 \times 10^{-4}$ \\
Weight decay & $10^{-4}$ \\

\midrule
\rowcolor{gray!10} 
\multicolumn{2}{l}{\textit{Training}} \\
Number of epochs & $150$ \\
Batch size & $1024$ \\

\bottomrule
\end{tabular}
\end{minipage}
\end{table}

\paragraph{Augmentations} During training, we augment the synthetic inputs using a simple stochastic augmentation module. Given an input tensor $Z$, the module returns an augmented tensor $Z_{\mathrm{a}}$ by applying one of the following transformations: Gaussian noise injection with scale $10^{-3}$, feature swapping, or random feature dropping. For random feature dropping, entries are masked according to a drop scale of $10$.

\subsection{MultiBench overview}
\label{multibench_overview}
MultiBench~\citep{liang_multibench_2021} is a large-scale benchmark for multimodal representation learning that unifies 15 datasets spanning 10 modalities, 20 prediction tasks, and 6 research areas under standardized data loading, preprocessing, and evaluation protocols. We use three of its four affective computing datasets, all of which share the \texttt{Text}, \texttt{Video}, and \texttt{Audio} modalities, allowing us to evaluate both two- and three-modality settings. MOSI contains 2,199 opinion video segments from YouTube vlogs, annotated with sentiment intensity in $[-3, 3]$. MOSEI is its larger counterpart, with 22,777 segments annotated for both sentiment and emotion. UR-FUNNY contains 16,514 segments from TED talks, labeled for the presence of humor. We formulate all three tasks as binary classification. In each task, the relevant cues are distributed across what is said, how it is said, and the speaker's facial expressions, so the three modalities carry complementary information, which makes these datasets well suited to our setting. All three datasets are provided as pre-extracted feature sequences, such as GloVe word embeddings for text, Facet facial features for video, and COVAREP acoustic features for audio. We use the same train/validation/test splits as \citet{liang_factorized_2023}.

\paragraph{Training and hyperparameter specifications.}
For all baselines (FactorCL, FOCAL, DisentangledSSL, and JointOpt), we follow the authors' implementations and use the same 5-layer Transformer encoder as \citet{liang_factorized_2023}. In contrast, RePercENT operates directly on sequences of embeddings and therefore requires no backbone encoder. All models use the same representation dimensionality of 300 across the three datasets, and model selection is based on validation performance. Table~\ref{tab:training_multibench_models} reports the training hyperparameters for RePercENT, which we tuned on the validation set. For the baselines, we adopt the loss-specific hyperparameters reported by their authors, while using the same optimizer and training schedule as RePercENT.
\begin{table}[h!]
\centering
\small
\caption{
Architecture and training specifications used for RePercENT on the MultiBench sentiment and humor detection.
}
\label{tab:training_multibench_models}
\begin{minipage}[t]{0.68\linewidth}
\centering
\caption*{\textbf{(a) Model architecture}}
\small
\setlength{\tabcolsep}{5pt}
\renewcommand{\arraystretch}{1.15}
\begin{tabularx}{\linewidth}{@{}>{\raggedright\arraybackslash}X ccc@{}}
\toprule
 & \multicolumn{3}{c}{\textbf{Encoder}} \\
\cmidrule(l){2-4}
\textbf{Hyperparameter} & \texttt{Image} & \texttt{Text} & \texttt{Audio} \\
\midrule
\rowcolor{gray!10}
\multicolumn{4}{@{}l@{}}{\textit{RePercENT}} \\
Embedding input dim       & $20\,/\,35\,/\,371^{\dagger}$ & 300 & $5\,/\,74\,/\,81^{\dagger}$ \\
Embedding sequence length & 50  & 50  & 50  \\
Depth in $\mathcal{D}_i$  & 3   & 3   & 3   \\
Cross-attention heads     & 2   & 2   & 2   \\
Transformer latent heads  & 2   & 2   & 2   \\
Latent dim $D_l$          & 300 & 300 & 300 \\
\bottomrule
\end{tabularx}

\vspace{3pt}
\raggedright\footnotesize
$^{\dagger}$\,MOSI\,/\,MOSEI\,/\,UR-FUNNY, respectively.
\end{minipage}
\hfill
\begin{minipage}[t]{0.3\linewidth}
\centering
\caption*{\textbf{(b) Training specification}}
\begin{tabular}{lc}
\toprule
\textbf{Component} & \textbf{Value} \\
\midrule

\rowcolor{gray!10} 
\multicolumn{2}{l}{\textit{Disentanglement loss}} \\
$\alpha$ & $1.0$ \\
$\beta$ & $1.0$ \\
$\lambda$ value & $0.001$ \\

\midrule
\rowcolor{gray!10} 
\multicolumn{2}{l}{\textit{Optimizer}} \\
Optimizer & Adam \\
Learning rate & $1 \times 10^{-4}$ \\
Weight decay & $10^{-5}$ \\

\midrule
\rowcolor{gray!10} 
\multicolumn{2}{l}{\textit{Training}} \\
Number of epochs & $100$ \\
Batch size & $512$ \\

\bottomrule
\end{tabular}
\end{minipage}
\end{table}

\subsection{IRFL overview}
\label{irfl_overview}
\paragraph{Dataset creation} The Image Recognition of Figurative Language (IRFL) dataset \citep{yosef2023irfl}. IRFL is a multimodal benchmark that comprises approximately 6,690 instances designed to test the comprehension of figurative language, including idioms, metaphors, and similes. Each text phrase is paired with images categorized by their relationship to the text. The latter falls in one of the categories: figurative, literal, partial literal, or figurative+literal. We extract the complete dataset including idioms, metaphors, and similes, and keep only the rows that correspond to figurative images and text. For instances where no definition is provided, we manually complete the missing annotations. Afterwards, each image is encoded using the OpenAI CLIP-ViT-B/32 vision transformer, while captions and definitions are encoded using the corresponding CLIP text transformer. We preserve the sequence-like structure of both modalities rather than using a single pooled embedding. For images, we retain the patch-level embeddings before the final projection layer. For text, we retain the contextual token embeddings from the final transformer layer instead of using only the pooled end-of-text (EOT) representation. We split the data into 2,594 training and 786 test samples, ensuring that no image or caption is shared across the two sets.

\paragraph{Detection task overview}
\label{detection_task_overview}
The figurative detection task evaluates Vision and Language Pre-Trained Models’ (VL-PTMs) ability to choose the image that best visualizes the meaning of a figurative expression. Figure \ref{fig: irfl_detection_task} provides two examples of the task, where even though there is shared information with most of the candidate images, only one describes best the figurative phrase. It is a very interesting task, as state-of-the-art VL models, such as CLIP-VIT-B/32, CLIP-RN50x64~\citep{pmlr-v139-radford21a}, BLIP~ \citep{Li2022BLIPBL} performed substantially worse than humans. 

\begin{figure}[ht]
    \centering
    \includegraphics[width=0.8\linewidth]{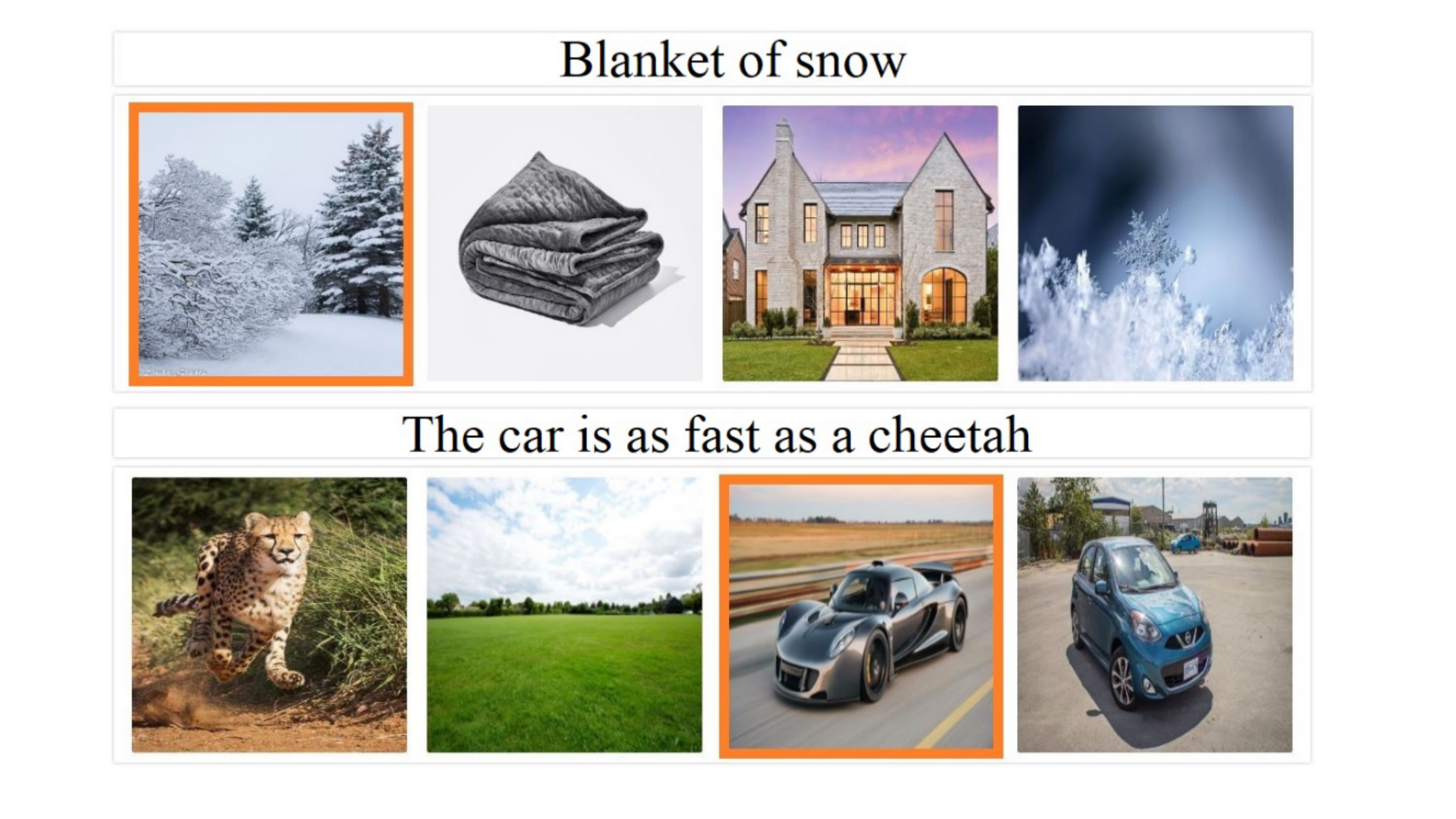}
    \caption{Adapted from \cite{yosef2023irfl}. Examples of the multimodal figurative language detection task for idiom, metaphor, and simile. The input is a figurative phrase and four candidate images (for idiom, we also show the definition). The correct answer is marked with an orange square.}
    \label{fig: irfl_detection_task}
\end{figure}

We evaluate all models in a zero-shot manner. Let $z_{\text{text}}$ denote the text query representation, derived from either the \texttt{Caption} alone or the concatenated \texttt{Caption} $\oplus$ \texttt{Definition}. For a given query, we compute the pairwise similarity between $z_{\text{text}}$ and the target figurative image representation $z_{\text{image}}$, as well as the representations of the three distractor images $z_{\text{dist}, i}$. A successful prediction requires the target similarity to strictly exceed all distractor similarities:
\[
\text{sim}(z_{\text{text}}, z_{\text{image}}) > \max_{i \in \{1,2,3\}} \text{sim}(z_{\text{text}}, z_{\text{dist}, i}).
\]

\paragraph{Training and Hyperparameter specifications}

For the alignment approach, we fine-tuned OpenAI CLIP ViT-B/32 in two settings: projection-only and end-to-end. In the projection-only setting, the CLIP image and text encoders remain frozen, and only the final visual and text projection layers were trained using a learning rate of $2 \times 10^{-4}$. In the end-to-end setting, all CLIP parameters are trainable and fine-tuned using a learning rate of $10^{-5}$ to avoid overfitting. Training in both cases uses the standard symmetric CLIP contrastive loss, AdamW optimization, cosine learning-rate decay with $10\%$ warmup, and gradient clipping.

Similarly to the synthetic setup, all disentanglement baselines are trained under the same training protocol as RePercENT. The separate encoder architectures as well as the training parameters are reported in Table \ref{tab:training_irfl_models}. Note that the two text modalities, \texttt{Caption} and \texttt{Definition}, share the same encoder architecture but are encoded by separate encoder instances with independent parameters.

All models are trained with five independent random seeds, and we report detection accuracy as the mean $\pm$ standard deviation across runs.

\begin{table}[ht]
\centering
\small
\caption{
Architecture and training specifications used of the disentanglement models for the IRFL Detection task experiments. 
}
\label{tab:training_irfl_models}
\begin{minipage}[t]{0.68\linewidth}
\centering
\caption*{\textbf{(a) Model architecture}}
\begin{tabular}{p{4.4cm}cc}
\toprule
\textbf{Model $\sim$ Hyperparameter} 
& \makecell{\texttt{Image} \\ \textbf{Encoder}} 
& \makecell{\texttt{Text} \\ \textbf{Encoder}} \\
\midrule

\rowcolor{gray!10} 
\multicolumn{3}{l}{\textit{RePercENT}}\\
Embedding input dim & $768$ & $512$ \\
Embedding sequence length & $49$ & $77$ \\
Depth in $\mathcal{D}_i$ & $2$ & $2$ \\
Cross-attention heads & $2$ & $2$ \\
Transformer latent heads & $1$ & $1$ \\
Latent dim $D_l$ & $256$ & $256$ \\

\midrule
\rowcolor{gray!10} 
\multicolumn{3}{l}{\textit{gMLP}}\\
Embedding input dim & $768$ & $512$ \\
Embedding sequence length & $49$ & $77$ \\
Feed-forward hidden dim & $352$ & $352$ \\
Projection head dim & $256$  & $256$ \\
Number of layers & $2$ & $2$ \\
Activation & GELU & GELU \\

\midrule
\rowcolor{gray!10} 
\multicolumn{3}{l}{\textit{GRU}}\\
Embedding input dim & $768$ & $512$ \\
Hidden dim & $256$ & $256$ \\
Latent dim & $256$ & $256$ \\
Number of layers & $1$ & $1$ \\
Activation & ReLU & ReLU \\

\bottomrule
\end{tabular}
\end{minipage}
\hfill
\begin{minipage}[t]{0.3\linewidth}
\centering
\caption*{\textbf{(b) Training specification}}
\begin{tabular}{lc}
\toprule
\textbf{Component} & \textbf{Value} \\
\midrule

\rowcolor{gray!10} 
\multicolumn{2}{l}{\textit{Disentanglement loss}} \\
$\alpha$ & $6.0$ \\
$\beta$ & $2.0$ \\
$\lambda$ value & $4.0$ \\

\midrule
\rowcolor{gray!10} 
\multicolumn{2}{l}{\textit{Optimizer}} \\
Optimizer & Adam \\
Learning rate & $8 \times 10^{-4}$ \\
Weight decay & $10^{-4}$ \\

\midrule
\rowcolor{gray!10} 
\multicolumn{2}{l}{\textit{Training}} \\
Number of epochs & $40$ \\
Batch size & $256$ \\

\bottomrule
\end{tabular}
\end{minipage}
\end{table}

\paragraph{Augmentations} To facilitate training, and avoid the inference of the CLIP model, we additionally pre-compute the augmentations used for training. Specifically, we augment training images using a fixed set of visually conservative transformations: horizontal flip, vertical flip, Gaussian blur, Gaussian noise, and posterization. Flip-based augmentations are combined with a mild random resized crop using the scale range $[0.95, 1.0]$ and bicubic interpolation to avoid removing the main object. Horizontal and vertical flips are applied with probability $0.95$. Gaussian blur uses a kernel size of $27$, Gaussian noise uses mean $0$ and standard deviation $0.1$ with clipping, and posterization uses $3$ bits with probability $0.95$. We do not apply hue or saturation jitter, since color information can be essential for figurative language. Figure \ref{fig:image_augmentations_irfl} visualizes the result of two augmentations using Gaussian noise and posterization. For the text augmentations, we only add minor formatting and neutral variation in order to fully preserve the context. During training, for each image or text, we sample one of the pre-computed embedding representations of their corresponding augmented views.
\begin{figure}[ht!]
    \centering
    \includegraphics[width=\textwidth]{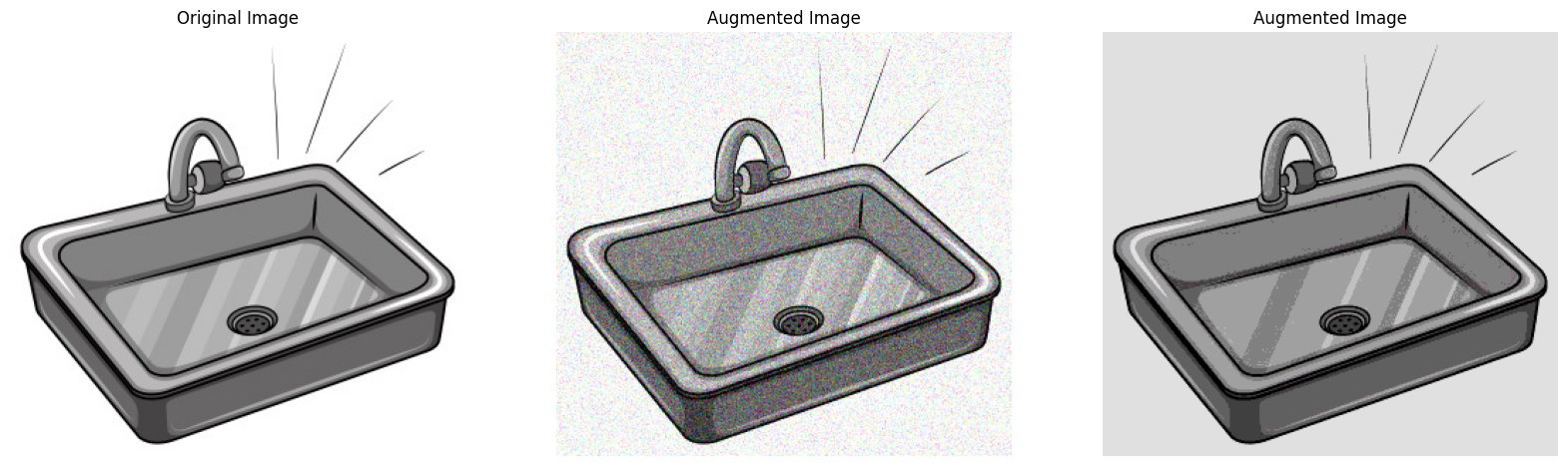}
    \caption{Example of two image augmentation variants used for the figurative language task. The original image from the IRFL dataset is depicted on the left and two additional augmented images are generated.}
    \label{fig:image_augmentations_irfl}
\end{figure}
\subsection{TCGA dataset extraction}
\label{tcga_dataset_extraction}
We utilize a pre-processed version of The Cancer Genome Atlas (TCGA) multimodal oncology cohort~\citep{weinstein2013cancer} by ~\citet{Tripathi2025.04.22.25326222} containing in total $10.857$ patient records. For our experiments, we consider the four modalities and FM embeddings reported in Table \ref{tab:tcga_modalities}. We group for each modality all the embeddings on a patient level, and preserve only the patients that have all four modalities available. Lastly, guided by the cancer-type distribution shown in Figure \ref{fig:cancer_type_distribution}, we restrict our analysis to cancer types with sufficient sample representation, and retain the 10 most prevalent cancer types to ensure reliable evaluation. This results in a final set of $5{,}732$ patients. The selected cancer types, together with brief descriptions of their clinical relevance, are summarized in Table \ref{tab:tcga_cancer_types}.

\begin{table}[ht]
\centering
\small
\setlength{\tabcolsep}{3pt}
\renewcommand{\arraystretch}{0.92}
\caption{
HONeYBEE TCGA modality embeddings used in the oncology experiments.
}
\label{tab:tcga_modalities}
\begin{tabularx}{\linewidth}{@{}lcl>{\raggedright\arraybackslash}X@{}}
\toprule
\textbf{Modality} 
& \textbf{Dim.} 
& \textbf{Encoder} 
& \textbf{Description} \\
\midrule

\texttt{Clinical} 
& $1024$ 
& Qwen3~\citep{Yang2025Qwen3TR}
& Patient-level clinical information, including structured and unstructured records such as demographics, laboratory values, medications, and clinical narratives. \\

\texttt{Pathology} 
& $1024$ 
& Qwen3~\citep{Yang2025Qwen3TR}
& Free-text pathology reports describing diagnostic and histopathological findings, processed as clinical text. \\

\texttt{WSI} 
& $1024$ 
& UNI~\citep{Chen2023AGS}
& Whole-slide histopathology images processed through tissue detection, stain normalization, patch extraction, and patch-level feature extraction. \\

\texttt{Molecular} 
& $48$ 
& SeNMo~\citep{waqas2024senmo}
& Multi-omics molecular profiles, including gene expression, DNA methylation, somatic mutations, miRNA, and protein expression. \\

\bottomrule
\end{tabularx}
\end{table}

\begin{table}
    \centering
    \small 
    \caption{Selected TCGA cancer types utilized in the Honeybee experimental evaluation.}
    \label{tab:tcga_cancer_types}
    \begin{tabular}{ll} 
    \toprule
    \textbf{TCGA Code} & \textbf{Cancer Type Description} \\
    \midrule
    GBM  & Glioblastoma multiforme \\
    LGG  & Brain Lower Grade Glioma \\
    HNSC & Head and Neck squamous cell carcinoma \\
    LUAD & Lung adenocarcinoma \\
    LUSC & Lung squamous cell carcinoma \\
    BRCA & Breast invasive carcinoma \\
    OV   & Ovarian serous cystadenocarcinoma \\
    PRAD & Prostate adenocarcinoma \\
    COAD & Colon adenocarcinoma \\
    KIRC & Kidney renal clear cell carcinoma \\
    \bottomrule
    \end{tabular}
\end{table}

\begin{figure}[h]
    \centering
    \includegraphics[width=\textwidth]{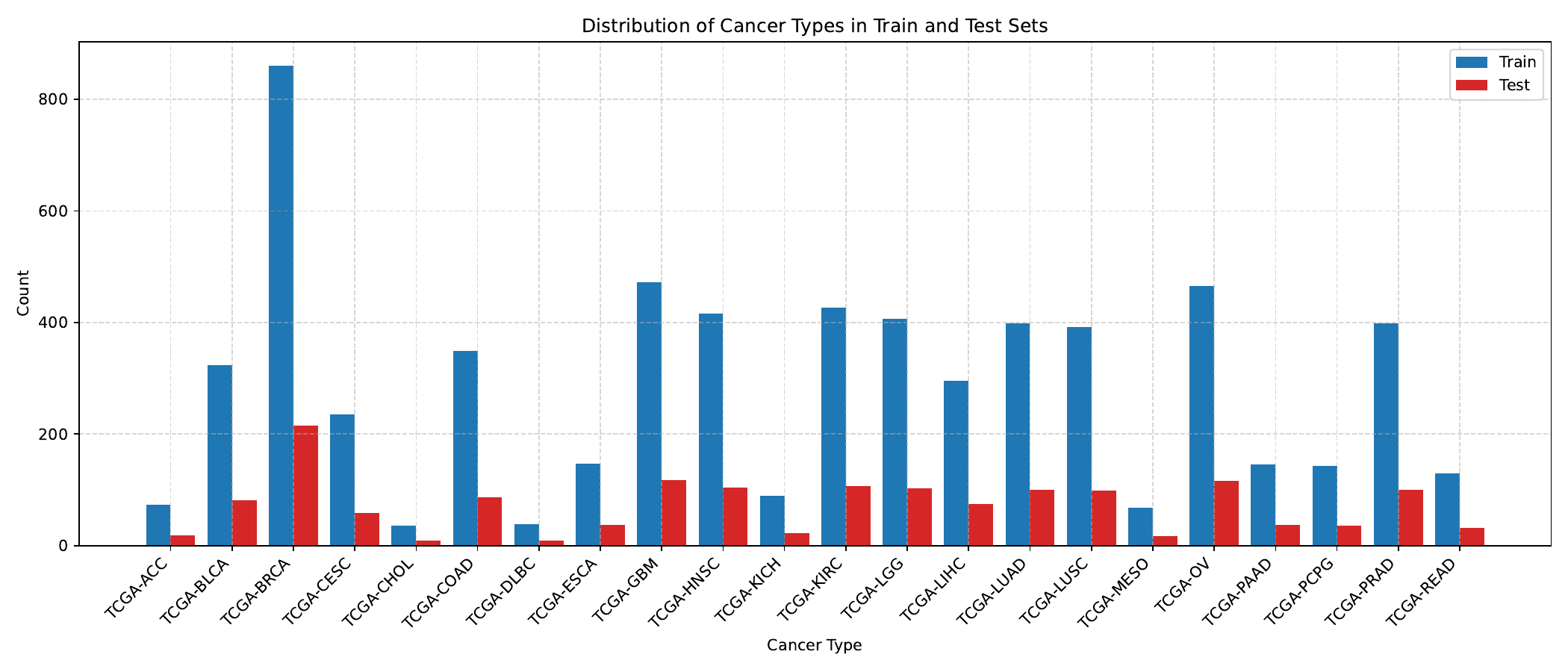}
    \caption{Cancer type patient count of the HONeYBEE extracted embeddings.}
    \label{fig:cancer_type_distribution}
\end{figure}

\paragraph{Training and Hyperparameter specifications}

We use all four modalities during training and extract both unique and shared representations for every modality pair. The specifications of the four RePercENT encoders are summarized in Table~\ref{tab:training_tcga} (a), while the corresponding training details are provided in Table~\ref{tab:training_tcga} (b). After training, for each extracted representation $\mathbf{u}_{ij}$ and $\mathbf{s}_{ij}$, as well as their concatenation, which together constitute the decomposition of modality $i$ with respect to modality $j$, we train linear probes and report cancer-type classification accuracy. Finally, we compare against linear probes trained on the original HONeYBEE embeddings, which represent the full modality-specific representations proposed by~\citet{Tripathi2025.04.22.25326222}.

\begin{table}[h]
\centering
\small
\setlength{\tabcolsep}{2.5pt}
\caption{
Architecture and training specifications used for the disentanglement models in the TCGA cohort experiments.
}
\label{tab:training_tcga}

\begin{minipage}[t]{0.69\textwidth}
\textbf{(a) Model architecture}

\vspace{0.5em}
\begin{tabular}{@{}p{3.1cm}cccc@{}}
\toprule
\textbf{Hyperparameters} 
& \makecell{\texttt{Clinical}\\\textbf{Encoder}} 
& \makecell{\texttt{Pathology}\\\textbf{Encoder}} 
& \makecell{\texttt{WSI}\\\textbf{Encoder}} 
& \makecell{\texttt{Molecular}\\\textbf{Encoder}} \\
\midrule

\rowcolor{gray!10} 
\multicolumn{5}{@{}l}{\textit{RePercENT}}\\
Embedding input dim  & $1024$ & $1024$ & $1024$ & $48$ \\
Input sequence length  & $1$    & $3$    & $16$   & $4$ \\
Depth in $\mathcal{D}_i$   & $8$    & $8$    & $8$    & $8$ \\
Cross-attention heads      & $2$    & $2$    & $2$    & $2$ \\
Transformer latent heads   & $1$    & $1$    & $1$    & $1$ \\
Latent dimension $D_l$     & $256$  & $256$  & $256$  & $256$ \\

\bottomrule
\end{tabular}
\end{minipage}
\hfill
\begin{minipage}[t]{0.3\textwidth}
\textbf{(b) Training specification}

\vspace{0.5em}
\begin{tabular}{@{}lc@{}}
\toprule
\textbf{Component} & \textbf{Value} \\
\midrule

\rowcolor{gray!10} 
\multicolumn{2}{@{}l}{\textit{Disentanglement loss}} \\
$\alpha$ & $5.0$ \\
$\beta$ & $0.5$ \\
$\lambda$ start value & $1.0$ \\
$\lambda$ end value & $3.0$ \\
$\lambda$ annealing iterations & $600$ \\
$\lambda$ start iteration & $200$ \\

\midrule
\rowcolor{gray!10} 
\multicolumn{2}{@{}l}{\textit{Optimizer}} \\
Optimizer & Adam \\
Learning rate & $2 \times 10^{-5}$ \\
Weight decay & $5 \times 10^{-4}$ \\

\midrule
\rowcolor{gray!10} 
\multicolumn{2}{@{}l}{\textit{Training}} \\
Number of epochs & $50$ \\
Batch size & $256$ \\

\bottomrule
\end{tabular}
\end{minipage}

\end{table}

\paragraph{Missing modality robustness setup}
\label{missing_modality_baselines}
For this experiment, we simulate missingness for the \texttt{WSI}-\texttt{Molecular} modality pair. At test time, we keep the \texttt{Molecular} data fully available and randomly drop WSI at varying missingness rates.

For RePercENT, we train separate linear probes on the decomposition of WSI w.r.t. Molecular, i.e., $D_{\mathrm{wsi}, \mathrm{mol}}$ and the extracted decomposition of Molecular w.r.t. WSI, i.e., $D_{\mathrm{mol}, \mathrm{wsi}}$ . At test time, we employ the simplest form of late fusion, where the predictions from the available modalities are combined by averaging their classifier scores.

We compare against the following three baselines trained on the original HONeYBEE modality embeddings:
\begin{itemize}
    \item \textbf{Late fusion + averaging:} We train separate linear probes on the raw WSI and molecular embeddings, and average the prediction scores of the modalities available at test time.
    \item \textbf{Early fusion + mean imputation:} We train a single probe on the concatenated embeddings of the two modalities. During evaluation, we replace the WSI missing instances, with the corresponding train-set modality mean.
    \item \textbf{Early fusion + modality dropout:} We use the same masked representation, but in this case the training set is augmented with randomly dropped modalities to expose the classifier to missing-modality patterns during training
\end{itemize}

Each method is applied both to the original HONeYBEE embeddings and to the representations of a matched-capacity baseline we term \textbf{ClipSur}, which controls for the additional capacity and training that RePercENT introduces. ClipSur consists of two 6-layer Transformer encoders for the \texttt{Mol--WSI} pair, with a combined parameter count within 10\% of RePercENT's, trained in a self-supervised manner with the standard CLIP contrastive loss~\citep{pmlr-v139-radford21a}. All methods use class-balanced logistic-regression probes. We report results across 10 uniformly spaced \texttt{WSI} availability levels, ranging from fully available ($100\%$) to fully missing ($0\%$). At each availability level, performance is averaged over 10 independently sampled stochastic missingness masks.

\subsection{Computational Resources}
All experiments were run on a single NVIDIA H200 GPU with 140GB of GPU memory, 96 physical CPU cores, 192 logical threads, 1.5TB of RAM, and 7.2TB of disk storage. The software environment used Python 3.10.12. The most computationally demanding setting was the synthetic five-modality experiment, which required approximately $3$ hours under the training regime described in Table~\ref{tab:training_models_synthetic}. Experiments on the real-world datasets were substantially cheaper, with all runs completing in under $1$ hour.

\section{Additional Experiments}

\subsection{Synthetic Dataset}

\label{synthetic_dataset_experiments}
In this section, we provide detailed linear probe results for all the baselines on the synthetic dataset, for each value of $M\in \{2,3,4,5\}$.

For the case of $M=2$ (Figure \ref{fig:pairwise_confusion_2m}), most architectures demonstrate a baseline ability to isolate components, characterized by high intra-component prediction and minimal cross-component leakage. However, architectural behavior diverges significantly as complexity increases.

\paragraph{MLP and GRU} The MLP exhibits the weakest specialization, with significantly lower diagonal accuracy even in the two-modality regime. Conversely, while the GRU improves the recovery of intended factors---particularly shared representations---its stronger diagonal performance is often offset by elevated off-diagonal values, indicating persistent entanglement. As $M$ increases (Figures \ref{fig:pairwise_confusion_mlp_4m}, \ref{fig:pairwise_confusion_gru_5m}), the GRU undergoes a fundamental component collapse, successfully recovering only the shared factors, while the unique components lose all discriminative power.
    
\paragraph{gMLP and RePercENT} These models demonstrate the most robust disentanglement performance as $M$ grows. Both consistently achieve the disentanglement objective where baselines fail, maintaining high diagonal dominance and near-chance off-diagonal structures well beyond the two-modality setting.
    
\paragraph{Scalability of RePercENT} Notably, RePercENT achieves performance parity with the gMLP while strictly adhering to the proposed pairwise slot-based design. The preservation of a clean, diagonal-dominant structure across all modality pairs demonstrates that RePercENT successfully overcomes the scalability bottleneck, maintaining granular disentanglement in high-dimensional multimodal settings. Moreover, our results indicate that performance at $M=2$ is an unreliable predictor of architectural robustness in higher-dimensional spaces. As $M$ increases, the information-theoretic complexity of interactions grows exponentially, introducing "interference" and potential "component collapse" risks that are absent in simpler settings. 

\begin{figure}[ht]
    \centering

    \begin{subfigure}[t]{0.48\linewidth}
        \centering
        \includegraphics[width=\linewidth]{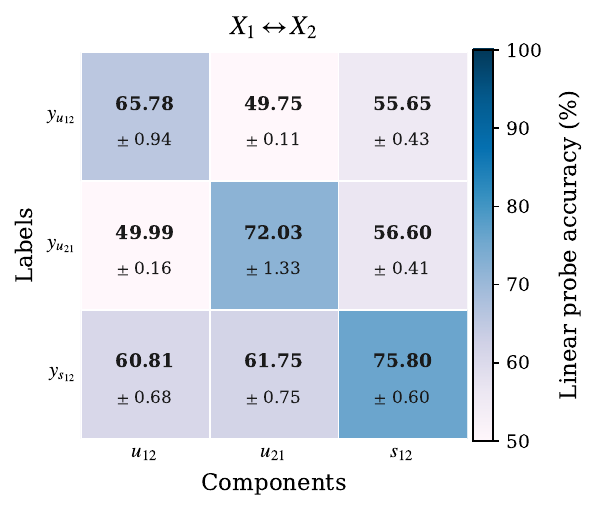}
        \caption{MLP}
        \label{fig:pairwise_confusion_mlp_2m}
    \end{subfigure}
    \hfill
    \begin{subfigure}[t]{0.48\linewidth}
        \centering
        \includegraphics[width=\linewidth]{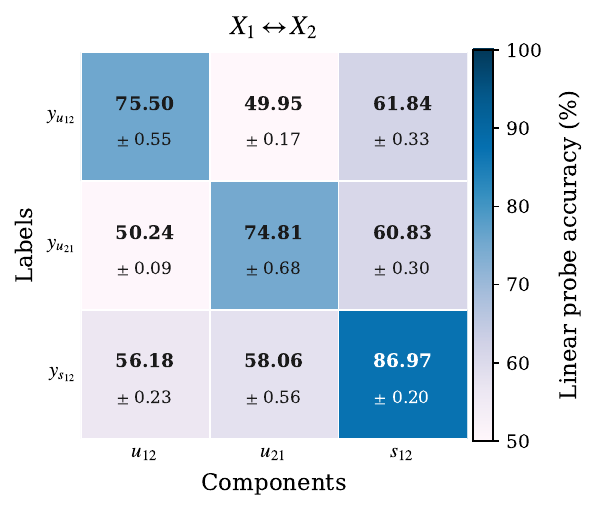}
        \caption{GRU}
        \label{fig:pairwise_confusion_gru_2m}
    \end{subfigure}

    \vspace{0.5em}

    \begin{subfigure}[t]{0.48\linewidth}
        \centering
        \includegraphics[width=\linewidth]{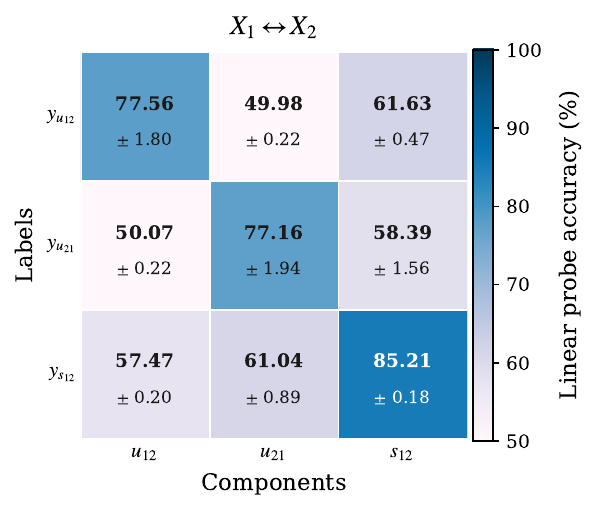}
        \caption{gMLP}
        \label{fig:pairwise_confusion_gmlp_2m}
    \end{subfigure}
    \hfill
    \begin{subfigure}[t]{0.48\linewidth}
        \centering
        \includegraphics[width=\linewidth]{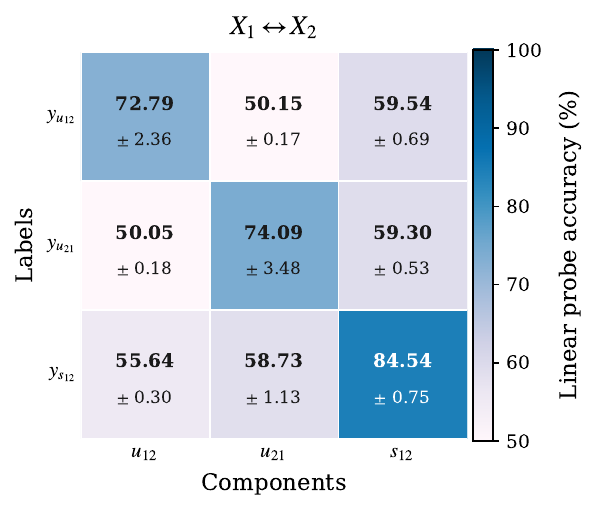}
        \caption{\textbf{RePercENT}}
        \label{fig:pairwise_confusion_repercent_2m}
    \end{subfigure}

    \caption{($\mathbf{M = 2}$) Pairwise confusion matrices for the synthetic setting with two modalities, shown for MLP, GRU, gMLP, and \textbf{RePercENT}. While all models largely separate unique and shared information, sequence-aware models yield stronger intra-component accuracy, as reflected by higher main-diagonal values.}
    \label{fig:pairwise_confusion_2m}
\end{figure}

\newpage

\begin{figure}[ht]
    \centering

    \begin{subfigure}[t]{\linewidth}
        \centering
        \includegraphics[width=0.9\linewidth]{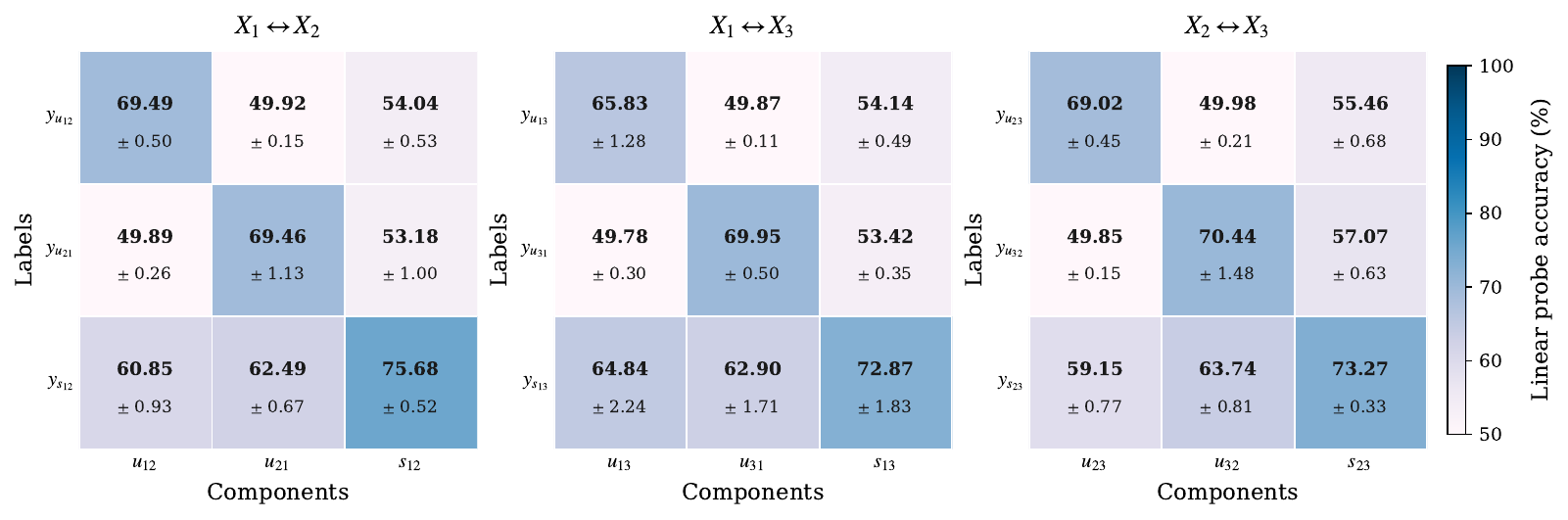}
        \caption{MLP}
        \label{fig:pairwise_confusion_mlp_3m}
    \end{subfigure}
    \vspace{0.5em}
    \begin{subfigure}[t]{\linewidth}
        \centering
        \includegraphics[width=0.9\linewidth]{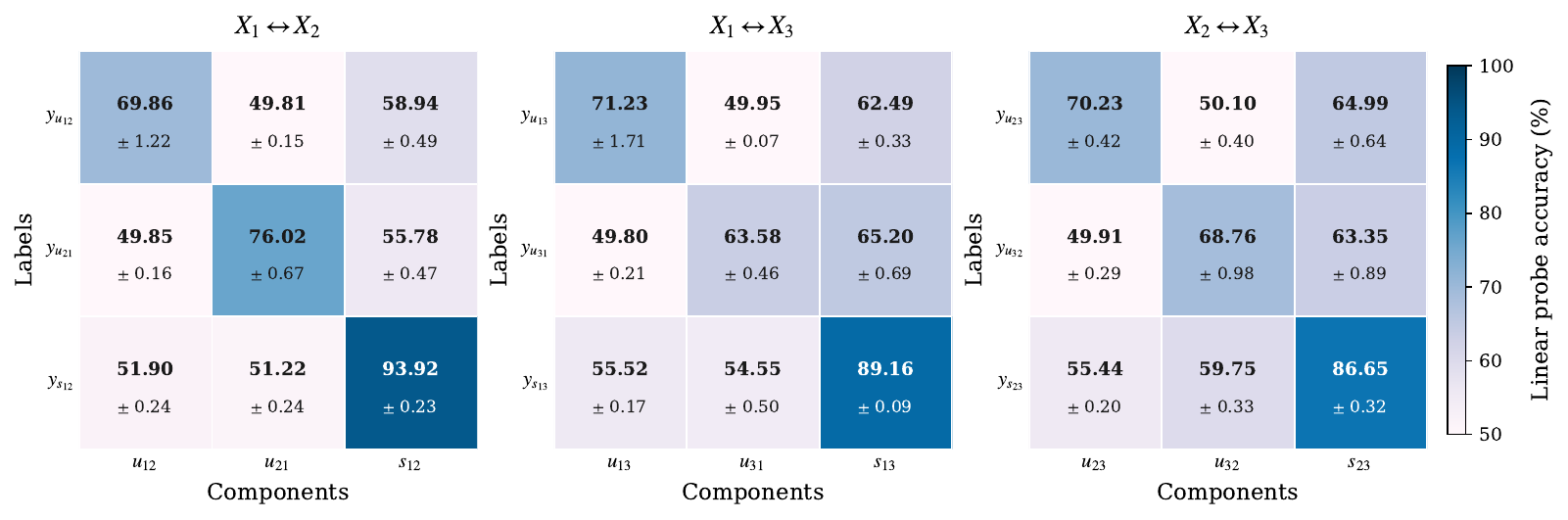}
        \caption{GRU}
        \label{fig:pairwise_confusion_gru_3m}
    \end{subfigure}

    \vspace{0.5em}

    \begin{subfigure}[t]{\linewidth}
        \centering
        \includegraphics[width=0.9\linewidth]{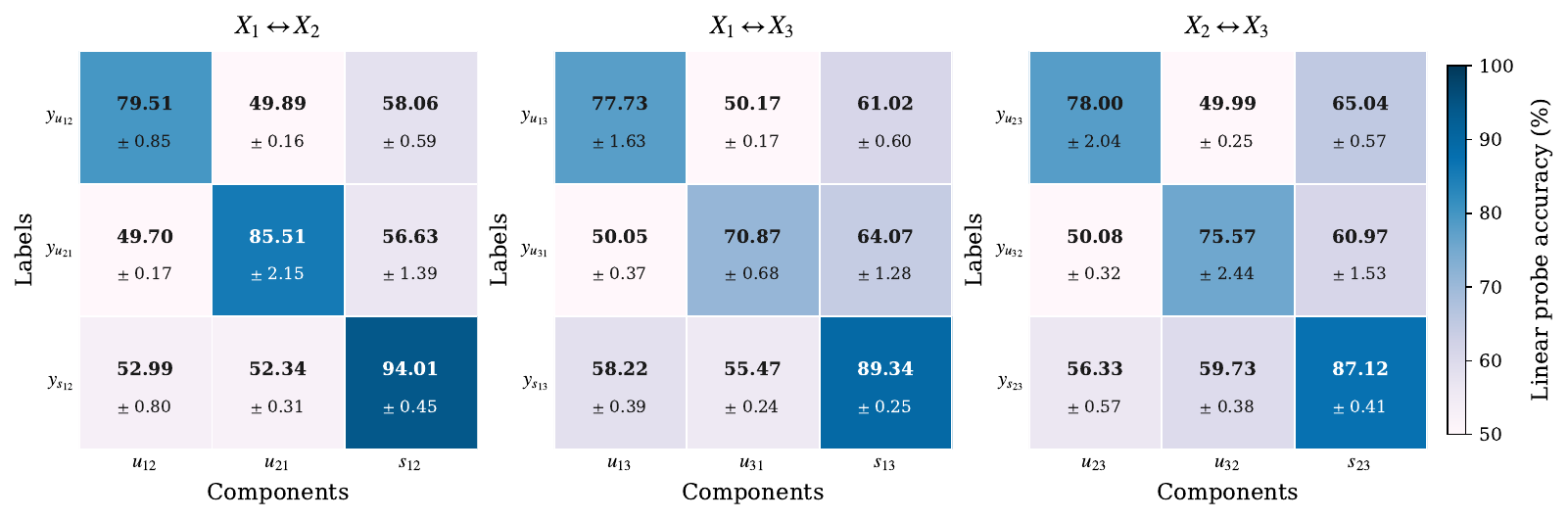}
        \caption{gMLP}
        \label{fig:pairwise_confusion_gmlp_3m}
    \end{subfigure}
    \vspace{0.5em}
    \begin{subfigure}[t]{\linewidth}
        \centering
        \includegraphics[width=0.9\linewidth]{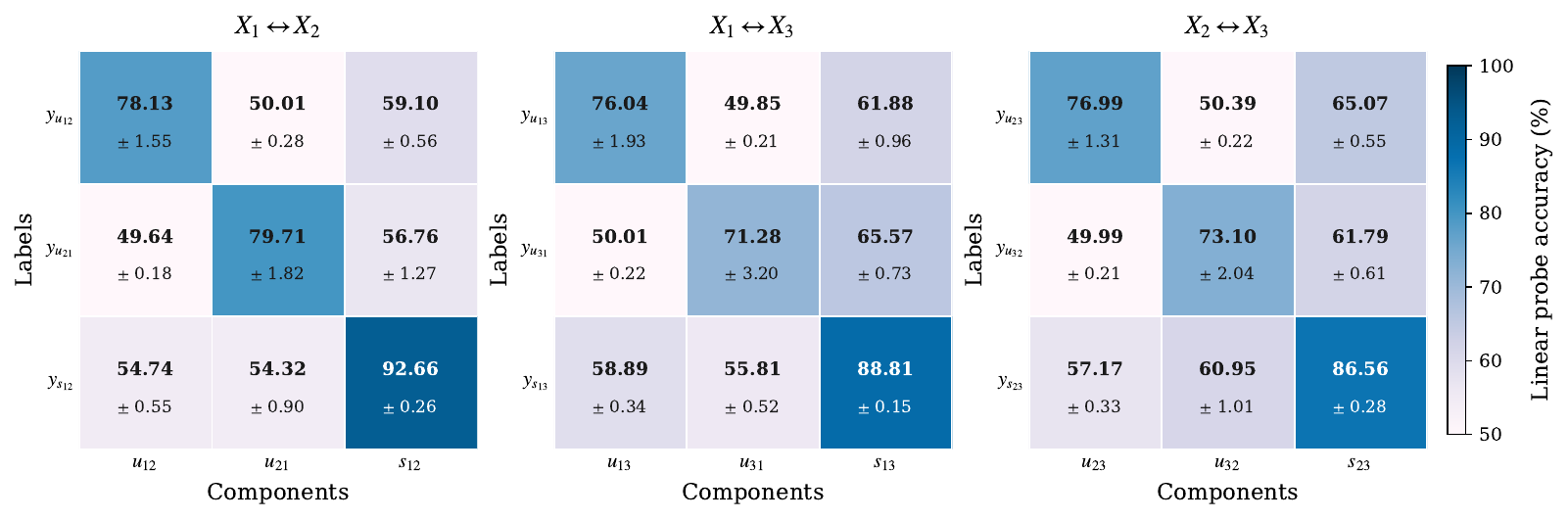}
        \caption{\textbf{RePercENT}}
        \label{fig:pairwise_confusion_repercent_3m}
    \end{subfigure}

    \caption{($\mathbf{M = 3}$) Pairwise confusion matrices for the synthetic setting with three modalities, shown for MLP, GRU, gMLP, and \textbf{RePercENT}. The performance of GRU and especially MLP degrades, while the gMLP and RePercENT present robust disentanglement performance.}
    \label{fig:pairwise_confusion_3m}
\end{figure}

\newpage

\begin{figure}[ht]
    \centering

    \begin{subfigure}[t]{\linewidth}
        \centering
        \includegraphics[width=\linewidth]{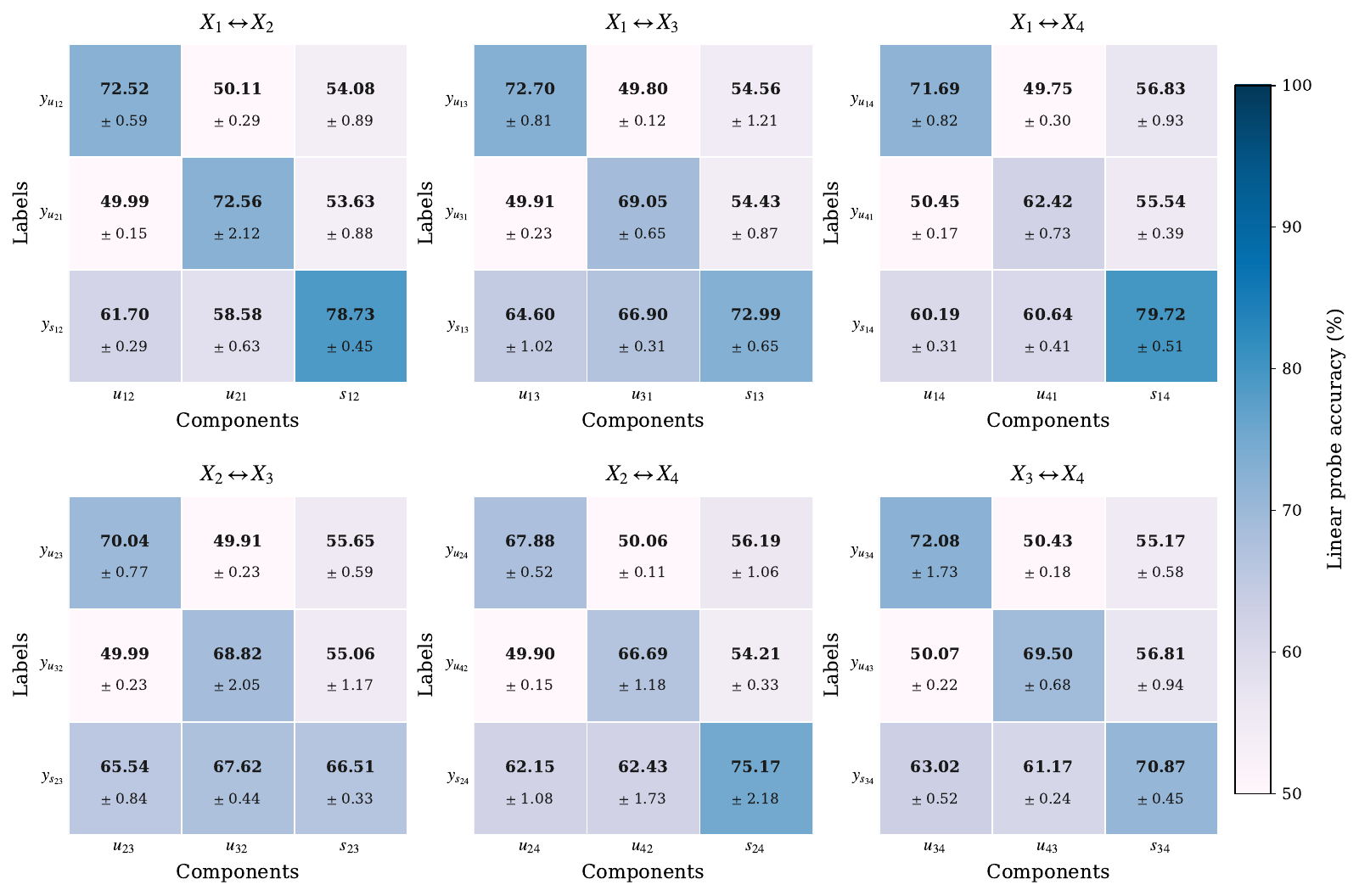}
        \caption{MLP}
        \label{fig:pairwise_confusion_mlp_4m}
    \end{subfigure}
    
    \vspace{1.5em}
     
    \begin{subfigure}[t]{\linewidth}
        \centering
        \includegraphics[width=\linewidth]{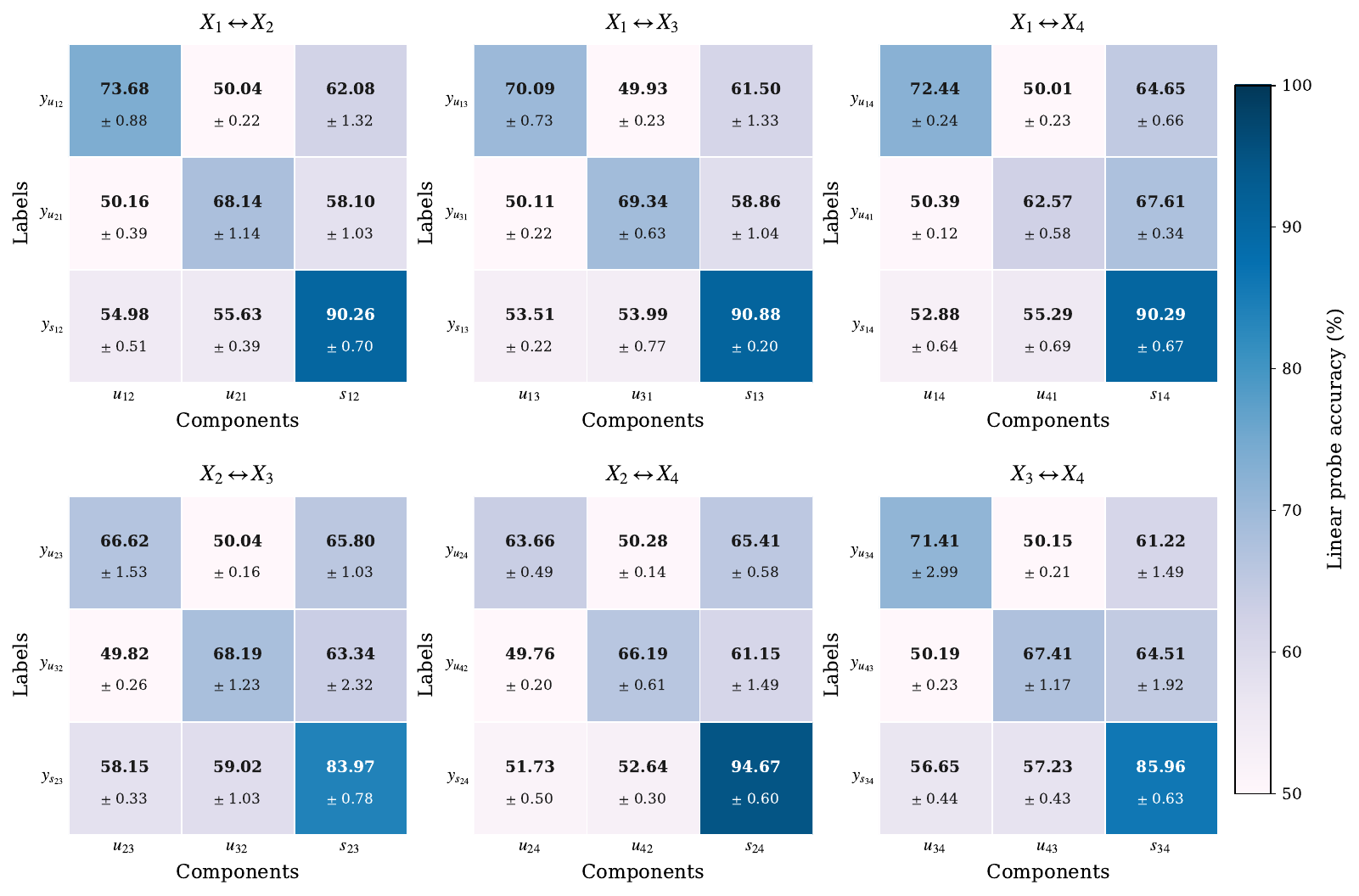}
        \caption{GRU}
        \label{fig:pairwise_confusion_gru_4m}
    \end{subfigure}
    \caption{($\mathbf{M = 4}$) Pairwise confusion matrices for the synthetic setting with four modalities, shown for MLP and GRU. The MLP is unable to recover the desired representations, as it exhibits substantial cross-component leakage and weak intra-component prediction, while the GRU preserves strong shared representations but weakly encodes unique components.}
    \label{fig:pairwise_confusion_mlp_gru_4m}
\end{figure}

\begin{figure}[ht]
    \centering

    \begin{subfigure}[t]{\linewidth}
        \centering
        \includegraphics[width=\linewidth]{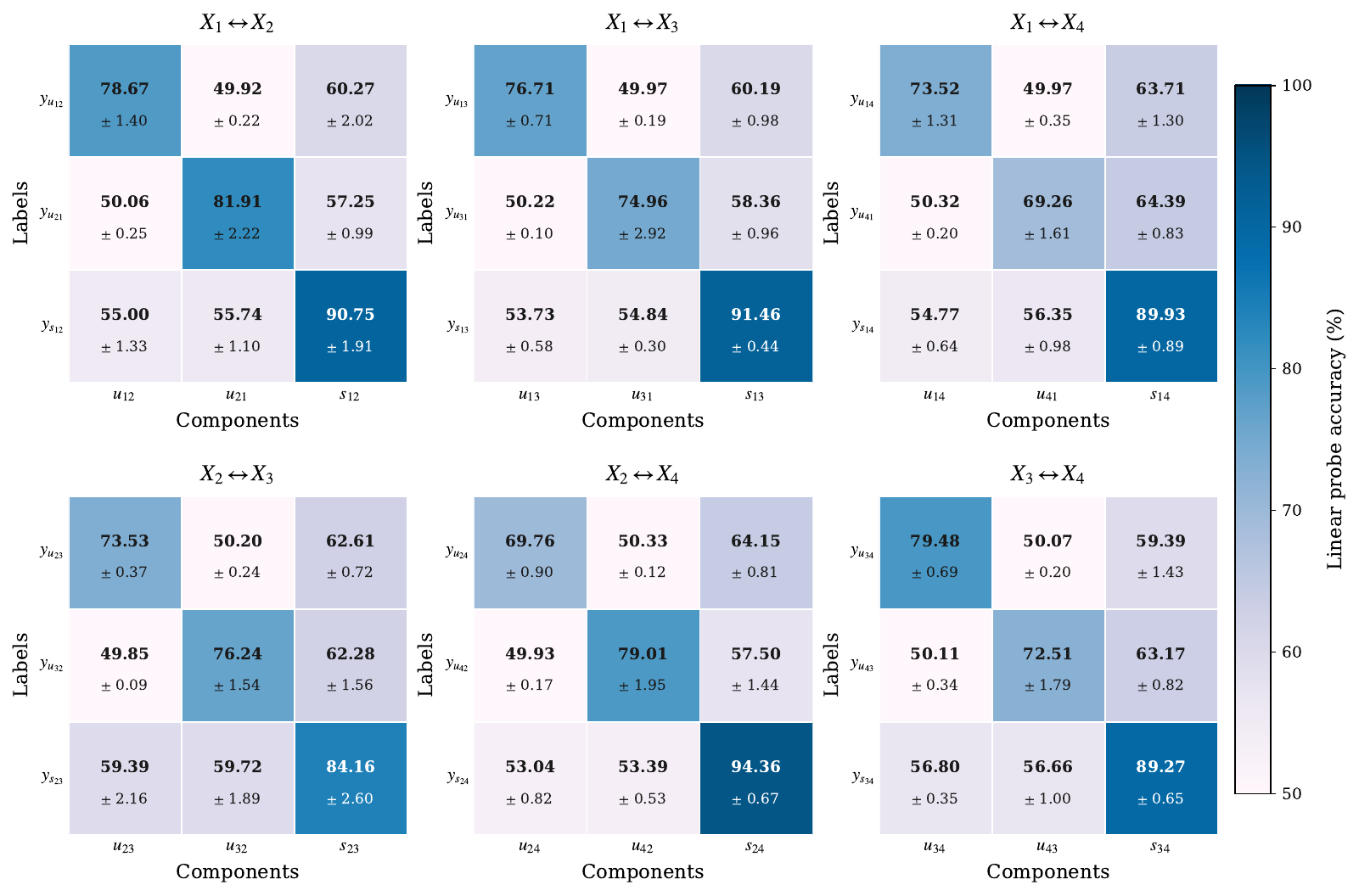}
        \caption{gMLP}
        \label{fig:pairwise_confusion_gmlp_4m}
    \end{subfigure}
    \vspace{0.5em}
    \begin{subfigure}[t]{\linewidth}
        \centering
        \includegraphics[width=\linewidth]{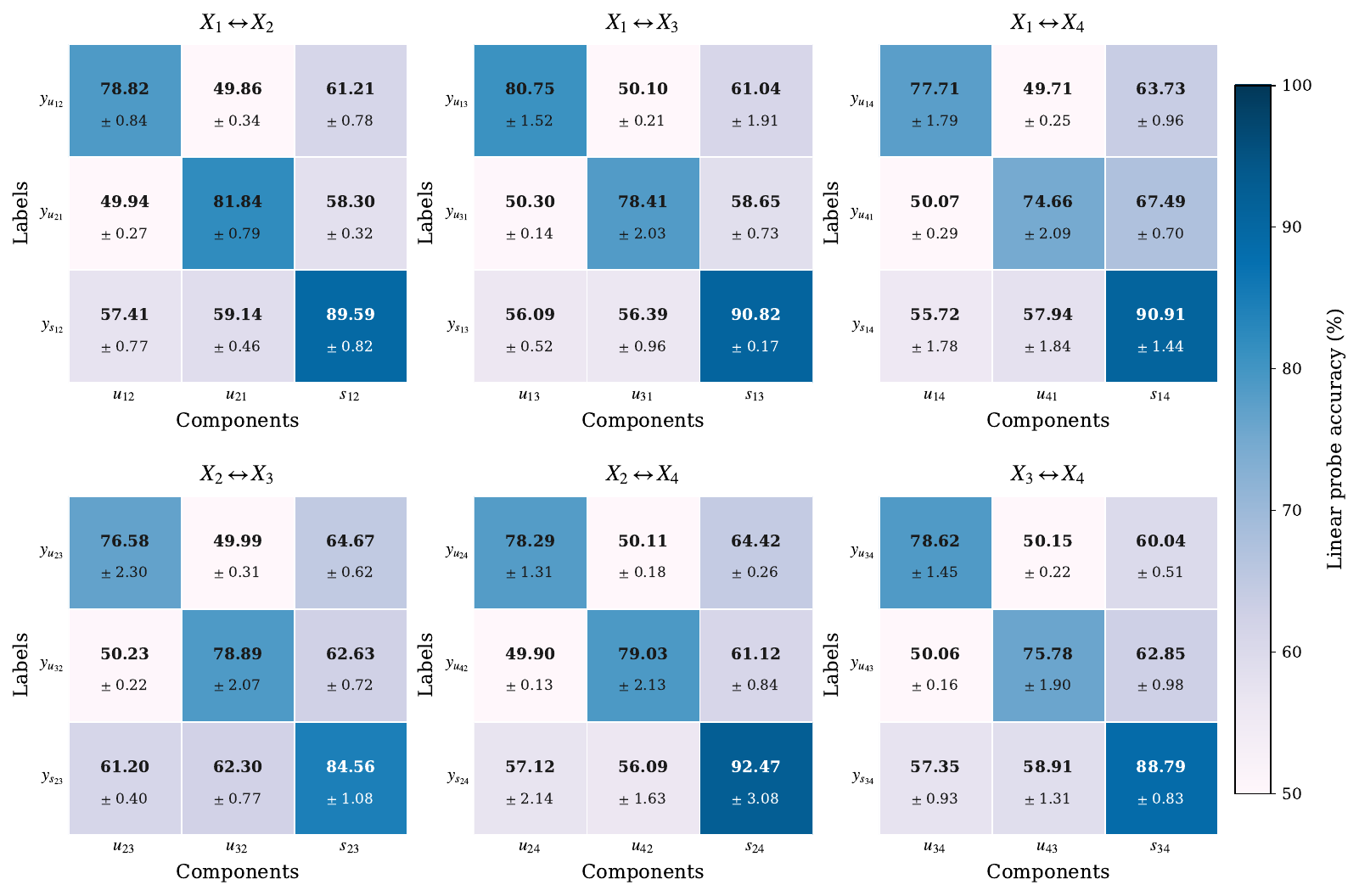}
        \caption{\textbf{RePercENT}}
        \label{fig:pairwise_confusion_repercent_4m}
    \end{subfigure}

    \caption{($\mathbf{M = 4}$) Pairwise confusion matrices for the synthetic setting with four modalities, shown for gMLP, and \textbf{RePercENT}. Both models achieve strong disentanglement, with RePercENT yielding slightly higher intra-component accuracy especially for the unique components, while gMLP exhibits marginally lower cross-component leakage.}
    \label{fig:pairwise_confusion_gmlp_repercent_4m}
\end{figure}

\newpage

\begin{figure}[ht]
    \centering

    \begin{subfigure}[t]{\linewidth}
        \centering
        \includegraphics[width=\linewidth]{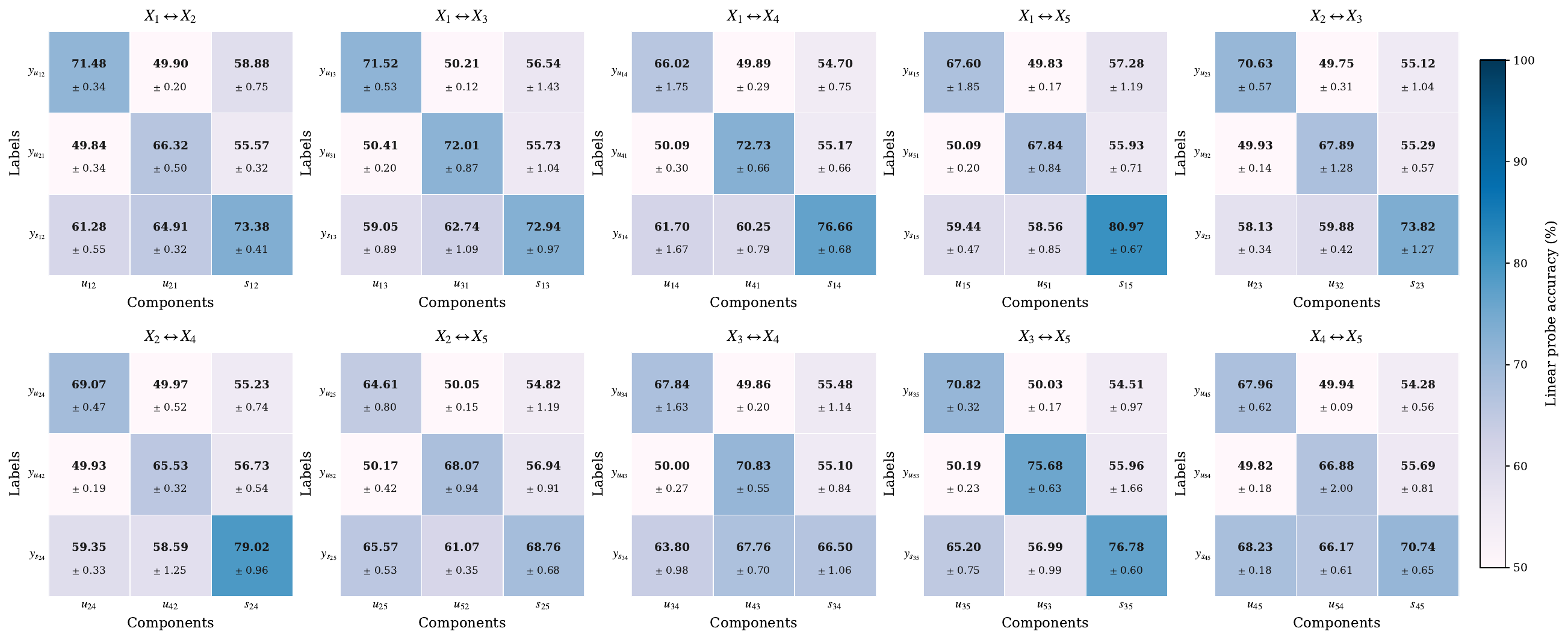}
        \caption{MLP}
        \label{fig:pairwise_confusion_mlp_5m}
    \end{subfigure}
    
    \vspace{1.5em}
     
    \begin{subfigure}[t]{\linewidth}
        \centering
        \includegraphics[width=\linewidth]{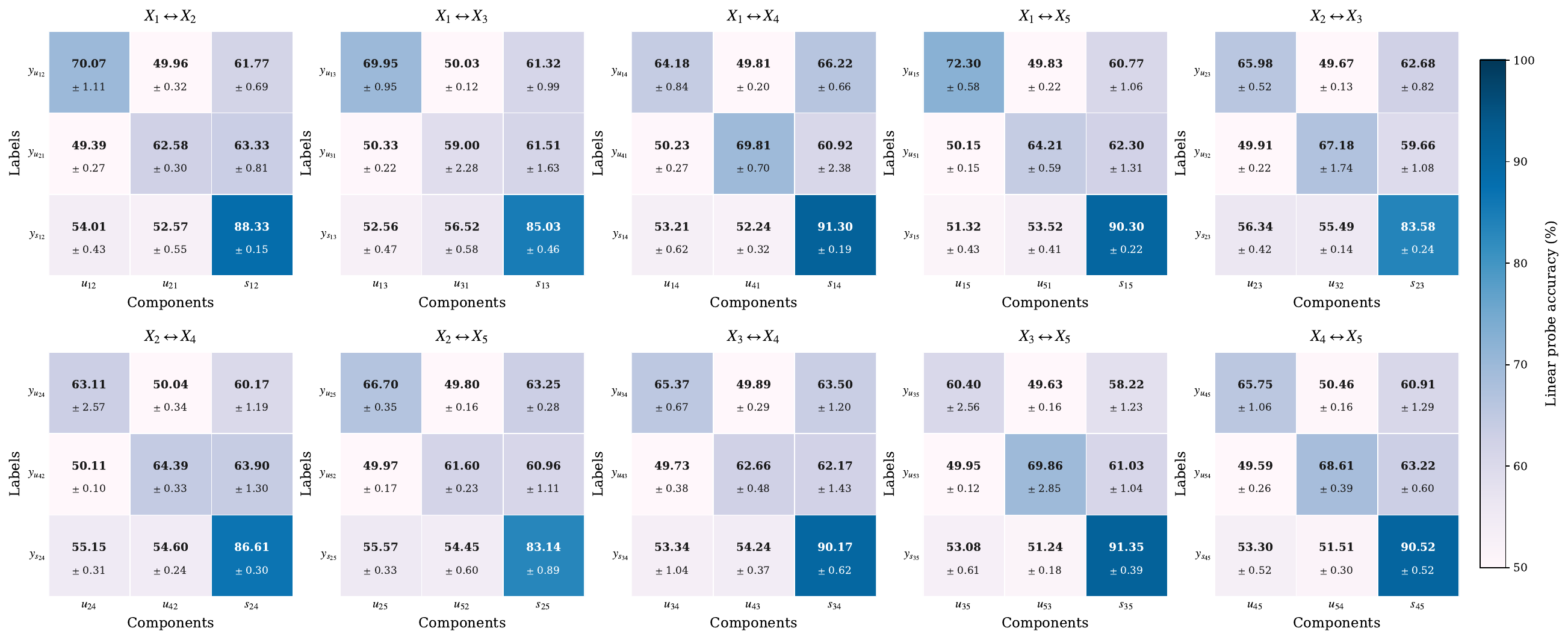}
        \caption{GRU}
        \label{fig:pairwise_confusion_gru_5m}
    \end{subfigure}
    \caption{($\mathbf{M = 5}$) Pairwise confusion matrices for the synthetic setting with five modalities, shown for MLP and GRU. Similarly to the case of $M=4$, the MLP is fails to recover the desired representations, while the GRU only captures the shared components successfully.}
    \label{fig:pairwise_confusion_mlp_gru_5m}
\end{figure}

\begin{figure}[ht]
    \centering

    \begin{subfigure}[t]{\linewidth}
        \centering
        \includegraphics[width=\linewidth]{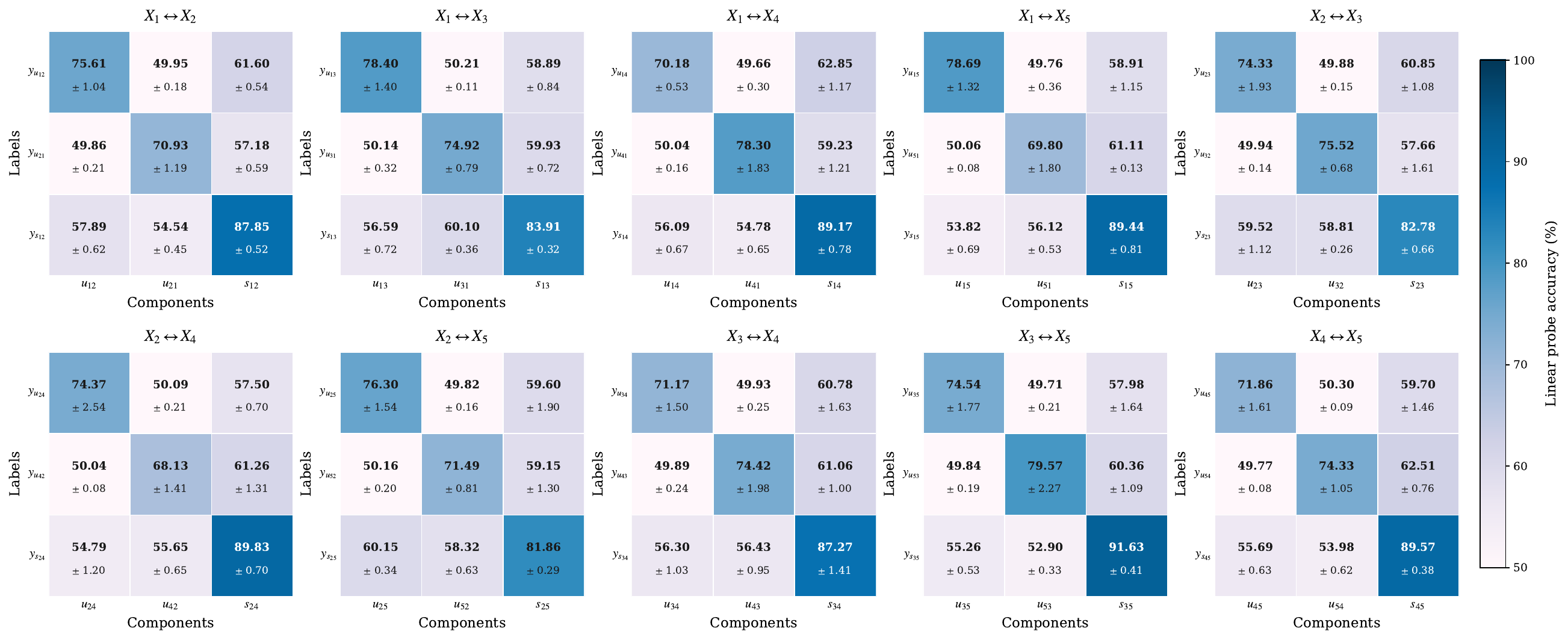}
        \caption{gMLP}
        \label{fig:pairwise_confusion_gmlp_5m}
    \end{subfigure}
    \vspace{0.5em}
    \begin{subfigure}[t]{\linewidth}
        \centering
        \includegraphics[width=\linewidth]{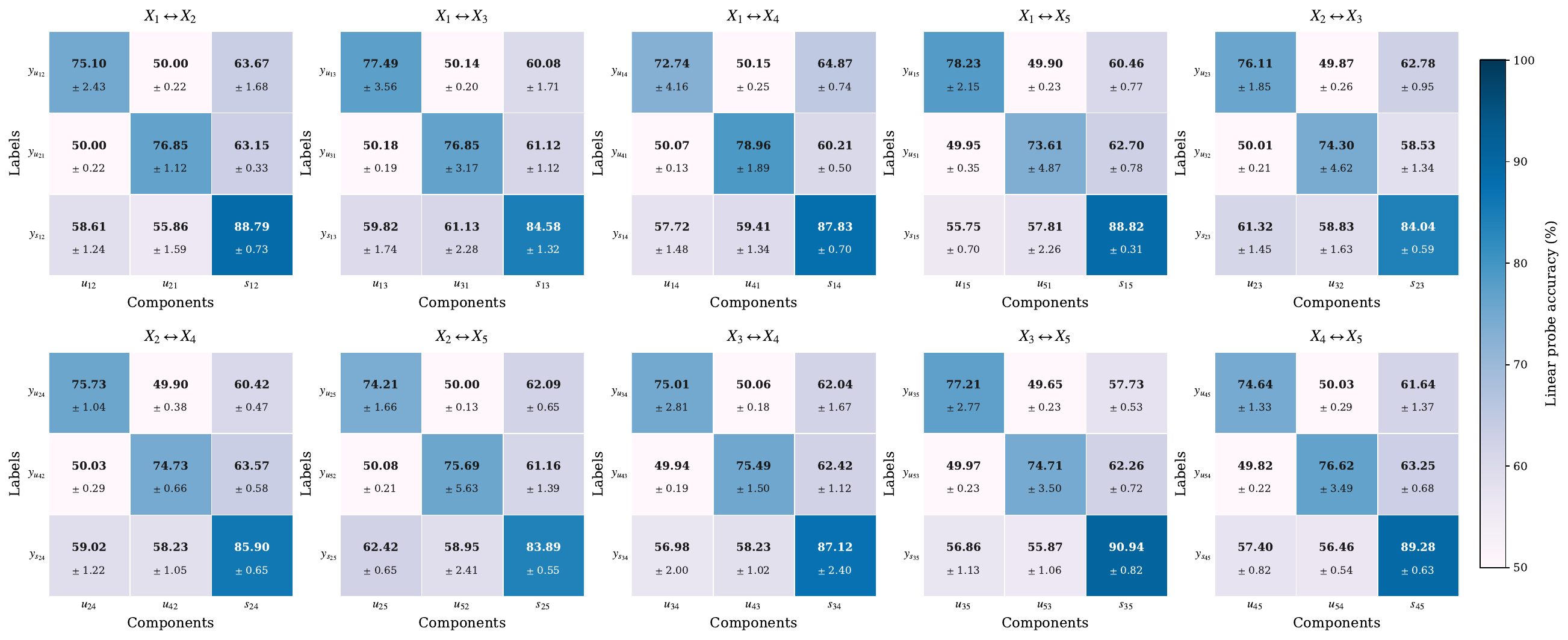}
        \caption{\textbf{RePercENT}}
        \label{fig:pairwise_confusion_repercent_5m}
    \end{subfigure}

    \caption{($\mathbf{M = 5}$) Pairwise confusion matrices for the synthetic setting with five modalities, shown for gMLP, and \textbf{RePercENT}. Despite the increased number of modality pairs, both models successfully encode pairwise unique and shared representations, yielding similar disentanglement performance.}
    \label{fig:pairwise_confusion_gmlp_repercent_5m}
\end{figure}

\clearpage

\paragraph{Hyperparameter $\alpha$} Figure \ref{fig:alpha_sweep} illustrates the impact of the balancing parameter $\alpha$, as we keep the remaining parameters constant, for different numbers of modalities, $M$. Importantly, for this experiment, we set $\lambda = 1$ and $\beta = 1$. We evaluate the disentanglement performance of the unique and shared components separately, according to the quantities:
\begin{equation}
    \Delta_s = \mathbb{E}\left[ Acc_{\hat{\mathbf{s}}_{ij}\to y_{s_{ij}}} \right]- \mathbb{E}\left[Acc_{\hat{\mathbf{s}}_{ij}\to y_{u_{ij}}}\right], \qquad \Delta_u = \mathbb{E}\left[ Acc_{\hat{\mathbf{u}}_{ij}\to y_{u_{ij}}} \right]- \mathbb{E}\left[Acc_{\hat{\mathbf{u}}_{ij}\to y_{s_{ij}}}\right].
    \label{eq: delta-s_delta-u}
\end{equation}
Notice that these $\Delta$ metrics, Eq. \eqref{eq: delta-s_delta-u}, measure the margin between correct intra-component prediction and inter-component leakage. A higher value indicates that the latent factors are not only task-relevant but also successfully isolated from one another. Observing both the unique and shared component performance, there is a similar emerging pattern. Concretely, for lower values of $\alpha$ the representation quality is quite low. The performance gradually increases as $\alpha$ increases, with the best behavior demonstrated within the range $[1, 10]$. 

This empirical trend aligns precisely with the requirements of Theorem \eqref{theo: the_01}, which states that $\alpha$ \textbf{must be greater than or equal to} $\lambda$ for stable disentanglement. This suggests that during joint optimization, prioritizing the recovery of shared components is more critical than over-penalizing the overlap between shared and unique subspaces. However, extreme values of $\alpha$ yield a negative effect, indicating that excessive imbalance between objectives eventually degrades both the discriminative power of the unique representations and the purity of the shared components.

\begin{figure}[ht]
    \centering

    \begin{subfigure}[t]{0.48\linewidth}
        \centering
        \includegraphics[width=\linewidth]{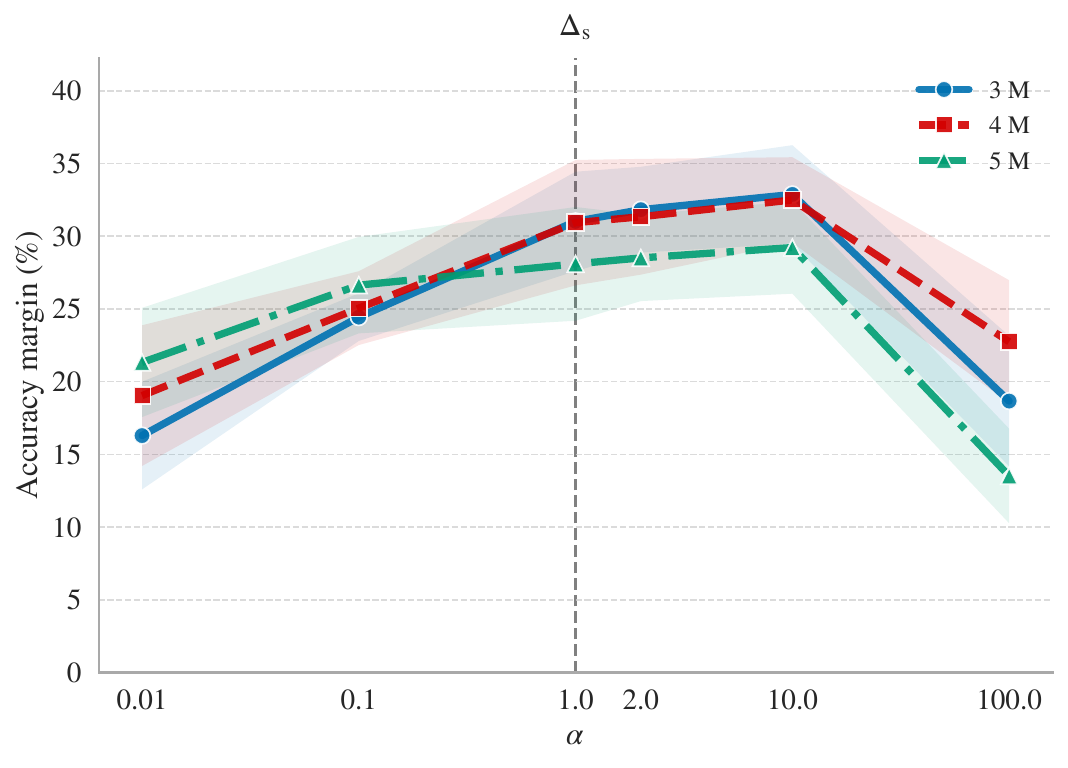}
        \caption{$\Delta_s$}
    \end{subfigure}
    \hfill
    \begin{subfigure}[t]{0.48\linewidth}
        \centering
        \includegraphics[width=\linewidth]{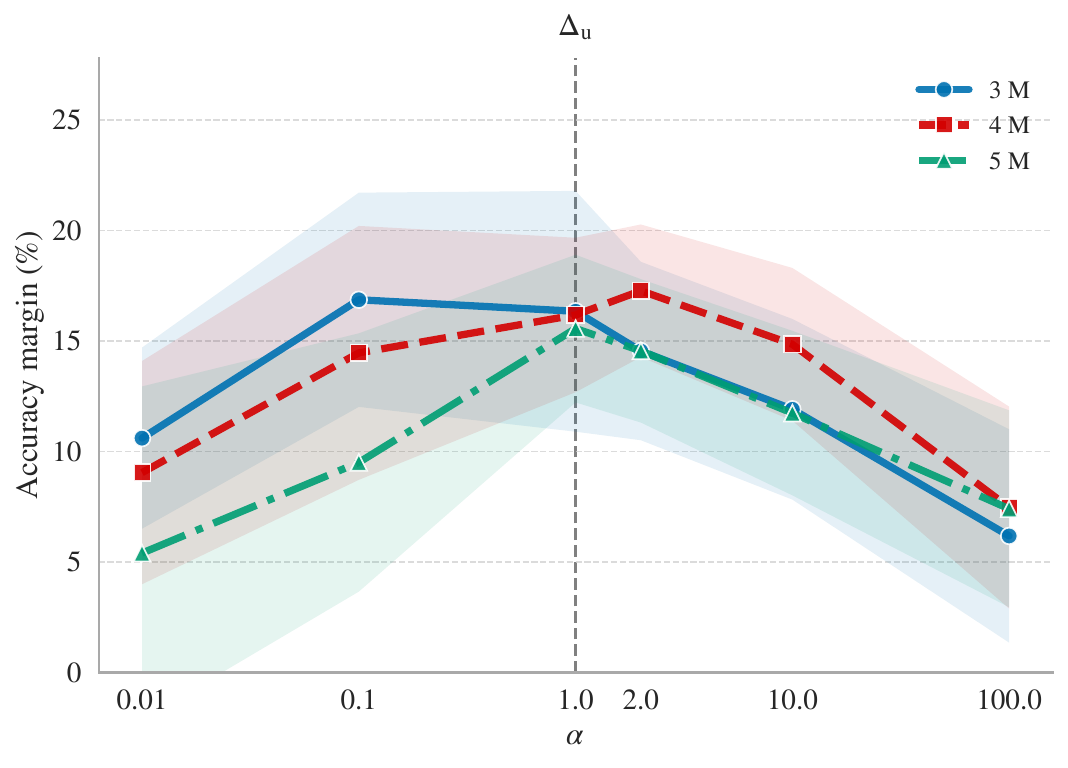}
        \caption{$\Delta_u$}
    \end{subfigure}

    \caption{Sweep of parameter $\mathbf{\alpha}$ across different $M$, when $\lambda$ is fixed to 1.}
    \label{fig:alpha_sweep}
\end{figure}

\clearpage
\subsection{MultiBench}
\label{additional_results_multibench}

In this section, we provide additional results on the MultiBench datasets. Table~\ref{tab:linear_probe_multibench_full} reports the full per-component breakdown. RePercENT matches or outperforms the baselines, and all disentanglement methods agree on the relative utility of each component: $U_T$ is the most predictive component for sentiment analysis (MOSI and MOSEI), followed by the shared component, while this order is reversed for humor detection (UR-FUNNY). Table~\ref{tab:linear_probe_multibench_3m} extends these results to the three-modality setting, where \texttt{Audio} is added and only RePercENT is directly applicable. Notably, even at $M=3$, RePercENT requires fewer parameters than any other baseline at $M=2$. In terms of downstream utility, \texttt{Text} clearly outperforms \texttt{Audio} in the \texttt{Text}--\texttt{Audio} pair, while in the \texttt{Audio}--\texttt{Video} pair the two modalities perform on par. An intuitive explanation is that both audio and video mainly carry paralinguistic cues, such as prosody and facial expressions, which are weaker indicators of sentiment and humor than the spoken content itself.

\begin{table}[h!]
    \centering
    \scriptsize
    \setlength{\tabcolsep}{9pt}
    \renewcommand{\arraystretch}{1}
    \caption{
    Linear probe accuracy (\%) across MOSI, MOSEI, and UR-FUNNY.
    We report (mean $\pm$ std.) over three seeds.
    For each dataset, the best result among the reported rows is shown in bold.
    }
    \label{tab:linear_probe_multibench_full}
    \begin{adjustbox}{max width=\linewidth, max totalheight=\textheight}
    \begin{tabular}{llccccc}
    \toprule
    \textbf{Model} & \textbf{Component} & \textbf{MOSI} ($\uparrow$) & \textbf{MOSEI} ($\uparrow$) & \textbf{UR-FUNNY} ($\uparrow$) & \textbf{Params} ($\downarrow$) & \textbf{FLOPs} ($\downarrow$)\\
    \midrule
    FactorCL
    & --  & 52.87 $\pm$ 0.96 & 75.65 $\pm$ 2.79 & 58.76 $\pm$ 1.87 & {21.98M} & {3.789B}\\
    \addlinespace[4pt]
    \multirow{4}{*}{FOCAL}
    & $S$   & 60.93 $\pm$ 0.36 & 77.20 $\pm$ 0.21 & 62.67 $\pm$ 0.94 & \multirow{4}{*}{17.68M} & \multirow{4}{*}{1.674B}\\
    & $U_V$ & 51.46 $\pm$ 2.18 & 69.59 $\pm$ 0.59 & 57.21 $\pm$ 0.83 \\
    & $U_T$ & 60.59 $\pm$ 2.73 & 77.83 $\pm$ 0.63 & 60.87 $\pm$ 0.89 \\
    \addlinespace[4pt]
    \multirow{4}{*}{JointOpt}
    & $S$   & 60.11 $\pm$ 0.96 & 77.40 $\pm$ 0.35 & 62.04 $\pm$ 0.47 & \multirow{4}{*}{17.68M} & \multirow{4}{*}{1.675B}\\
    & $U_V$ & 49.51 $\pm$ 2.79 & 69.35 $\pm$ 0.42 & 56.58 $\pm$ 1.74 \\
    & $U_T$ & 61.42 $\pm$ 2.56 & 77.33 $\pm$ 0.59 & 60.33 $\pm$ 1.16 \\
    \addlinespace[4pt]
    \multirow{4}{*}{DisentangledSSL}
    & $S$   & 64.48 $\pm$ 1.13 & 74.18 $\pm$ 1.06 & 61.56 $\pm$ 0.88 & \multirow{4}{*}{17.68M} & \multirow{4}{*}{0.835B}\\
    & $U_V$ & 50.83 $\pm$ 1.63 & 70.90 $\pm$ 0.14 & 57.78 $\pm$ 0.52 \\
    & $U_t$ & \textbf{65.35 $\pm$ 1.48} & 74.33 $\pm$ 0.59 & 61.03 $\pm$ 0.39 \\

    \addlinespace[4pt]
    \rowcolor{mutedblue}
    & $S$   & 59.96 $\pm$ 1.10 & 76.65 $\pm$ 0.10 & \textbf{62.95 $\pm$ 0.87} & & \\
    \rowcolor{mutedblue}
    & $U_V$ & 53.30 $\pm$ 0.69 & 68.96 $\pm$ 0.68 & 57.25 $\pm$ 0.77 & & \\
    \rowcolor{mutedblue}
    \multirow{-3}{*}{\textbf{RePercENT}}
    & $U_T$ & 62.59 $\pm$ 0.83 & \textbf{77.99 $\pm$ 0.32} & 62.41 $\pm$ 2.08
    & \multirow{-3}{*}{\textbf{7.34M}} & \multirow{-3}{*}{\textbf{0.052B}} \\
    
    \bottomrule
    \end{tabular}
    \end{adjustbox}
    \end{table}

\begin{table}[ht]
\centering
\small
\setlength{\tabcolsep}{3.5pt}
\renewcommand{\arraystretch}{1.10}
\caption{Linear probe accuracy across MOSI, MOSEI, and UR-FUNNY.
We report (mean $\pm$ std.) over three seeds for the full 
\texttt{Video-Audio-Text} setting.
}
\label{tab:linear_probe_multibench_3m}

\begin{adjustbox}{max width=\linewidth}
\begin{tabular}{llccccc}
\toprule
\textbf{Model}
& \textbf{Component}
& \textbf{MOSI} $\uparrow$
& \textbf{MOSEI} $\uparrow$
& \textbf{UR-FUNNY} $\uparrow$
& \textbf{Params} $\downarrow$
& \textbf{FLOPs} $\downarrow$ \\
\midrule
\multirow{3}{*}{\makecell{RePercENT \\ \texttt{(Video--Audio)}}}
& $S$
& $54.95 \pm 2.18$
& $69.04 \pm 0.35$
& $58.10 \pm 1.55$ & \multirow{6}{*}{\textbf{15.29M}}
& \multirow{6}{*}{\textbf{0.94B}} \\

& $U_{V}$
& $53.55 \pm 0.75$
& $69.10 \pm 0.46$
&  $56.33 \pm 1.57$
& & \\

& $U_{A}$
& $53.79 \pm 1.01$
& $70.76 \pm 0.21$
& $55.92 \pm 1.37$
& & \\
\cmidrule(lr){1-5}
\multirow{3}{*}{\makecell{RePercENT \\ \texttt{(Text--Audio)}}}
& $S$
& $61.61 \pm 2.27$
& $76.69 \pm 0.53$
& $60.40 \pm 1.02$\\

& $U_{T}$
& $62.30 \pm 1.49$
& $77.21 \pm 0.56$
& $62.85 \pm 1.05$
& & \\

& $U_{A}$
& $55.88 \pm 1.85$
& $70.93 \pm 0.28$
& $55.23 \pm 1.67$
& & \\
\bottomrule
\end{tabular}
\end{adjustbox}
\end{table}

\paragraph{Additional results on disentanglement metrics} Tables~\ref{tab:alignment}--\ref{tab:recon_gain_linear} report the disentanglement metrics on the remaining MultiBench datasets. Consistent with the MOSEI results, RePercENT achieves the highest alignment between the shared components and the largest reconstruction gain, indicating that its shared and modality-specific representations are complementary. In terms of cross-modal leakage on UR-FUNNY, no model's $U_{T,V}$ carries measurable information about the \texttt{Video} modality, while for $U_{V,T}$, RePercENT exhibits the smallest leakage. Together, these results showcase RePercENT's state-of-the-art disentanglement quality and consistency across datasets.

\begin{table}[h!]
\centering
\small
\setlength{\tabcolsep}{5pt}
\caption{Cross-modal alignment between the shared components $S_{V, T}$ and $S_{T, V}$,
measured by CKA and mKNN; higher is better
for both. We report (mean$\pm$ std.) over three seeds on MOSI and UR-FUNNY.}
\label{tab:alignment}
\begin{tabular}{lcccc}
\toprule
& \multicolumn{2}{c}{MOSI} &  \multicolumn{2}{c}{UR-FUNNY} \\
\cmidrule(lr){2-3}\cmidrule(lr){4-5}
Model & CKA & mKNN & CKA & mKNN\\
\midrule
FOCAL & $0.0326_{\pm.0014}$ & $0.0471_{\pm.0023}$ & $0.1130_{\pm.0137}$ & $0.0307_{\pm.0009}$ \\
JointOpt & $0.0362_{\pm.0034}$ & $0.0488_{\pm.0019}$ & $0.0707_{\pm.0061}$ & $0.0287_{\pm.0003}$ \\
DisentangledSSL & $0.0195_{\pm.0055}$ & $0.0443_{\pm.0012}$ & $0.0813_{\pm.0168}$ & $0.0242_{\pm.0005}$ \\
\rowcolor{mutedblue}
RePercENT & $\mathbf{0.7226}_{\pm.0381}$ & $\mathbf{0.1794}_{\pm.0139}$ & $\mathbf{0.8753}_{\pm.0155}$ & $\mathbf{0.3278}_{\pm.0149}$ \\
\bottomrule
\end{tabular}
\end{table}

\begin{table}[h!]
\centering
\small
\setlength{\tabcolsep}{7pt}
\caption{Cross leakage from the unique components: $R^2$ of reconstructing each
modality from the \emph{other} modality's unique representation, under
ridge-regularized linear probes (lower is better). We report (mean $\pm$ std.) over three seeds on UR-FUNNY}
\label{tab:cross_leakage}
\begin{tabular}{lcccc}
\toprule
Model & $U_{T, V}\!\to\! X_{V}$ & $U_{V, T}\!\to\!X_T$ \\
\midrule
FOCAL & $-0.0157_{\pm.0016}$ & $0.0523_{\pm.0264}$ \\
JointOpt & $-0.0166_{\pm.0020}$ & $0.0602_{\pm.0118}$ \\
DisentangledSSL & $0.0040_{\pm.0078}$ & $0.0545_{\pm.0168}$ \\
\rowcolor{mutedblue}
RePercENT & $\mathbf{-0.0036_{\pm.0010}}$ & $\mathbf{0.0395_{\pm.0359}}$ \\
\bottomrule
\end{tabular}
\end{table}

\begin{table}[h!]
\centering
\small
\setlength{\tabcolsep}{6pt}
\caption{Reconstruction gain under ridge-regularized linear probes: reconstruction
from the combined shared+specific components minus the mean of the shared-only and
specific-only reconstructions, per modality (higher is better). We report (mean $\pm$ std.)
over three seeds on UR-FUNNY.}
\label{tab:recon_gain_linear}
\begin{tabular}{lcc}
\toprule
Model & RG-\texttt{Text}& RG-\texttt{Video} \\
\midrule
FOCAL & $0.0067_{\pm.0012}$ & $0.0206_{\pm.0008}$ \\
JointOpt & $0.0068_{\pm.0007}$ & $0.0203_{\pm.0008}$ \\
DisentangledSSL & $0.0045_{\pm.0009}$ & $0.0206_{\pm.0016}$ \\
\rowcolor{mutedblue}
RePercENT & $\mathbf{0.0366_{\pm.0029}}$ & $\mathbf{0.0452_{\pm.0035}}$ \\
\bottomrule
\end{tabular}
\end{table}

\newpage
\subsection{IRFL}
\label{scaling_complexity_irfl}
\paragraph{Scaling complexity} Table \ref{tab:model_complexity_irfl} illustrates the scaling behavior of RePercENT compared to its strongest baseline, i.e., gMLP, as the number of modalities increases from two to three. For the two-modality calculation, we simulate the corresponding setting by excluding the definition encoders. Additionally, the inference latency is measured using a batch size of $32$ and averaged over $100$ forward passes. 

Our proposed framework demonstrates exceptional parameter efficiency. The gMLP requires nearly a $3\times$ increase in parameter count to accommodate the third modality, whereas RePercENT scales by a factor of only $\sim1.5$. This demonstrates that our method maintains a substantially lower computational footprint, both in parameter count and floating point operations (FLOPs), when transitioning to higher-dimensional multimodal settings. These gains are also reflected in the average inference time, where our method requires a $\sim2.2$ increase, while the gMLP rises to $\sim 2.9$ with the presence of a third modality.

\begin{table}[h]
\centering
\small
\caption{Scaling complexity for RePercENT and strongest disentanglement counterpart gMLP, as we go from two to three modalities for the IRFL detection task. Growth denotes the ratio 3-modality/2-modality. The last column reports the relative reduction of RePercENT compared to gMLP at $M=3$.}
\label{tab:model_complexity_irfl}
\begin{tabularx}{\linewidth}{>{\raggedright\arraybackslash}Xcccc}
\toprule
\textbf{Metric $\sim$ Model} 
& $\mathbf{M = 2}$ (reference) 
& $\mathbf{M = 3}$ 
& \textbf{Growth} 
& \textbf{Rel. diff. @ $\mathbf{M=3}$} \\
\midrule

\rowcolor{gray!10} 
\multicolumn{5}{l}{\textit{Parameter count (M)} ($\downarrow$)}\\
gMLP      
& $6.12$ 
& $17.17$ 
& $2.81 \times$ 
& reference \\
\textbf{RePercENT}     
& $6.06$ 
& $\mathbf{8.96}$ 
& $\mathbf{1.48 \times}$ 
& $\mathbf{47.8\%}$ lower \\

\midrule
\rowcolor{gray!10} 
\multicolumn{5}{l}{\textit{Floating-point operations, FLOPs (B)} ($\downarrow$)}\\
gMLP       
& $20.89$ 
& $63.16$ 
& $3.02 \times$ 
& reference \\
\textbf{RePercENT}       
& $5.67$ 
& $\mathbf{9.50}$ 
& $\mathbf{1.68 \times}$ 
& $\mathbf{85.0\%}$ lower \\

\midrule
\rowcolor{gray!10} 
\multicolumn{5}{l}{\textit{Inference latency (ms)} ($\downarrow$)}\\
gMLP   
& $15.32$ 
& $44.12$ 
& $2.88 \times$ 
& reference \\
\textbf{RePercENT}     
& $16.89$ 
& $\mathbf{37.10}$ 
& $\mathbf{2.20 \times}$ 
& $\mathbf{15.9\%}$ lower \\

\bottomrule
\end{tabularx}
\end{table}

\paragraph{Enriching text with definition}
When the definitions are added all the models benefit naturally, with our framework attaining competitive performance. 

 \begin{table}[h]
    \centering
    \caption{IRFL figurative language detection results averaged over 5 random seeds. As a fixed reference, we provide zero-shot CLIP alongside two fine-tuned variants, including projection-only and end-to-end fine-tuning. Top-1 accuracy (\%) is reported (mean $\pm$ std).}
    \label{tab:irfl_results_def}
    \scriptsize
    \setlength{\tabcolsep}{4pt}
    \renewcommand{\arraystretch}{0.92}
    \resizebox{\textwidth}{!}{
    \begin{tabular}{@{}lcccc@{}}
        \toprule
        \textbf{Model} & \textbf{Idioms} & \textbf{Simile} & \textbf{Metaphor--\scriptsize{OoD}} & \textbf{Overall} \\
        \midrule
        \multicolumn{5}{@{}l}{\textit{Image vs Caption $\oplus$ Definition}} \\

        CLIP-ViT-B/32 \scriptsize{(zero-shot)}
            & 26.00 & 48.38 & 36.04 & 38.78 \\
        CLIP-ViT-B/32 \scriptsize{(finetune-proj)}
            & $36.00 \pm 1.2$
            & $66.5 \pm 1.6$
            & $41.44 \pm 2.0$
            & $48.67 \pm 1.3$ \\
        CLIP-ViT-B/32 \scriptsize{(finetune-all)}
            & $38.1 \pm 1.1$
            & $70.4 \pm 1.5$
            & $42.16 \pm 3.2$
            & $50.81 \pm 0.7$ \\
        \midrule
        gMLP
            & $43.80 \pm 1.6$
            & $64.55 \pm 2.9$
            & $44.92 \pm 3.0$
            & $51.36 \pm 1.7$ \\
        GRU
            & $46.40 \pm 1.8$
            & $57.40 \pm 2.4$
            & $44.92 \pm 4.1$
            & $49.56 \pm 1.7$ \\
        \rowcolor{mutedblue}
        \textbf{RePercENT}
            & $42.50 \pm 3.4$
            & $65.49 \pm 2.4$
            & $44.98 \pm 2.1$
            & $\mathbf{51.38 \pm 0.8}$ \\

        \bottomrule
    \end{tabular}
    }
    \end{table}

\newpage
\subsection{TCGA}

\paragraph{Additional Baselines} On TCGA, we complement our four-modality results by comparing RePercENT against leading disentanglement approaches: FOCAL, DisentangledSSL, and JointOpt. Since none of these baselines inherently supports training on all four modalities, we restrict this experiment to the \texttt{Molecular} and \texttt{WSI} modalities and retrain RePercENT on this pair. For each baseline, we evaluate two variants: (i) with a separate 5-layer Transformer backbone per modality, matching RePercENT's capacity (B), and (ii) without a backbone, using 3-layer MLP encoders (NB). All methods use a latent dimensionality of 300. Tables~\ref{tab:mol_wsi_component_comparison} and~\ref{tab:wsi_mol_component_comparison} report the per-component performance on cancer type prediction. RePercENT achieves the highest overall accuracy on both pairs and outperforms all baselines on nine out of the ten cancer types for \texttt{WSI--Mol} and on half of them for \texttt{Mol--WSI}. Finally, the two-modality \texttt{Mol--WSI} results closely match the four-modality results in Table~\ref{tab:cancer_type_classification} and any small differences likely stem from optimization variance and minor architectural changes, namely the slightly larger latent dimensionality and four fewer latent slots per RePercENT encoder.

\begin{table*}[h]
\centering
\scriptsize
\setlength{\tabcolsep}{2.7pt}
\renewcommand{\arraystretch}{1.08}
\caption{Linear-probe accuracy (\%) of the Mol--WSI representations across TCGA cohorts (mean $\pm$ std). The best result in each column is shown in \textbf{bold}. B and NB denote models with and without a backbone, respectively. {\small{$^\dagger$GBM is saturated, with most methods reaching 100\%, so no result is highlighted.}}}
\label{tab:mol_wsi_component_comparison}
\resizebox{\textwidth}{!}{
\begin{tabular}{llccccccccccc}
\toprule
\textbf{Model}
& \textbf{Comp.}
& \textbf{Overall}
& \textbf{BRCA}
& \textbf{COAD}
& \textbf{GBM}$^\dagger$
& \textbf{HNSC}
& \textbf{KIRC}
& \textbf{LGG}
& \textbf{LUAD}
& \textbf{LUSC}
& \textbf{OV}
& \textbf{PRAD} \\
\midrule

\rowcolor{mutedblue}
& $D_{\mathrm{mol},\mathrm{wsi}}$
& $\mathbf{ {65.68 \pm 1.0}}$
& ${68.37 \pm 4.4}$
& $\mathbf{{27.36 \pm 3.2}}$
& ${98.14 \pm 0.4}$
& ${62.31 \pm 2.0}$
& $\mathbf{{58.32 \pm 2.9}}$
& $\mathbf{{44.71 \pm 1.5}}$
& ${46.00 \pm 2.2}$
& ${62.65 \pm 2.8}$
& $81.38 \pm 1.8$
& $\mathbf{{92.20 \pm 0.5}}$ \\
\rowcolor{mutedblue}
& $S_{{\mathrm{mol},\mathrm{wsi}}}$
& $61.24 \pm 0.6$
& $68.19 \pm 4.1$
& $19.77 \pm 2.9$
& $97.63 \pm 0.7$
& $43.46 \pm 4.2$
& $53.83 \pm 2.7$
& $37.65 \pm 3.7$
& $38.60 \pm 3.7$
& $54.29 \pm 1.5$
& $\mathbf{{85.17 \pm 0.7}}$
& $91.60 \pm 0.6$ \\
\rowcolor{mutedblue}
\multirow{-3}{*}{RePercENT}
& $U_{{\mathrm{mol},\mathrm{wsi}}}$
& $63.52 \pm 0.7$
& $67.16 \pm 4.1$
& $25.06 \pm 4.1$
& $97.97 \pm 0.5$
& $53.08 \pm 3.6$
& $56.82 \pm 2.0$
& $40.59 \pm 2.0$
& $44.00 \pm 3.2$
& $60.82 \pm 2.6$
& $81.03 \pm 1.7$
& $91.80 \pm 0.8$ \\

\addlinespace[4pt]
\multirow{3}{*}{DisentangledSSL (B)}
& $D_{{\mathrm{mol},\mathrm{wsi}}}$
& $ {46.50 \pm 3.5}$
& $64.65 \pm 4.3$
& ${2.99 \pm 5.5}$
& $ 100.00 \pm 0.0$
& $ {12.31 \pm 6.7}$
& $ {35.33 \pm 8.2}$
& $ {3.14 \pm 1.8}$
& $ {32.60 \pm 4.0}$
& $36.33 \pm 21.3$
& $ {66.55 \pm 7.3}$
& $ {74.60 \pm 10.0}$ \\
& $S_{{\mathrm{mol},\mathrm{wsi}}}$
& $42.98 \pm 3.2$
& $ {70.98 \pm 6.0}$
& $2.30 \pm 5.1$
& $ 100.00 \pm 0.0$
& $8.85 \pm 7.0$
& $33.46 \pm 8.6$
& $0.78 \pm 1.3$
& $31.00 \pm 3.8$
& $34.29 \pm 19.9$
& $61.03 \pm 14.9$
& $39.20 \pm 23.8$ \\
& $U_{{\mathrm{mol},\mathrm{wsi}}}$
& $44.93 \pm 3.3$
& $70.14 \pm 7.6$
& $2.76 \pm 5.6$
& $ 100.00 \pm 0.0$
& $10.96 \pm 6.1$
& $34.21 \pm 8.2$
& $1.57 \pm 1.1$
& $31.40 \pm 4.5$
& $ {36.53 \pm 21.0}$
& $64.48 \pm 9.7$
& $52.60 \pm 19.1$ \\

\addlinespace[4pt]
\multirow{3}{*}{DisentangledSSL (NB)}
& $D_{{\mathrm{mol},\mathrm{wsi}}}$
& $ {50.78 \pm 0.2}$
& $61.02 \pm 0.4$
& $5.75 \pm 1.2$
& $ 100.00 \pm 0.0$
& $ {35.19 \pm 1.6}$
& $ {25.05 \pm 0.4}$
& $ {8.82 \pm 3.0}$
& $ {52.20 \pm 1.1}$
& $ {55.51 \pm 1.7}$
& $69.83 \pm 0.0$
& $68.20 \pm 0.8$ \\
& $S_{{\mathrm{mol},\mathrm{wsi}}}$
& $48.61 \pm 0.2$
& $ {64.28 \pm 0.8}$
& $ {8.28 \pm 1.5}$
& $ 100.00 \pm 0.0$
& $23.85 \pm 2.8$
& $23.55 \pm 1.0$
& $1.57 \pm 2.6$
& $48.40 \pm 2.1$
& $49.59 \pm 2.2$
& $ {71.03 \pm 0.8}$
& $63.20 \pm 1.3$ \\
& $U_{\mathrm{mol},\mathrm{wsi}}$
& $49.29 \pm 0.3$
& $62.23 \pm 0.2$
& $7.59 \pm 0.6$
& $ 100.00 \pm 0.0$
& $29.81 \pm 2.5$
& $24.49 \pm 0.4$
& $3.14 \pm 1.5$
& $49.00 \pm 3.4$
& $48.16 \pm 2.7$
& $70.52 \pm 0.4$
& $ {68.60 \pm 0.9}$ \\

\addlinespace[4pt]
\multirow{3}{*}{FOCAL (B)}
& $D_{\mathrm{mol},\mathrm{wsi}}$
& $ {47.52 \pm 1.6}$
& $60.56 \pm 1.7$
& $ {5.29 \pm 5.8}$
& $98.31 \pm 0.0$
& $ {22.88 \pm 2.9}$
& $ {40.19 \pm 10.7}$
& $ {2.94 \pm 2.0}$
& $ {26.40 \pm 6.1}$
& $52.24 \pm 6.2$
& $ {57.76 \pm 5.5}$
& $ {79.80 \pm 2.6}$ \\
& $S_{\mathrm{mol},\mathrm{wsi}}$
& $46.63 \pm 1.4$
& $60.65 \pm 1.7$
& $4.14 \pm 5.2$
& $ {99.15 \pm 0.0}$
& $22.69 \pm 3.2$
& $39.07 \pm 12.1$
& $0.39 \pm 0.9$
& $24.40 \pm 4.7$
& $ {53.27 \pm 5.7}$
& $53.62 \pm 1.9$
& $79.20 \pm 2.6$ \\
& $U_{\mathrm{mol},\mathrm{wsi}}$
& $47.10 \pm 1.6$
& $ {60.84 \pm 2.2}$
& $4.83 \pm 5.7$
& $98.47 \pm 0.4$
& $ {22.88 \pm 2.7}$
& $39.81 \pm 11.1$
& $1.76 \pm 1.3$
& $25.20 \pm 5.4$
& $51.63 \pm 6.4$
& $56.90 \pm 5.0$
& $79.00 \pm 2.0$ \\

\addlinespace[4pt]
\multirow{3}{*}{FOCAL (NB)}
& $D_{\mathrm{mol},\mathrm{wsi}}$
& $ {53.36 \pm 0.5}$
& $59.72 \pm 1.5$
& $ {13.56 \pm 2.1}$
& $ 100.00 \pm 0.0$
& $43.85 \pm 2.2$
& $25.79 \pm 0.5$
& $ {10.98 \pm 0.8}$
& $\mathbf{ {53.40 \pm 1.1}}$
& $54.90 \pm 0.9$
& $ {70.34 \pm 0.5}$
& $ {80.60 \pm 0.6}$ \\
& $S_{\mathrm{mol},\mathrm{wsi}}$
& $48.95 \pm 0.3$
& $62.14 \pm 0.7$
& $6.21 \pm 1.0$
& $ 100.00 \pm 0.0$
& $25.96 \pm 4.5$
& $24.49 \pm 0.4$
& $2.55 \pm 1.6$
& $48.40 \pm 3.0$
& $53.06 \pm 1.3$
& $69.83 \pm 0.9$
& $67.20 \pm 1.1$ \\
& $U_{\mathrm{mol},\mathrm{wsi}}$
& $48.74 \pm 0.7$
& $\mathbf{ {75.16 \pm 1.3}}$
& $0.46 \pm 0.6$
& $99.83 \pm 0.4$
& $\mathbf{ {76.15 \pm 2.8}}$
& $ {26.92 \pm 0.4}$
& $2.16 \pm 2.1$
& $42.60 \pm 4.0$
& $\mathbf{ {68.98 \pm 6.4}}$
& $41.03 \pm 3.4$
& $11.20 \pm 4.6$ \\

\addlinespace[4pt]
\multirow{3}{*}{JointOpt (B)}
& $D_{\mathrm{mol},\mathrm{wsi}}$
& $ {54.73 \pm 1.5}$
& $61.40 \pm 4.1$
& $ {9.20 \pm 3.5}$
& $ 100.00 \pm 0.0$
& $ {36.73 \pm 3.2}$
& $ {47.29 \pm 1.7}$
& $ {17.25 \pm 5.7}$
& $ {39.00 \pm 1.6}$
& $ {56.73 \pm 1.7}$
& $ {66.90 \pm 2.9}$
& $ {91.20 \pm 2.6}$ \\
& $S_{\mathrm{mol},\mathrm{wsi}}$
& $52.48 \pm 1.2$
& $62.23 \pm 4.1$
& $6.67 \pm 3.0$
& $ 100.00 \pm 0.0$
& $30.96 \pm 4.0$
& $ {47.29 \pm 1.6}$
& $9.02 \pm 4.1$
& $35.60 \pm 2.2$
& $55.31 \pm 1.7$
& $64.48 \pm 8.2$
& $87.80 \pm 4.0$ \\
& $U_{\mathrm{mol},\mathrm{wsi}}$
& $51.67 \pm 1.4$
& $ {62.88 \pm 3.5}$
& $6.44 \pm 2.8$
& $ 100.00 \pm 0.0$
& $25.38 \pm 7.6$
& $46.54 \pm 2.8$
& $6.86 \pm 2.9$
& $32.60 \pm 3.4$
& $55.31 \pm 1.5$
& $63.45 \pm 6.3$
& $90.20 \pm 2.4$ \\

\addlinespace[4pt]
\multirow{3}{*}{JointOpt (NB)}
& $D_{\mathrm{mol},\mathrm{wsi}}$
& $ {51.00 \pm 0.6}$
& $60.74 \pm 0.3$
& $5.75 \pm 1.4$
& $ 100.00 \pm 0.0$
& $ {36.73 \pm 2.2}$
& $ {24.86 \pm 0.5}$
& $ {10.00 \pm 1.9}$
& $ {49.80 \pm 1.8}$
& $ {56.33 \pm 1.1}$
& $69.83 \pm 0.6$
& $ {70.40 \pm 1.5}$ \\
& $S_{\mathrm{mol},\mathrm{wsi}}$
& $48.33 \pm 1.4$
& $ {66.60 \pm 7.0}$
& $6.90 \pm 0.8$
& $ 100.00 \pm 0.0$
& $25.58 \pm 2.1$
& $23.18 \pm 1.0$
& $2.55 \pm 2.3$
& $46.00 \pm 1.6$
& $51.43 \pm 3.4$
& $70.69 \pm 0.6$
& $54.80 \pm 25.6$ \\
& $U_{\mathrm{mol},\mathrm{wsi}}$
& $48.23 \pm 0.6$
& $64.65 \pm 1.7$
& $ {7.59 \pm 3.1}$
& $ 100.00 \pm 0.0$
& $25.58 \pm 4.1$
& $23.18 \pm 1.7$
& $1.37 \pm 1.9$
& $41.40 \pm 3.2$
& $48.57 \pm 1.9$
& $ {70.86 \pm 1.1}$
& $65.60 \pm 5.0$ \\

\bottomrule
\end{tabular}
}
\end{table*}

\begin{table*}[h]
\centering
\scriptsize
\setlength{\tabcolsep}{2.7pt}
\renewcommand{\arraystretch}{1.08}
\caption{Linear-probe accuracy (\%) of the WSI--Mol representations across TCGA cohorts (mean $\pm$ std). The best result in each column is shown in \textbf{bold}. B and NB denote models with and without a backbone, respectively. }
\label{tab:wsi_mol_component_comparison}
\resizebox{\textwidth}{!}{
\begin{tabular}{llccccccccccc}
\toprule
\textbf{Model}
& \textbf{Comp.}
& \textbf{Overall}
& \textbf{BRCA}
& \textbf{COAD}
& \textbf{GBM}
& \textbf{HNSC}
& \textbf{KIRC}
& \textbf{LGG}
& \textbf{LUAD}
& \textbf{LUSC}
& \textbf{OV}
& \textbf{PRAD} \\
\midrule
\rowcolor{mutedblue}
& $D_{\mathrm{wsi},\mathrm{mol}}$
& $41.62 \pm 1.0$
& $76.28\pm{2.5}$
& $32.64\pm{1.0}$
& $36.44\pm{4.0}$
& $23.46\pm{2.5}$
& $60.37\pm{2.1}$
& $\mathbf{25.88\pm{2.3}}$
& $24.20\pm{1.5}$
& $22.45\pm{1.4}$
& $\mathbf{42.24\pm{3.6}}$
& $\mathbf{31.40\pm{1.1}}$ \\
\rowcolor{mutedblue}
& $S_{\mathrm{wsi},\mathrm{mol}}$
& $\mathbf{ {42.07 \pm 0.4}}$
& $80.19\pm{0.8}$
& $\mathbf{33.56\pm{3.8}}$
& $\mathbf{40.51\pm{1.9}}$
& $\mathbf{24.62\pm{2.4}}$
& $\mathbf{61.31\pm{1.1}}$
& $20.20\pm{2.0}$
& $\mathbf{24.40\pm{1.3}}$
& $21.63\pm{2.2}$
& $39.83\pm{1.7}$
& $29.60\pm{2.5}$ \\
\rowcolor{mutedblue}
\multirow{-3}{*}{RePercENT}
& $U_{\mathrm{wsi},\mathrm{mol}}$
& $40.49 \pm 0.6$
& $75.44\pm{2.4}$
& $29.89\pm{3.0}$
& $37.29\pm{0.9}$
& $20.38\pm{2.2}$
& $60.00\pm{2.3}$
& $24.31\pm{2.1}$
& $22.20\pm{1.8}$
& $\mathbf{23.27\pm{2.2}}$
& $40.34\pm{2.6}$
& $30.20\pm{2.3}$ \\

\addlinespace[4pt]
\multirow{3}{*}{DisentangledSSL (B)}
& $D_{\mathrm{wsi},\mathrm{mol}}$
& $ {30.32 \pm 3.1}$
& $93.58 \pm 1.2$
& $ {16.09 \pm 7.6}$
& $ {31.36 \pm 5.3}$
& $ {7.50 \pm 4.7}$
& $ {39.81 \pm 13.8}$
& $ {2.35 \pm 2.1}$
& $ {16.80 \pm 6.3}$
& $ {6.53 \pm 4.5}$
& $ {16.55 \pm 3.7}$
& $ {0.40 \pm 0.5}$ \\

& $S_{\mathrm{wsi},\mathrm{mol}}$
& $29.24 \pm 3.3$
& $93.77 \pm 1.3$
& $15.63 \pm 9.0$
& $28.31 \pm 5.7$
& $5.77 \pm 3.9$
& $38.88 \pm 17.6$
& $1.18 \pm 1.3$
& $14.20 \pm 6.6$
& $5.51 \pm 4.3$
& $15.86 \pm 5.8$
& $0.00 \pm 0.0$ \\

& $U_{\mathrm{wsi},\mathrm{mol}}$
& $29.59 \pm 3.0$
& $ {93.95 \pm 1.4}$
& $15.86 \pm 9.2$
& $28.81 \pm 6.7$
& $6.54 \pm 4.4$
& $38.69 \pm 14.1$
& $1.18 \pm 1.6$
& $15.80 \pm 6.3$
& $5.31 \pm 3.9$
& $16.38 \pm 6.4$
& $0.20 \pm 0.4$ \\

\addlinespace[4pt]
\multirow{3}{*}{DisentangledSSL (NB)}
& $D_{\mathrm{wsi},\mathrm{mol}}$
& $ {38.74 \pm 0.7}$
& $ {95.07 \pm 1.0}$
& $ {28.74 \pm 1.6}$
& $ {39.83 \pm 1.7}$
& $ {15.77 \pm 4.1}$
& $ {48.41 \pm 1.2}$
& $ {10.00 \pm 2.3}$
& $ {23.60 \pm 2.3}$
& $ {11.22 \pm 2.0}$
& $ {32.76 \pm 1.1}$
& $ {17.00 \pm 1.9}$ \\

& $S_{\mathrm{wsi},\mathrm{mol}}$
& $35.69 \pm 0.2$
& $ {95.07 \pm 0.2}$
& $23.68 \pm 1.3$
& $37.29 \pm 2.2$
& $11.15 \pm 2.3$
& $45.79 \pm 1.3$
& $4.12 \pm 0.8$
& $22.00 \pm 1.0$
& $9.59 \pm 1.7$
& $30.52 \pm 2.8$
& $8.80 \pm 1.6$ \\

& $U_{\mathrm{wsi},\mathrm{mol}}$
& $35.78 \pm 0.3$
& $94.79 \pm 0.6$
& $24.14 \pm 2.7$
& $39.49 \pm 1.8$
& $10.96 \pm 2.2$
& $45.61 \pm 1.2$
& $4.90 \pm 1.2$
& $22.40 \pm 2.3$
& $8.98 \pm 2.2$
& $28.10 \pm 1.6$
& $10.00 \pm 1.6$ \\

\addlinespace[4pt]
\multirow{3}{*}{FOCAL (B)}
& $D_{\mathrm{wsi},\mathrm{mol}}$
& $ {28.98 \pm 0.4}$
& $89.12 \pm 2.1$
& $ {16.78 \pm 4.8}$
& $25.42 \pm 1.2$
& $ {6.15 \pm 4.7}$
& $ {37.57 \pm 1.8}$
& $ {2.94 \pm 2.5}$
& $ {18.40 \pm 2.8}$
& $8.16 \pm 2.0$
& $ {13.10 \pm 4.9}$
& $ {5.00 \pm 2.1}$ \\

& $S_{\mathrm{wsi},\mathrm{mol}}$
& $28.35 \pm 0.7$
& $ {89.40 \pm 1.4}$
& $16.55 \pm 3.0$
& $ {25.93 \pm 2.7}$
& $5.96 \pm 4.8$
& $34.21 \pm 2.4$
& $1.57 \pm 2.3$
& $17.20 \pm 4.9$
& $ {8.57 \pm 0.9}$
& $11.21 \pm 3.3$
& $ {5.00 \pm 1.6}$ \\

& $U_{\mathrm{wsi},\mathrm{mol}}$
& $27.81 \pm 0.7$
& $88.84 \pm 2.2$
& $15.17 \pm 3.0$
& $24.07 \pm 4.3$
& $5.77 \pm 5.6$
& $35.33 \pm 3.2$
& $1.96 \pm 1.8$
& $ {18.40 \pm 2.2}$
& $6.73 \pm 2.0$
& $9.31 \pm 4.2$
& $4.80 \pm 2.1$ \\

\addlinespace[4pt]
\multirow{3}{*}{FOCAL (NB)}
& $D_{\mathrm{wsi},\mathrm{mol}}$
& $ {39.67 \pm 0.4}$
& $94.33 \pm 0.9$
& $ {31.49 \pm 1.7}$
& $ {37.80 \pm 1.0}$
& $ {22.69 \pm 1.8}$
& $ {47.48 \pm 1.5}$
& $ {11.37 \pm 2.6}$
& $19.80 \pm 3.5$
& $ {13.88 \pm 1.4}$
& $ {33.97 \pm 1.2}$
& $ {21.40 \pm 1.1}$ \\

& $S_{\mathrm{wsi},\mathrm{mol}}$
& $36.09 \pm 0.3$
& $ {94.98 \pm 0.4}$
& $27.36 \pm 1.3$
& $37.12 \pm 0.9$
& $13.46 \pm 1.9$
& $45.42 \pm 1.9$
& $5.69 \pm 2.1$
& $20.40 \pm 1.8$
& $7.96 \pm 1.1$
& $29.31 \pm 1.6$
& $11.60 \pm 1.5$ \\

& $U_{\mathrm{wsi},\mathrm{mol}}$
& $37.31 \pm 0.2$
& $ {94.98 \pm 0.5}$
& $24.60 \pm 1.0$
& $36.78 \pm 1.9$
& $19.42 \pm 2.8$
& $43.18 \pm 2.4$
& $8.43 \pm 0.5$
& $ {21.80 \pm 2.2}$
& $11.84 \pm 2.1$
& $30.69 \pm 2.9$
& $15.00 \pm 2.2$ \\

\addlinespace[4pt]
\multirow{3}{*}{JointOpt (B)}
& $D_{\mathrm{wsi},\mathrm{mol}}$
& $ {32.29 \pm 1.2}$
& $91.72 \pm 0.8$
& $ {18.16 \pm 1.7}$
& $ {33.90 \pm 1.6}$
& $ {16.35 \pm 2.0}$
& $ {37.94 \pm 3.5}$
& $ {3.73 \pm 2.1}$
& $19.80 \pm 1.3$
& $ {9.80 \pm 3.3}$
& $18.10 \pm 1.8$
& $ {5.60 \pm 5.9}$ \\

& $S_{\mathrm{wsi},\mathrm{mol}}$
& $30.95 \pm 1.6$
& $91.63 \pm 0.9$
& $16.78 \pm 1.7$
& $32.37 \pm 4.0$
& $13.85 \pm 3.4$
& $34.58 \pm 6.1$
& $2.35 \pm 2.0$
& $ {21.00 \pm 2.2}$
& $9.18 \pm 3.0$
& $16.03 \pm 1.4$
& $2.80 \pm 4.7$ \\

& $U_{\mathrm{wsi},\mathrm{mol}}$
& $30.97 \pm 0.8$
& $ {92.47 \pm 0.8}$
& $16.32 \pm 2.1$
& $31.36 \pm 2.9$
& $12.12 \pm 1.9$
& $36.07 \pm 4.3$
& $1.76 \pm 0.4$
& $19.80 \pm 3.1$
& $9.59 \pm 3.0$
& $ {18.45 \pm 1.6}$
& $1.60 \pm 1.5$ \\

\addlinespace[4pt]
\multirow{3}{*}{JointOpt (NB)}
& $D_{\mathrm{wsi},\mathrm{mol}}$
& $ {38.22 \pm 0.4}$
& $94.98 \pm 0.6$
& $ {29.43 \pm 2.6}$
& $ {36.95 \pm 1.7}$
& $ {16.73 \pm 0.5}$
& $ {47.10 \pm 1.9}$
& $ {9.61 \pm 1.9}$
& $22.40 \pm 2.2$
& $ {11.22 \pm 1.4}$
& $ {31.55 \pm 1.4}$
& $ {17.40 \pm 1.8}$ \\

& $S_{\mathrm{wsi},\mathrm{mol}}$
& $35.57 \pm 0.4$
& $\mathbf{ {95.35 \pm 0.8}}$
& $25.52 \pm 2.9$
& $35.25 \pm 1.0$
& $11.92 \pm 0.9$
& $45.23 \pm 1.4$
& $4.12 \pm 1.3$
& $ {23.20 \pm 1.9}$
& $9.59 \pm 2.0$
& $27.93 \pm 1.0$
& $9.20 \pm 1.8$ \\

& $U_{\mathrm{wsi},\mathrm{mol}}$
& $34.72 \pm 0.5$
& $94.70 \pm 0.8$
& $23.45 \pm 1.9$
& $34.92 \pm 2.2$
& $9.04 \pm 0.9$
& $44.11 \pm 0.8$
& $5.69 \pm 0.8$
& $21.20 \pm 2.3$
& $8.98 \pm 1.7$
& $27.59 \pm 1.6$
& $8.60 \pm 2.7$ \\

\bottomrule
\end{tabular}
}
\end{table*}
\newpage
\label{towards_interpretability}
\paragraph{Towards interpretability} Beyond serving as a geometric diagnostic of the learned shared spaces, the heatmaps in Figure \ref{fig:centroid_distance_tcga} suggest that the learned shared representations capture biologically plausible cancer-type structure. For example, the lung cancer subtypes LUAD and LUSC often appear closer to one another than to unrelated cancer types, which is consistent with their shared tissue of origin. More broadly, the presence of structured off-diagonal similarities suggests that the shared representations are not only encoding cancer-type separability, but also capturing relationships between diseases that may arise from common tissue, histology, or molecular programs. Importantly, the observed geometry suggests that RePercENT may expose meaningful cross-cancer relationships, motivating future work that integrates domain-specific annotations, such as mutation status, molecular subtypes, pathway activity, immune profiles, to systematically validate and characterize the biological factors driving these patterns.

\begin{figure}[h]
    \centering
    \includegraphics[width=\linewidth]{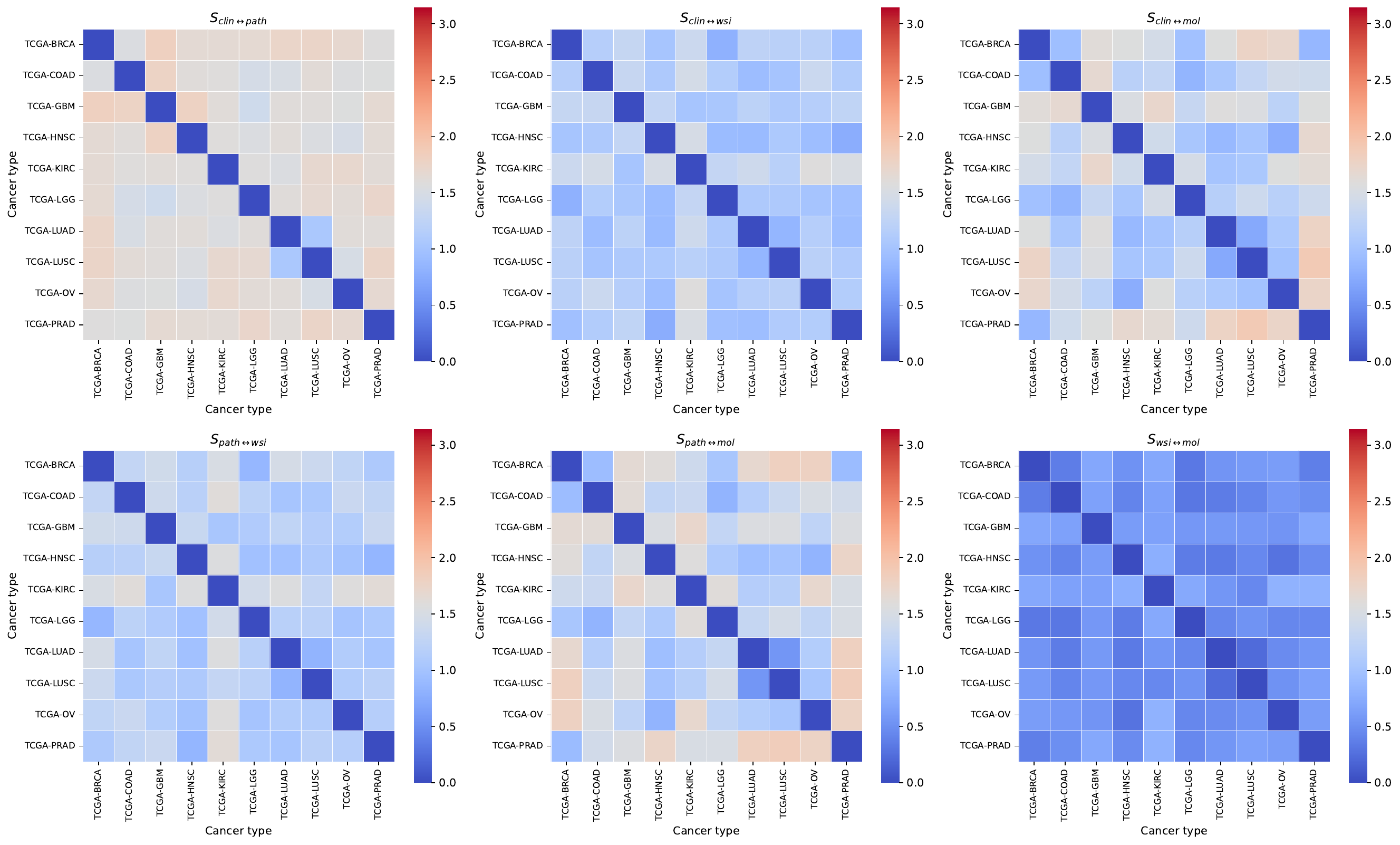}
    \caption{Evaluation of the angular distance between the studied cancer types, across the extracted shared components. The shared components, as they are modeled with the use of von-Mises Fisher distribution, they lie on a hyper-sphere. We cluster the derived shared embeddings across modalities, based on their cancer type and calculate their centroid. The heatmap reflects the angular distance between those centroids within each shared component, where values close to zero (blue) denote proximity between those cancer types, while values closer to $\pi$ (red) imply that the embeddings between these two matrices are further apart.}
    \label{fig:centroid_distance_tcga}
\end{figure}

%%%%%%%%%%%%%%%%%%%%%%%%%%%%%%%%%%%%%%%%%%%%%%%%%%%%%%%%%%%%

\clearpage
\section*{NeurIPS Paper Checklist}

%%% BEGIN INSTRUCTIONS %%%
The checklist is designed to encourage best practices for responsible machine learning research, addressing issues of reproducibility, transparency, research ethics, and societal impact. Do not remove the checklist: {\bf The papers not including the checklist will be desk rejected.} The checklist should follow the references and follow the (optional) supplemental material.  The checklist does NOT count towards the page
limit. 

Please read the checklist guidelines carefully for information on how to answer these questions. For each question in the checklist:
\begin{itemize}
    \item You should answer \answerYes{}, \answerNo{}, or \answerNA{}.
    \item \answerNA{} means either that the question is Not Applicable for that particular paper or the relevant information is Not Available.
    \item Please provide a short (1--2 sentence) justification right after your answer (even for \answerNA). 
   % \item {\bf The papers not including the checklist will be desk rejected.}
\end{itemize}

{\bf The checklist answers are an integral part of your paper submission.} They are visible to the reviewers, area chairs, senior area chairs, and ethics reviewers. You will also be asked to include it (after eventual revisions) with the final version of your paper, and its final version will be published with the paper.

The reviewers of your paper will be asked to use the checklist as one of the factors in their evaluation. While \answerYes{} is generally preferable to \answerNo{}, it is perfectly acceptable to answer \answerNo{} provided a proper justification is given (e.g., error bars are not reported because it would be too computationally expensive'' or ``we were unable to find the license for the dataset we used''). In general, answering \answerNo{} or \answerNA{} is not grounds for rejection. While the questions are phrased in a binary way, we acknowledge that the true answer is often more nuanced, so please just use your best judgment and write a justification to elaborate. All supporting evidence can appear either in the main paper or the supplemental material, provided in appendix. If you answer \answerYes{} to a question, in the justification please point to the section(s) where related material for the question can be found.

IMPORTANT, please:
\begin{itemize}
    \item {\bf Delete this instruction block, but keep the section heading ``NeurIPS Paper Checklist"},
    \item  {\bf Keep the checklist subsection headings, questions/answers and guidelines below.}
    \item {\bf Do not modify the questions and only use the provided macros for your answers}.
\end{itemize}

%%% END INSTRUCTIONS %%%

\begin{enumerate}

\item {\bf Claims}
    \item[] Question: Do the main claims made in the abstract and introduction accurately reflect the paper's contributions and scope?
    \item[] Answer: \answerYes{} % Replace by \answerYes{}, \answerNo{}, or \answerNA{}.
    \item[] Justification: All the claims found in the abstract and introduction are supported by the results and theorems discussed in the paper.
    \item[] Guidelines:
    \begin{itemize}
        \item The answer \answerNA{} means that the abstract and introduction do not include the claims made in the paper.
        \item The abstract and/or introduction should clearly state the claims made, including the contributions made in the paper and important assumptions and limitations. A \answerNo{} or \answerNA{} answer to this question will not be perceived well by the reviewers. 
        \item The claims made should match theoretical and experimental results, and reflect how much the results can be expected to generalize to other settings. 
        \item It is fine to include aspirational goals as motivation as long as it is clear that these goals are not attained by the paper. 
    \end{itemize}

\item {\bf Limitations}
    \item[] Question: Does the paper discuss the limitations of the work performed by the authors?
    \item[] Answer: \answerYes{} % Replace by \answerYes{}, \answerNo{}, or \answerNA{}.
    \item[] Justification: We discuss all the limitations and future work in the dedicated section \ref{limiations_and_future_work}.
    \item[] Guidelines:
    \begin{itemize}
        \item The answer \answerNA{} means that the paper has no limitation while the answer \answerNo{} means that the paper has limitations, but those are not discussed in the paper. 
        \item The authors are encouraged to create a separate ``Limitations'' section in their paper.
        \item The paper should point out any strong assumptions and how robust the results are to violations of these assumptions (e.g., independence assumptions, noiseless settings, model well-specification, asymptotic approximations only holding locally). The authors should reflect on how these assumptions might be violated in practice and what the implications would be.
        \item The authors should reflect on the scope of the claims made, e.g., if the approach was only tested on a few datasets or with a few runs. In general, empirical results often depend on implicit assumptions, which should be articulated.
        \item The authors should reflect on the factors that influence the performance of the approach. For example, a facial recognition algorithm may perform poorly when image resolution is low or images are taken in low lighting. Or a speech-to-text system might not be used reliably to provide closed captions for online lectures because it fails to handle technical jargon.
        \item The authors should discuss the computational efficiency of the proposed algorithms and how they scale with dataset size.
        \item If applicable, the authors should discuss possible limitations of their approach to address problems of privacy and fairness.
        \item While the authors might fear that complete honesty about limitations might be used by reviewers as grounds for rejection, a worse outcome might be that reviewers discover limitations that aren't acknowledged in the paper. The authors should use their best judgment and recognize that individual actions in favor of transparency play an important role in developing norms that preserve the integrity of the community. Reviewers will be specifically instructed to not penalize honesty concerning limitations.
    \end{itemize}

\item {\bf Theory assumptions and proofs}
    \item[] Question: For each theoretical result, does the paper provide the full set of assumptions and a complete (and correct) proof?
    \item[] Answer: \answerYes{} % Replace by \answerYes{}, \answerNo{}, or \answerNA{}.
    \item[] Justification: The paper provides a set of detailed definitions and theorems that motivate the choice of the developed model. We provide all the necessary information-theory background in Appendix \ref{information_theory_background} and all the corresponding formal proofs are also presented and can be found in Appendix \ref{proofs}.
    \item[] Guidelines:
    \begin{itemize}
        \item The answer \answerNA{} means that the paper does not include theoretical results. 
        \item All the theorems, formulas, and proofs in the paper should be numbered and cross-referenced.
        \item All assumptions should be clearly stated or referenced in the statement of any theorems.
        \item The proofs can either appear in the main paper or the supplemental material, but if they appear in the supplemental material, the authors are encouraged to provide a short proof sketch to provide intuition. 
        \item Inversely, any informal proof provided in the core of the paper should be complemented by formal proofs provided in appendix or supplemental material.
        \item Theorems and Lemmas that the proof relies upon should be properly referenced. 
    \end{itemize}

    \item {\bf Experimental result reproducibility}
    \item[] Question: Does the paper fully disclose all the information needed to reproduce the main experimental results of the paper to the extent that it affects the main claims and/or conclusions of the paper (regardless of whether the code and data are provided or not)?
    \item[] Answer: \answerYes{} % Replace by \answerYes{}, \answerNo{}, or \answerNA{}.
    \item[] Justification: The datasets we use are publicly available and we provide a step-by-step algorithm on how to generate the synthetic data in Appendix \ref{synthetic_data_creation}. To make our experiments reproducible, the set of all hyperparameters of the models used along with the training specifications are provided in detail in the Appendix.
    \item[] Guidelines:
    \begin{itemize}
        \item The answer \answerNA{} means that the paper does not include experiments.
        \item If the paper includes experiments, a \answerNo{} answer to this question will not be perceived well by the reviewers: Making the paper reproducible is important, regardless of whether the code and data are provided or not.
        \item If the contribution is a dataset and\slash or model, the authors should describe the steps taken to make their results reproducible or verifiable. 
        \item Depending on the contribution, reproducibility can be accomplished in various ways. For example, if the contribution is a novel architecture, describing the architecture fully might suffice, or if the contribution is a specific model and empirical evaluation, it may be necessary to either make it possible for others to replicate the model with the same dataset, or provide access to the model. In general. releasing code and data is often one good way to accomplish this, but reproducibility can also be provided via detailed instructions for how to replicate the results, access to a hosted model (e.g., in the case of a large language model), releasing of a model checkpoint, or other means that are appropriate to the research performed.
        \item While NeurIPS does not require releasing code, the conference does require all submissions to provide some reasonable avenue for reproducibility, which may depend on the nature of the contribution. For example
        \begin{enumerate}
            \item If the contribution is primarily a new algorithm, the paper should make it clear how to reproduce that algorithm.
            \item If the contribution is primarily a new model architecture, the paper should describe the architecture clearly and fully.
            \item If the contribution is a new model (e.g., a large language model), then there should either be a way to access this model for reproducing the results or a way to reproduce the model (e.g., with an open-source dataset or instructions for how to construct the dataset).
            \item We recognize that reproducibility may be tricky in some cases, in which case authors are welcome to describe the particular way they provide for reproducibility. In the case of closed-source models, it may be that access to the model is limited in some way (e.g., to registered users), but it should be possible for other researchers to have some path to reproducing or verifying the results.
        \end{enumerate}
    \end{itemize}

\item {\bf Open access to data and code}
    \item[] Question: Does the paper provide open access to the data and code, with sufficient instructions to faithfully reproduce the main experimental results, as described in supplemental material?
    \item[] Answer: \answerYes{} % Replace by \answerYes{}, \answerNo{}, or \answerNA{}.
    \item[] Justification: Yes, the full code with detailed instructions for reproducibility have been made public and can be found here: \href{https://github.com/Vascorn/RePercENT}{https://github.com/Vascorn/RePercENT}.
    \item[] Guidelines:
    \begin{itemize}
        \item The answer \answerNA{} means that paper does not include experiments requiring code.
        \item Please see the NeurIPS code and data submission guidelines (\url{https://neurips.cc/public/guides/CodeSubmissionPolicy}) for more details.
        \item While we encourage the release of code and data, we understand that this might not be possible, so \answerNo{} is an acceptable answer. Papers cannot be rejected simply for not including code, unless this is central to the contribution (e.g., for a new open-source benchmark).
        \item The instructions should contain the exact command and environment needed to run to reproduce the results. See the NeurIPS code and data submission guidelines (\url{https://neurips.cc/public/guides/CodeSubmissionPolicy}) for more details.
        \item The authors should provide instructions on data access and preparation, including how to access the raw data, preprocessed data, intermediate data, and generated data, etc.
        \item The authors should provide scripts to reproduce all experimental results for the new proposed method and baselines. If only a subset of experiments are reproducible, they should state which ones are omitted from the script and why.
        \item At submission time, to preserve anonymity, the authors should release anonymized versions (if applicable).
        \item Providing as much information as possible in supplemental material (appended to the paper) is recommended, but including URLs to data and code is permitted.
    \end{itemize}

\item {\bf Experimental setting/details}
    \item[] Question: Does the paper specify all the training and test details (e.g., data splits, hyperparameters, how they were chosen, type of optimizer) necessary to understand the results?
    \item[] Answer: \answerYes{} % Replace by \answerYes{}, \answerNo{}, or \answerNA{}.
    \item[] Justification: We provide detailed descriptions of all the training parameters, procedure and hyperparameters used for every model that we train in the Appendix.
    \item[] Guidelines:
    \begin{itemize}
        \item The answer \answerNA{} means that the paper does not include experiments.
        \item The experimental setting should be presented in the core of the paper to a level of detail that is necessary to appreciate the results and make sense of them.
        \item The full details can be provided either with the code, in appendix, or as supplemental material.
    \end{itemize}

\item {\bf Experiment statistical significance}
    \item[] Question: Does the paper report error bars suitably and correctly defined or other appropriate information about the statistical significance of the experiments?
    \item[] Answer: \answerYes{} % Replace by \answerYes{}, \answerNo{}, or \answerNA{}.
    \item[] Justification: We provide all our results in terms of downstream task performance in the format "mean $\pm$ standard deviation from the mean".
    \item[] Guidelines:
    \begin{itemize}
        \item The answer \answerNA{} means that the paper does not include experiments.
        \item The authors should answer \answerYes{} if the results are accompanied by error bars, confidence intervals, or statistical significance tests, at least for the experiments that support the main claims of the paper.
        \item The factors of variability that the error bars are capturing should be clearly stated (for example, train/test split, initialization, random drawing of some parameter, or overall run with given experimental conditions).
        \item The method for calculating the error bars should be explained (closed form formula, call to a library function, bootstrap, etc.)
        \item The assumptions made should be given (e.g., Normally distributed errors).
        \item It should be clear whether the error bar is the standard deviation or the standard error of the mean.
        \item It is OK to report 1-sigma error bars, but one should state it. The authors should preferably report a 2-sigma error bar than state that they have a 96\% CI, if the hypothesis of Normality of errors is not verified.
        \item For asymmetric distributions, the authors should be careful not to show in tables or figures symmetric error bars that would yield results that are out of range (e.g., negative error rates).
        \item If error bars are reported in tables or plots, the authors should explain in the text how they were calculated and reference the corresponding figures or tables in the text.
    \end{itemize}

\item {\bf Experiments compute resources}
    \item[] Question: For each experiment, does the paper provide sufficient information on the computer resources (type of compute workers, memory, time of execution) needed to reproduce the experiments?
    \item[] Answer: \answerYes{} % Replace by \answerYes{}, \answerNo{}, or \answerNA{}.
    \item[] Justification: We provide detailed information about the used resources in the Appendix.
    \item[] Guidelines:
    \begin{itemize}
        \item The answer \answerNA{} means that the paper does not include experiments.
        \item The paper should indicate the type of compute workers CPU or GPU, internal cluster, or cloud provider, including relevant memory and storage.
        \item The paper should provide the amount of compute required for each of the individual experimental runs as well as estimate the total compute. 
        \item The paper should disclose whether the full research project required more compute than the experiments reported in the paper (e.g., preliminary or failed experiments that didn't make it into the paper). 
    \end{itemize}
    
\item {\bf Code of ethics}
    \item[] Question: Does the research conducted in the paper conform, in every respect, with the NeurIPS Code of Ethics \url{https://neurips.cc/public/EthicsGuidelines}?
    \item[] Answer: \answerYes{} % Replace by \answerYes{}, \answerNo{}, or \answerNA{}.
    \item[] Justification: We preserve anonymity in our submission and our submission abides by the NeurIPS Code of Ethics.
    \item[] Guidelines:
    \begin{itemize}
        \item The answer \answerNA{} means that the authors have not reviewed the NeurIPS Code of Ethics.
        \item If the authors answer \answerNo, they should explain the special circumstances that require a deviation from the Code of Ethics.
        \item The authors should make sure to preserve anonymity (e.g., if there is a special consideration due to laws or regulations in their jurisdiction).
    \end{itemize}

\item {\bf Broader impacts}
    \item[] Question: Does the paper discuss both potential positive societal impacts and negative societal impacts of the work performed?
    \item[] Answer: \answerYes{} % Replace by \answerYes{}, \answerNo{}, or \answerNA{}.
    \item[] Justification: We discuss potential positive and negative societal impacts in the Appendix \ref{broader_impacts}.
    \item[] Guidelines:
    \begin{itemize}
        \item The answer \answerNA{} means that there is no societal impact of the work performed.
        \item If the authors answer \answerNA{} or \answerNo, they should explain why their work has no societal impact or why the paper does not address societal impact.
        \item Examples of negative societal impacts include potential malicious or unintended uses (e.g., disinformation, generating fake profiles, surveillance), fairness considerations (e.g., deployment of technologies that could make decisions that unfairly impact specific groups), privacy considerations, and security considerations.
        \item The conference expects that many papers will be foundational research and not tied to particular applications, let alone deployments. However, if there is a direct path to any negative applications, the authors should point it out. For example, it is legitimate to point out that an improvement in the quality of generative models could be used to generate Deepfakes for disinformation. On the other hand, it is not needed to point out that a generic algorithm for optimizing neural networks could enable people to train models that generate Deepfakes faster.
        \item The authors should consider possible harms that could arise when the technology is being used as intended and functioning correctly, harms that could arise when the technology is being used as intended but gives incorrect results, and harms following from (intentional or unintentional) misuse of the technology.
        \item If there are negative societal impacts, the authors could also discuss possible mitigation strategies (e.g., gated release of models, providing defenses in addition to attacks, mechanisms for monitoring misuse, mechanisms to monitor how a system learns from feedback over time, improving the efficiency and accessibility of ML).
    \end{itemize}
    
\item {\bf Safeguards}
    \item[] Question: Does the paper describe safeguards that have been put in place for responsible release of data or models that have a high risk for misuse (e.g., pre-trained language models, image generators, or scraped datasets)?
    \item[] Answer: \answerNA{} % Replace by \answerYes{}, \answerNo{}, or \answerNA{}.
    \item[] Justification:  The release of our data and models upon acceptance, does not pose a direct risk of misuse.
    \item[] Guidelines:
    \begin{itemize}
        \item The answer \answerNA{} means that the paper poses no such risks.
        \item Released models that have a high risk for misuse or dual-use should be released with necessary safeguards to allow for controlled use of the model, for example by requiring that users adhere to usage guidelines or restrictions to access the model or implementing safety filters. 
        \item Datasets that have been scraped from the Internet could pose safety risks. The authors should describe how they avoided releasing unsafe images.
        \item We recognize that providing effective safeguards is challenging, and many papers do not require this, but we encourage authors to take this into account and make a best faith effort.
    \end{itemize}

\item {\bf Licenses for existing assets}
    \item[] Question: Are the creators or original owners of assets (e.g., code, data, models), used in the paper, properly credited and are the license and terms of use explicitly mentioned and properly respected?
    \item[] Answer: \answerYes{} % Replace by \answerYes{}, \answerNo{}, or \answerNA{}.
    \item[] Justification: We provide citations for all sources of code and/or data used. For the datasets, we reference the original open-source work they are based on, in full accordance with their licensing terms.
    \item[] Guidelines:
    \begin{itemize}
        \item The answer \answerNA{} means that the paper does not use existing assets.
        \item The authors should cite the original paper that produced the code package or dataset.
        \item The authors should state which version of the asset is used and, if possible, include a URL.
        \item The name of the license (e.g., CC-BY 4.0) should be included for each asset.
        \item For scraped data from a particular source (e.g., website), the copyright and terms of service of that source should be provided.
        \item If assets are released, the license, copyright information, and terms of use in the package should be provided. For popular datasets, \url{paperswithcode.com/datasets} has curated licenses for some datasets. Their licensing guide can help determine the license of a dataset.
        \item For existing datasets that are re-packaged, both the original license and the license of the derived asset (if it has changed) should be provided.
        \item If this information is not available online, the authors are encouraged to reach out to the asset's creators.
    \end{itemize}

\item {\bf New assets}
    \item[] Question: Are new assets introduced in the paper well documented and is the documentation provided alongside the assets?
    \item[] Answer: \answerNA{} % Replace by \answerYes{}, \answerNo{}, or \answerNA{}.
    \item[] Justification:  This paper does not release any new assets beyond code, which is adequately documented in order to reproduce all results from the paper.
    \item[] Guidelines:
    \begin{itemize}
        \item The answer \answerNA{} means that the paper does not release new assets.
        \item Researchers should communicate the details of the dataset\slash code\slash model as part of their submissions via structured templates. This includes details about training, license, limitations, etc. 
        \item The paper should discuss whether and how consent was obtained from people whose asset is used.
        \item At submission time, remember to anonymize your assets (if applicable). You can either create an anonymized URL or include an anonymized zip file.
    \end{itemize}

\item {\bf Crowdsourcing and research with human subjects}
    \item[] Question: For crowdsourcing experiments and research with human subjects, does the paper include the full text of instructions given to participants and screenshots, if applicable, as well as details about compensation (if any)? 
    \item[] Answer: \answerNA{} % Replace by \answerYes{}, \answerNo{}, or \answerNA{}.
    \item[] Justification: No crowdsourcing or human subjects were involved in the experiments conducted for this paper.
    \item[] Guidelines:
    \begin{itemize}
        \item The answer \answerNA{} means that the paper does not involve crowdsourcing nor research with human subjects.
        \item Including this information in the supplemental material is fine, but if the main contribution of the paper involves human subjects, then as much detail as possible should be included in the main paper. 
        \item According to the NeurIPS Code of Ethics, workers involved in data collection, curation, or other labor should be paid at least the minimum wage in the country of the data collector. 
    \end{itemize}

\item {\bf Institutional review board (IRB) approvals or equivalent for research with human subjects}
    \item[] Question: Does the paper describe potential risks incurred by study participants, whether such risks were disclosed to the subjects, and whether Institutional Review Board (IRB) approvals (or an equivalent approval/review based on the requirements of your country or institution) were obtained?
    \item[] Answer: \answerNA{} % Replace by \answerYes{}, \answerNo{}, or \answerNA{}.
    \item[] Justification: No crowdsourcing or human subjects were involved in the experiments conducted for this paper.
    \item[] Guidelines:
    \begin{itemize}
        \item The answer \answerNA{} means that the paper does not involve crowdsourcing nor research with human subjects.
        \item Depending on the country in which research is conducted, IRB approval (or equivalent) may be required for any human subjects research. If you obtained IRB approval, you should clearly state this in the paper. 
        \item We recognize that the procedures for this may vary significantly between institutions and locations, and we expect authors to adhere to the NeurIPS Code of Ethics and the guidelines for their institution. 
        \item For initial submissions, do not include any information that would break anonymity (if applicable), such as the institution conducting the review.
    \end{itemize}

\item {\bf Declaration of LLM usage}
    \item[] Question: Does the paper describe the usage of LLMs if it is an important, original, or non-standard component of the core methods in this research? Note that if the LLM is used only for writing, editing, or formatting purposes and does \emph{not} impact the core methodology, scientific rigor, or originality of the research, declaration is not required.
    %this research? 
    \item[] Answer: \answerNA{} % Replace by \answerYes{}, \answerNo{}, or \answerNA{}.
    \item[] Justification: LLMs are not a core component of this paper in terms of methods.
    \item[] Guidelines:
    \begin{itemize}
        \item The answer \answerNA{} means that the core method development in this research does not involve LLMs as any important, original, or non-standard components.
        \item Please refer to our LLM policy in the NeurIPS handbook for what should or should not be described.
    \end{itemize}

\end{enumerate}

\end{document}